\documentclass{article}
\usepackage{iclr2027_preprint,times}

\usepackage{amsmath,amssymb,amsthm,mathtools}
\usepackage{booktabs}
\usepackage{graphicx}
\usepackage{enumitem}
\usepackage{algorithm}
\usepackage{algpseudocode}
\usepackage{hyperref}
\hypersetup{hidelinks}
\usepackage{url}
\usepackage{xcolor}
\usepackage{xspace}

\newtheorem{theorem}{Theorem}
\newtheorem{proposition}[theorem]{Proposition}
\newtheorem{lemma}[theorem]{Lemma}
\newtheorem{corollary}[theorem]{Corollary}
\theoremstyle{definition}
\newtheorem{assumption}{Assumption}
\theoremstyle{remark}
\newtheorem{remark}[theorem]{Remark}

\DeclareMathOperator{\E}{\mathbb{E}}
\newcommand{\R}{\mathbb{R}}
\newcommand{\Id}{I_d}
\newcommand{\gK}{g_{\mathrm{K}}}
\newcommand{\gU}{g_{\mathrm{U}}}
\newcommand{\ghK}{\widehat{g}_{\mathrm{K}}}
\newcommand{\inner}[2]{\langle #1, #2\rangle}
\newcommand{\Fk}{\mathcal{F}_k}
\newcommand{\Ek}{\E_k}
\newcommand{\bt}{\tilde{\beta}}
\newcommand{\piso}{\textsc{Piso}\xspace}
\newcommand{\pisotwo}{\textsc{Piso}$^{2}$\xspace}

\title{Partially Informed Stochastic Optimization\\ via Zeroth-Order Residual Tracking}

\author{El Mahdi Chayti, Alireza Mirrokni, Martin Jaggi \\Machine Learning and Optimization Laboratory (MLO) , EPFL}

\iclrfinalcopy   

\begin{document}

\maketitle

\begin{abstract}
Many optimization problems provide direct access to only part of the gradient, while the remaining component must be inferred from function evaluations. This structure arises in performative prediction, performative reinforcement learning, and bilevel optimization with a black-box lower level. We introduce Partially Informed Stochastic Optimization (PISO), which combines the available gradient component with a running estimate of the residual. By centring each zeroth-order observation on this estimate, PISO makes the variance due to random probing depend on the tracking error rather than the residual’s magnitude. This yields a randomized sketch iteration with an explicit contraction mechanism and a principled interpretation of momentum. For smooth nonconvex objectives, we establish convergence guarantees that depend on the residual’s variation and initial magnitude, without requiring a uniform bound on the unavailable gradient component. With independently sampled function evaluations and Gaussian or spherical probes, PISO attains leading function-evaluation complexities of \(O(d^{2}\varepsilon^{-6})\) under Lipschitz gradients and \(O(d^{2}\varepsilon^{-5})\) under an additional Lipschitz-Hessian assumption. These rates require only one probe direction, two function samples, and one known-component sample per update. The analysis also quantifies how batching trades observations for fewer deployments. Across five settings, PISO variants improve on the tested baselines with comparable oracle access in three, reach a comparable solution faster in bilevel routing, and perform best among methods without response-model derivatives in performative reinforcement learning\footnote{The implementation is available at \url{https://github.com/alirezamirrokni/PISO}.}.

\end{abstract}

\section{Introduction}

Let $F:\R^d\to\R$ be a function whose gradient decomposes as
\begin{equation}\label{eq:split}
\nabla F(x) \;=\; \gK(x) \;+\; \gU(x),
\end{equation}
where $\gK$ can be estimated by direct sampling and $\gU$ cannot: the only access to $\gU$ is
through noisy evaluations of $F$. Problems of this form arise in several areas.

\textbf{Performative prediction.} The decision changes the data distribution
\citep{perdomo2020performative}: one minimizes $F(x)=\E_{\xi\sim\mathcal D(x)}[\ell(x,\xi)]$, where
$\mathcal D(x)$ is the distribution the deployed model $x$ induces. Differentiating under the two
dependencies makes $\gK(x)=\E_{\xi\sim\mathcal D(x)}[\nabla_x\ell(x,\xi)]$ the ordinary loss gradient
on the population currently observed, available by sampling, and $\gU$ the term involving the
derivative of the map $x\mapsto\mathcal D(x)$, which the practitioner does not possess. A lender
adjusting a credit-scoring rule can compute the first, but learns how applicants adapt their reported
features only by deploying a rule and observing outcomes \citep{levanon2021strategic}.

\textbf{Performative reinforcement learning.} The environment reacts to the deployed policy
\citep{bell2021newcomblike,mandal2023performative}. The performative policy gradient theorem of
\citet{basu2025pepg} splits $\nabla_\theta V^{\pi_\theta}_{\pi_\theta}$ into a score term in
$\nabla_\theta\log\pi_\theta$, the familiar policy gradient and our $\gK$, plus terms in
$\nabla_\theta\log P_{\pi_\theta}$ and $\nabla_\theta r_{\pi_\theta}$, our $\gU$, which describe how
the kernel and reward respond to the policy and require a model of that response.

\textbf{Bilevel problems with an opaque lower level.} Two parties act in sequence. The \emph{leader}
moves first, chooses $x$ and pays the objective; the \emph{followers} take $x$ as given and settle on
a response $y^\star(x)$, for instance the equilibrium of a game among themselves
\citep{cui2023mpec,fischer2026blackbox}. With $F(x)=f(x,y^\star(x))$ and $f$ known to the leader,
\[
\nabla F(x)=\underbrace{\nabla_xf\big(x,y^\star(x)\big)}_{\gK(x)}
+\underbrace{\big[\partial_xy^\star(x)\big]^{\!\top}\nabla_yf\big(x,y^\star(x)\big)}_{\gU(x)} :
\]
the leader can evaluate both partial derivatives once $y^\star(x)$ is observed, so what is missing is
the response Jacobian. In toll setting, the leader is the road authority and the followers are the
drivers who reroute; here $\gK$ is exact, whereas in the two performative settings it must be
sampled. The same structure appears whenever a differentiable model is composed with a black-box
simulator, or with a legacy or proprietary subsystem.

\paragraph{Prior approaches and this work.} Treating $F$ as a black box and applying zeroth-order
optimization to all of $\nabla F$ \citep{nesterov2017random,ghadimi2013stochastic} discards $\gK$.
Recent work in each area exploits it: \citet{hikima2025guided} use the computable component as a
guide in performative prediction and \citet{hikima2025estimators} attain
$\mathcal O(d^2\varepsilon^{-6})$ samples by averaging many probe directions per step;
\citet{fischer2026blackbox} keep the leader's partial gradients exact and estimate only the
followers' response Jacobian. These methods share a template, estimating a gradient or a Jacobian at the current iterate and then
reducing its variance, and a scope, each being developed for one area. We keep the premise that
$\gK$ should be used and change the object being estimated: we maintain a running estimate of the
opaque part and centre every zeroth-order observation on it, so that each observation reports how
wrong the current estimate is rather than a fresh guess. The update is then a randomized fixed-point
iteration whose error contracts, with stochastic error proportional to the current estimation error
rather than to the size of $\gU$, and because it needs only \eqref{eq:split}, one algorithm and one
analysis cover all three areas.

\paragraph{Contributions.}
\begin{itemize}[leftmargin=1.2em,itemsep=1pt]
\item We formulate optimization under \eqref{eq:split} as \emph{residual identification} and derive
\piso and \pisotwo (\S\ref{sec:method}), showing that centring the zeroth-order observation on the
running estimate turns gradient estimation into an error-feedback contraction.
\item We give a single convergence analysis (\S\ref{sec:theory}) whose complexity is governed by the
smoothness $L_v$ and initial magnitude $\|v(x_0)\|$ of the residual, with no dependence on
a uniform bound on $\|\gU\|$ and no bounded-gradient, bounded-domain, convexity or PL assumption. It covers the two
stochastic oracle regimes and any probe distribution satisfying three moment conditions, matching the
$d^2\varepsilon^{-6}$ and $d^2\varepsilon^{-5}$ rates of \citet{hikima2025estimators} with $L$
replaced by $L_v$, one probe per update instead of $\Omega(d)$, and one sample of each oracle per
iteration.
\item The analysis assigns each parameter a role and a prescription: the momentum $\rho$ is the step
size of the sketch and has a stability window of width $2/c_d$; a second momentum $\rho_K$ on the
known component is what makes one sample per iteration admissible; and the weight $\gamma$ is a
control-variate coefficient interpolating between black-box and white-box treatment. We also separate
the cost of \emph{deploying} a decision from that of observing outcomes under it, and show the
deployment count can be cut from $\Theta(d^2\varepsilon^{-6})$ to $\Theta(d^2\varepsilon^{-2})$ at no
cost in total observations, and to $\Theta(d\,\varepsilon^{-2})$ at the cost of a factor $d$ in
observations.
\item We evaluate on five settings spanning the three areas (\S\ref{sec:experiments}). \piso is best
on pricing, strategic classification and a security game, matches the strongest bilevel baseline at
convergence on a routing game while reaching it earlier, and on performative RL is best among methods
that query the environment rather than differentiating a model of it. A momentum sweep locates the
optimum at the predicted $1-1/c_d$ and exhibits the predicted divergence below $1-2/c_d$.
\end{itemize}

\section{From gradient estimation to residual tracking}\label{sec:method}

\paragraph{Setting and notation.} We minimize a function $F:\R^d\to\R$ satisfying
\eqref{eq:split}, given two oracles. The first returns $\ghK(x)$, unbiased for $\gK(x)$ with
variance $\sigma_K^2/n_K$, where $n_K$ is a batch size; we write
$\delta_k:=\gamma\big(\ghK(x_k)-\gK(x_k)\big)$ for its error scaled by a weight $\gamma>0$ introduced
below. The second returns noisy values $\widehat F(x)$, each an average of $n_F$ samples. At
iteration $k$ we draw a \emph{probe} $u_k\in\R^d$ with a \emph{score} $s_k\in\R^d$ satisfying
\begin{equation}\label{eq:unbiased_sketch}
\E\big[s_ku_k^\top\big]=\Id,
\qquad\text{so that}\qquad
\E\big[\inner{z}{u_k}\,s_k\big]=z \quad\text{for every fixed } z\in\R^d ,
\end{equation}
for instance $u_k=s_k\sim\mathcal N(0,\Id)$ or $u_k=s_k=\sqrt d\,w_k$ with $w_k$ uniform on the unit
sphere. Let $c_d$ be the smallest constant with $\E\|\inner{z}{u_k}s_k\|^2\le c_d\|z\|^2$; it equals
$d+2$ and $d$ for those two choices. From the two function queries we form the symmetric two-point
directional derivative, which decomposes as
\begin{equation}\label{eq:twopoint}
D_k \;:=\; \frac{\widehat F(x_k+\mu_ku_k)-\widehat F(x_k-\mu_ku_k)}{2\mu_k}
\;=\; \inner{\nabla F(x_k)}{u_k} + b_k + \xi_k ,
\end{equation}
where $b_k$ is the finite-difference bias and $\xi_k$ has conditional mean zero. An iteration costs
$2n_F$ function evaluations and $n_K$ known-component samples while touching only the three points
$x_k$ and $x_k\pm\mu_ku_k$; in a performative problem each of those is a \emph{deployed decision},
and Appendix~\ref{app:deploy} takes up the distinction. Throughout, $\eta$ is the step size, $\mu_k$
the probe radius, $\rho,\rho_K\in[0,1)$ two momentum parameters, $\beta:=1-\rho$,
$\beta_K:=1-\rho_K$, and $m_k,h_k$ the states maintained by the method.

\paragraph{The tracked object.} For the weight $\gamma>0$ define the \emph{residual}
\begin{equation}\label{eq:residual}
v(x) \;:=\; \nabla F(x) - \gamma\,\gK(x) \;=\; \gU(x) + (1-\gamma)\,\gK(x) .
\end{equation}
At $\gamma=1$ it is exactly $\gU$, and at $\gamma=0$ the full gradient. We keep $\gamma$ general
because the two endpoints are two known designs, a fully black-box method and the principle of
applying zeroth-order estimation only where it is needed \citep{fischer2026blackbox}, and because the
analysis identifies a better choice. Our goal is a state $m_k$ tracking $v(x_k)$, since then
$\gamma\ghK(x_k)+m_k$ estimates $\nabla F(x_k)$.

\paragraph{A first estimator, and its cost.} Rearranging \eqref{eq:twopoint} with
$\nabla F=\gamma\gK+v$ shows what the two queries reveal about $v$:
\begin{equation}\label{eq:scalarmeas}
D_k - \inner{\gamma\ghK(x_k)}{u_k}
\;=\; \inner{v(x_k)}{u_k} \;-\;\inner{\delta_k}{u_k} \;+\; b_k+\xi_k .
\end{equation}
Each iteration supplies one noisy \emph{scalar} linear measurement of the unknown vector $v(x_k)$
along $u_k$. Multiplying by the score turns it into a vector estimate, and by
\eqref{eq:unbiased_sketch} the estimate $R^{\mathrm{fresh}}_k:=(D_k-\inner{\gamma\ghK(x_k)}{u_k})s_k$
is unbiased for $v(x_k)$ up to $b_k$. Its second moment, however, is of order $c_d\|v(x_k)\|^2$: a
single probe inflates the estimate by the dimension, and the error persists at that level however
long the method runs. Averaging $N$ probes per iteration divides the inflation by $N$, but the
guarantee of \citet{hikima2025estimators} requires $N=\Omega(d)$, so $\Omega(d)$ queries precede
every update; in performative settings each query is a deployment.

\paragraph{Measuring the discrepancy instead.} The estimator $R^{\mathrm{fresh}}_k$ discards the
state: it reconstructs $v(x_k)$ from scratch even when $m_k$ is already close. Suppose instead we
treat $m_k$ as known and estimate only the discrepancy $v(x_k)-m_k$. Subtracting $\inner{m_k}{u_k}$
from \eqref{eq:scalarmeas} gives a scalar measurement of that discrepancy along $u_k$; sketching it
with $s_k$ estimates the discrepancy as a vector; and adding $m_k$ back recovers an estimate of
$v(x_k)$:
\begin{equation}\label{eq:centred}
R_k \;:=\; m_k + \big(D_k - \inner{\gamma\ghK(x_k)+m_k}{u_k}\big)\,s_k .
\end{equation}
By \eqref{eq:unbiased_sketch} this has the same target as $R^{\mathrm{fresh}}_k$, since the two
occurrences of $m_k$ cancel in expectation. What changes is the scale of the error, now set by
$\|v(x_k)-m_k\|$ rather than by $\|v(x_k)\|$, so it shrinks as the state improves. This is the device
behind control variates and error feedback: estimate the correction to a known quantity, not the
quantity itself.

\begin{algorithm}[t]
\caption{\piso ($\lambda=0$) and \pisotwo ($\lambda\in(0,1]$)}\label{alg:piso}
\begin{algorithmic}[1]
\Require $x_0$, $m_0=0$, $h_0=\ghK(x_0)$, weight $\gamma>0$, momenta $\rho,\rho_K$, mixing $\lambda$,
schedules $\{\eta_k\},\{\mu_k\}$
\For{$k=0,\dots,T-1$}
  \State Draw $\ghK(x_k)$; draw a probe direction $u_k$ with score $s_k$
  \State $D_k \gets \big[\widehat F(x_k+\mu_ku_k)-\widehat F(x_k-\mu_ku_k)\big]/(2\mu_k)$
  \State $R_k \gets m_k+\big(D_k-\inner{\gamma\ghK(x_k)+m_k}{u_k}\big)s_k$
     \Comment{\emph{residual observation, centred on $m_k$}}
  \State $m_{k+1}\gets \rho\,m_k+(1-\rho)R_k$
  \State $h_{k+1}\gets \rho_K\,h_k+(1-\rho_K)\,\ghK(x_k)$
     \Comment{\emph{known-component average}}
  \State $G_k \gets \gamma\,h_{k+1}+\lambda R_k+(1-\lambda)m_{k+1}$
  \State $x_{k+1}\gets x_k-\eta_kG_k$ \Comment{sign reversed for maximization}
\EndFor
\end{algorithmic}
\end{algorithm}

We call the method \emph{Partially Informed Stochastic Optimization} (\piso), and \pisotwo for the
variant with $\lambda>0$, which reads out the direction from a mixture of the fresh observation
$R_k$ and the state. Line 6 tracks the known component as well, because $\ghK(x_k)$ enters the step
directly and using it fresh injects its full variance $\gamma^2\sigma_K^2$ into $G_k$; averaging
reduces this to $\mathcal O(\beta_K\gamma^2\sigma_K^2)$ at the cost of a lag the step size must
respect, and $\rho_K=0$ recovers the fresh estimate. We deliberately do \emph{not} substitute
$h_{k+1}$ for $\ghK(x_k)$ inside $R_k$: doing so would change the tracked object to
$\nabla F(x_k)-\gamma h_{k+1}$, whose drift constant is $L_v+\gamma L_K$ rather than $L_v$
(Appendix~\ref{app:main}). With $\bt:=1-(1-\lambda)\rho$ and $\zeta_k:=\gamma(h_k-\gK(x_k))$, both
$\Fk$-measurable, the error recursions are as follows.

\begin{lemma}[\piso is sketched error feedback]\label{lem:structure}
Let $E_k:=m_k-v(x_k)$, $\Delta_k:=G_k-\nabla F(x_k)$ and $P_k:=s_ku_k^\top$. With $m_0=0$:
(i) $\lambda R_k+(1-\lambda)m_{k+1} = (1-\bt)m_k+\bt R_k$;
(ii) $E_{k+1} = (\Id-\beta P_k)E_k + \beta\big(\!-\!\inner{\delta_k}{u_k}+b_k+\xi_k\big)s_k+(v_k-v_{k+1})$;
(iii) $\Delta_k = (1-\beta_K)\zeta_k+(\Id-\bt P_k)E_k+\big(\beta_K\Id-\bt P_k\big)\delta_k+\bt(b_k+\xi_k)s_k$;
(iv) $\zeta_{k+1} = (1-\beta_K)\zeta_k+\beta_K\delta_k+\gamma\big(\gK(x_k)-\gK(x_{k+1})\big)$.
\end{lemma}

\paragraph{Why this contracts.} In the idealized case of exact oracles and a stationary iterate,
(ii) reduces to the exact identity
\begin{equation}\label{eq:teaser}
E_{k+1}=(\Id-\beta P_k)E_k .
\end{equation}
This is a step of stochastic gradient descent on the quadratic $\tfrac12\|m-v\|^2$ with step size
$\beta$, the gradient $E_k$ replaced by its sketch $P_kE_k$, which is unbiased by
\eqref{eq:unbiased_sketch}. Taking second moments gives the textbook step-size trade-off, a linear
descent term against a quadratic variance term:
$\Ek\|(\Id-bP_k)z\|^2\le\theta(b)\|z\|^2$ with
\begin{equation}\label{eq:theta}
\theta(b) \;:=\; 1-2b+c_db^2 .
\end{equation}
So $\beta$ is the step size of a randomized fixed-point iteration for $v$: under (D1)--(D2),
$\theta(\beta)$ bounds the contraction factor, is minimized at $\beta=1/c_d$, and guarantees
contraction for $0<\beta<2/c_d$; for the Gaussian, sphere and coordinate families (D2) is an equality
and the statements are exact:
\begin{equation}\label{eq:rhostar}
\rho^\star=1-1/c_d,
\qquad\text{and the residual iteration is neutral at }\rho=1-2/c_d\text{ and \emph{expands} for }\rho<1-2/c_d .
\end{equation}
In particular $\rho=0$ is not ``no momentum'' but $\beta=1$, an expanding map with
$\theta(1)=c_d-1$. The update sits in the family of error-feedback and
sketch-and-project methods \citep{richtarik2021ef21,gower2015randomized}. In (iii) the oracle error
$\delta_k$ appears only through $\beta_K\Id-\bt P_k$, whose second moment against a mean-zero
$\delta_k$ is $(\beta_K^2-2\beta_K\bt+c_d\bt^2)\E\|\delta_k\|^2$, and through $\zeta_k$, whose steady
state is $\mathcal O(\beta_K\gamma^2\sigma_K^2)$ by (iv). Both are tunable, which is what allows one
known-component sample per iteration.

\section{Convergence}\label{sec:theory}

\begin{assumption}[Smoothness]\label{as:smooth}
$F$ is $L$-smooth and bounded below by $F^\star$, with $\Delta_F:=F(x_0)-F^\star$; and $\gK$ is
$L_K$-Lipschitz. We write $L_v$ for the smallest constant with
$\Ek\|v_{k+1}-v_k\|^2\le L_v^2\eta^2\Ek\|G_k\|^2$.
\end{assumption}

Since $v$ is $(L+\gamma L_K)$-Lipschitz, $L_v\le L+\gamma L_K$ always; the rates below improve on
black-box ones when $L_v$ is small, which depends on the problem. One sets $\eta$ with an a-priori
upper bound $\bar L_v$ (e.g.\ $L+\gamma L_K$); the trajectory-specific $L_v$ sharpens the bound only
retrospectively.

\begin{assumption}[Oracles]\label{as:oracles}
$\Ek[\ghK(x_k)]=\gK(x_k)$ with variance $\sigma_{K,k}^2$, drawn independently of $(u_k,s_k)$; and
$\E[\xi_k\mid\Fk,u_k,s_k,\delta_k]=0$ with $\E[\xi_k^2\mid\Fk,u_k,s_k]\le\sigma_{D,k}^2$. Write
$\Sigma_{b,k}^2 := \Ek[b_k^2\|s_k\|^2]$.
\end{assumption}

\begin{assumption}[Probe distribution]\label{as:dir}
$(u_k,s_k)$ is drawn i.i.d.\ given $\Fk$, independently of $\delta_k$, and satisfies (D1)
$\Ek[s_ku_k^\top]=\Id$; (D2) $\Ek\|\inner{z}{u_k}s_k\|^2\le c_d\|z\|^2$; (D3) moment bounds
$\varsigma^2,\varsigma_4^2,\varsigma_6^2,\varsigma_8^2$ on $\|s_k\|^2$ against powers of $\|u_k\|$.
\end{assumption}

Only (D1)--(D2) enter the contraction; (D3) enters the noise constants. Theorem~\ref{thm:main} is
stated in terms of the \emph{mean} bias $\bar b_k:=\Ek[b_ks_k]$ and a second bias vector
$\bar B_k:=\Ek[b_k\|s_k\|^2u_k]$; the corollaries bound them in one of two ways
(Appendix~\ref{app:bias}). If the family satisfies a \emph{smoothing identity}
$\E[\tfrac{F(x+\mu u)-F(x-\mu u)}{2\mu}s]=\nabla F_{\nu,\kappa\mu}(x)$ for a symmetric measure $\nu$
and scale $\kappa$, then $\|\bar b_k\|\le L\mu_k\chi_1$ with $\chi_1:=\kappa\E_\nu\|w\|$, and
$\tfrac12Q\mu_k^2\chi_2$ with $\chi_2:=\kappa^2\E_\nu\|w\|^2$ under a $Q$-Lipschitz Hessian
\citep{berahas2022theoretical,maggiar2018derivative}; Gaussian and sphere probes satisfy it with
$\chi_1=\sqrt d$, $\chi_2=d$. Coordinate probes do not, and a direct bound gives $\chi_1=d/2$, a
factor $\sqrt d$ worse, which is the structural cause of the known gap between coordinate and random
directions. The distinction between $\bar b_k$ and $\Sigma_{b,k}$ controls the dimension dependence:
the bias accumulates coherently over the $\sim1/\beta$ steps of the momentum window, so its mean is
undamped while its fluctuation is damped by $\beta$ like any other noise.

\begin{theorem}[Main]\label{thm:main}
Under Assumptions~\ref{as:smooth}--\ref{as:dir}, with $\theta:=\theta(\bt)$, $\beta\le1/c_d$ and
\begin{equation}\label{eq:stepcond}
\eta\;\le\;\min\Big\{\tfrac{1}{2L},\ \tfrac{\beta}{11\sqrt\theta\,L_v},\ \tfrac{\beta_K}{9\,\gamma L_K}\Big\},
\end{equation}
the iterates of Algorithm~\ref{alg:piso} run with a constant step size $\eta_k\equiv\eta$ (the radius
$\mu_k$ may vary) satisfy
\begin{align}\label{eq:main}
\frac1T\sum_{k<T}\E\|\nabla F(x_k)\|^2
&\le \frac{2\Delta_F}{\eta T}+\frac{6\theta\|v(x_0)\|^2}{\beta T}
+\frac{6\gamma^2\sigma_{K}^2}{n_K\beta_K T}\notag\\
&\quad+\frac2T\sum_{k<T}\Big[\big(\tfrac32\beta_K^2+\tfrac32c_d\bt^2+\tfrac{15}{4}\theta\beta c_d+\tfrac92\beta_K\big)\gamma^2\sigma_{K,k}^2\notag\\
&\qquad\qquad\qquad+\big(\tfrac32\bt^2+\tfrac{15}{4}\theta\beta\big)\big(\varsigma^2\sigma_{D,k}^2+\Sigma_{b,k}^2\big)\notag\\
&\qquad\qquad\qquad+30\theta\|\bar b_k\|^2+30\theta\beta^2\|\bar B_k\|^2\Big].
\end{align}
\end{theorem}

No bounded gradients, bounded domain, convexity, PL condition, or uniform bound on $\|\gU\|$ is
used: the opaque component enters only through its value at $x_0$, via $\|v(x_0)\|$, and through
its drift $L_v$.
Every occurrence of the known-oracle variance carries a factor $\beta c_d$, $\bt^2c_d$ or $\beta_K$,
so a single known-component sample per iteration suffices; at $\rho_K=0$ the last coefficient is
$\Theta(1)$ and $\sigma_K^2$ becomes a floor removable only by drawing
$\Theta(\gamma^2\sigma_K^2/\varepsilon^2)$ samples per iteration (Appendix~\ref{app:complexity}). The
three bias terms sit at different orders in $\beta$, and only the mean $\|\bar b_k\|^2$ is undamped.

\subsection{Oracle regimes and complexity}\label{sec:regimes}

Whether the two function queries in \eqref{eq:twopoint} can share randomness determines the noise
model. If they evaluate the same realization, the noise is dominated by a $\mu$-free gradient term, a
finite probe radius of order $\min\{\sigma_g/\ell,\varepsilon/L\}$ suffices, and the method reaches
$\varepsilon$-stationarity in
$\mathcal O\big(\Delta_v d(\gamma^2\sigma_K^2+\sigma_g^2)\varepsilon^{-4}\big)$ iterations with
$\Delta_v:=L_v\Delta_F+\|v(x_0)\|^2$ (Appendix~\ref{app:complexity}, Regime A). In a performative
problem the two queries \emph{deploy} $x_k\pm\mu_ku_k$ and sample from two different distributions,
so they cannot share randomness and
\begin{equation}\label{eq:sigDmain}
\sigma_{D,k}^2\;\le\;\frac{\sigma_F^2}{2\mu_k^2\,n_F},
\end{equation}
with $\sigma_F^2$ a bound on the conditional variance of the objective under the deployed
distribution. The probe radius must then balance \eqref{eq:sigDmain} against the finite-difference
bias, and the balance is asymmetric: variance is damped by $\beta$ in \eqref{eq:main}, the mean bias
is not, which fixes the exponent of $\varepsilon$. Taking $\lambda=0$, so $\bt=\beta$, and the
sharp bound $\|\bar B_k\|\le3dL\mu_k\chi_1$ of the random families under $\beta\le1/(3d)$, the noise
part of \eqref{eq:main} is at most
\begin{equation}\label{eq:floormain}
\big(11\beta c_d+12\beta_K\big)\frac{\gamma^2\sigma_K^2}{n_K}
+11\beta\,\varsigma^2\sigma_D^2+11\beta\,\Sigma_b^2+60\|\bar b\|^2+60L^2\mu^2\chi_1^2 .
\end{equation}

\begin{corollary}[Independent randomization]\label{cor:B}
Write $C_1:=102000\,L^2\sigma_F^2\varsigma^2\chi_1^2$ and $C_2:=1220\,\varsigma^2\sigma_F^2Q\chi_2$, and
let $n_K\ge1$ and $\beta_K=\min\{1,n_K\varepsilon^2/(72\gamma^2\sigma_K^2)\}$, with the warm start
$h_0=\ghK(x_0)$, $\mu_\star$ as in Appendix~\ref{app:complexity}, and $\eta$ \emph{equal to} the
right-hand side of \eqref{eq:stepcond}. Under $L$-smoothness alone, take
$\beta=n_F\varepsilon^4/C_1$ with $1\le n_F\le n_F^{\max}:=C_1n_K/(66c_d\gamma^2\sigma_K^2\varepsilon^2)$
and $\varepsilon^2\le7\gamma^2\sigma_K^2/(dn_K)$; if $\nabla^2F$ is $Q$-Lipschitz, take instead
$\beta=n_F\varepsilon^3/C_2$ with $n_F\le C_2n_K/(66c_d\gamma^2\sigma_K^2\varepsilon)$ and
$\varepsilon^2\le\gamma^2\sigma_K^2/(dn_K)$. The total number of function evaluations to reach
$\frac1T\sum_{k<T}\E\|\nabla F_k\|^2\le\varepsilon^2$ then satisfies
\begin{align}
\text{under $L$-smoothness alone:}&\qquad
\mathsf N=\mathcal O\Big(\frac{\Delta_v\;L^2\sigma_F^2\;\varsigma^2\chi_1^2}{\varepsilon^{6}}
+\frac{n_F}{n_K}\cdot\frac{\gamma^3L_K\sigma_K^2\Delta_F}{\varepsilon^{4}}+R_{\mathrm{lo}}\Big),
\label{eq:rateB6}\\
\text{if $\nabla^2F$ is $Q$-Lipschitz:}&\qquad
\mathsf N=\mathcal O\Big(\frac{\Delta_v\;\sigma_F^2\;\varsigma^2Q\chi_2}{\varepsilon^{5}}
+\frac{n_F}{n_K}\cdot\frac{\gamma^3L_K\sigma_K^2\Delta_F}{\varepsilon^{4}}+R_{\mathrm{lo}}\Big),
\label{eq:rateB5}
\end{align}
where $R_{\mathrm{lo}}:=n_F\gamma^4\sigma_K^4/(n_K^2\varepsilon^4)+n_FL\Delta_F/\varepsilon^2$
collects the initialization and $L$-branch terms, lower order at $n_F=1$ and made explicit in
Appendix~\ref{app:complexity}, where a batched warm start is also shown to remove the former. The leading terms are independent of $n_F$ and $n_K$ and carry no
$\sigma_K$.
\end{corollary}

Reading $\varsigma^2\chi_1^2$ and $\varsigma^2\chi_2$ off the probe constants gives $d^2$ for
Gaussian and sphere probes in both cases; coordinate probes, which lack the smoothing identity, are
treated by a separate corollary with the direct bias bounds and give $d^3$ and $d^{5/2}$
(Appendix~\ref{app:complexity}). These are the six complexities reported by
\citet{hikima2025estimators} for plain multi-direction methods, obtained here from one analysis and
a table of constants. Three things change. \emph{The factor multiplying
$\Delta_F$ is $L_v$, not $L$}: the step-size condition \eqref{eq:stepcond} is a minimum of three
branches, the $L_v/\beta$ branch dominates once $\beta=\mathcal O(\varepsilon^2)$, and $L$ survives
only at order $\varepsilon^{-2}$, while the $L^2$ in \eqref{eq:rateB6} is the finite-difference bias
common to any two-point method (Appendix~\ref{app:LvsLv}). \emph{One probe per update suffices},
where the guarantee of \citet{hikima2025estimators} needs $N\ge81d$, which at $d=30$ exceeds their
own $5000$-sample budget (Appendix~\ref{app:queries}). \emph{One sample of each oracle per iteration
suffices}: $n_F=n_K=1$ is admissible, though not the best operating point in a performative problem,
since $\mathsf N$ is flat in $n_F$ while the number of deployed decisions $3T$ falls as $1/n_F$
(\S\ref{sec:deploy}).

\begin{table}[t]
\centering\footnotesize
\setlength{\tabcolsep}{3pt}
\resizebox{\textwidth}{!}{%
\begin{tabular}{@{}llccl@{}}
\toprule
Method & Setting & Complexity & Probes & Additional structure\\
\midrule
Two-point ZO \citep{ghadimi2013stochastic} & black box & $\mathcal O(d\varepsilon^{-4})$ & $1$ & shared randomness\\
Guided ZO \citep{hikima2025guided} & perf.\ pred. & $\mathcal O(d^{4}\varepsilon^{-6})$ & $1$ & known component\\
Multi-direction ZO \citep{hikima2025estimators} & perf.\ pred. & $\mathcal O(d^{2}\varepsilon^{-6})$, $\mathcal O(d^{2}\varepsilon^{-5})$ & $\Omega(d)$ & $Q$-Lipschitz ($\varepsilon^{-5}$)\\
PZOS \citep{fischer2026blackbox} & MPEC & $\mathcal O(d^{3/2}\varepsilon^{-5})$, Goldstein & $1$ & bounded $\nabla f$, exact $\nabla_x f$\\
PePG \citep{basu2025pepg} & perf.\ RL & $\Omega(|\mathcal S||\mathcal A|^2\varepsilon^{-2}(1{-}\gamma)^{-3})$ & n/a & response model\\
\midrule
\piso (this work) & all three & $\mathcal O(d^{2}\varepsilon^{-6})$, $\mathcal O(d^{2}\varepsilon^{-5})$ & $1$ & split \eqref{eq:split}, $\gK$ oracle\\
\bottomrule
\end{tabular}}
\caption{Complexities as reported by each source; criteria and oracle models differ, so only the
last two columns compare directly. Against the guided line we improve $d^4\to d^2$; against the
frontier we match $d^2$ with $L\to L_v$ and one probe per update; PZOS has a Goldstein guarantee we
lack but assumes bounded gradients and a deterministic follower; PePG requires a differentiable
response model.}
\label{tab:comparison}
\end{table}

\paragraph{The weight $\gamma$ and the smoothness gain.} The analysis assigns each parameter a role
and a prescription (Appendix~\ref{app:params}); we single out $\gamma$ because it governs the constant
that separates this work from prior methods. The drift term in Lemma~\ref{lem:structure}(ii) is, by
\eqref{eq:residual}, exactly
\begin{equation}\label{eq:driftsplit}
v_{k+1}-v_k=\mathsf d^U_k+(1-\gamma)\,\mathsf d^K_k,
\qquad \mathsf d^U_k:=\gU(x_{k+1})-\gU(x_k),\quad \mathsf d^K_k:=\gK(x_{k+1})-\gK(x_k).
\end{equation}

\begin{proposition}[Per-step drift-minimizing weight]\label{prop:gamma}
Fix an iteration $k$ and the displacement $x_{k+1}-x_k$, and write
$\mathsf A_k:=\Ek\|\mathsf d^U_k\|^2$, $\mathsf B_k:=\Ek\inner{\mathsf d^U_k}{\mathsf d^K_k}$,
$\mathsf C_k:=\Ek\|\mathsf d^K_k\|^2$, with $\varphi_k:=\mathsf B_k/\sqrt{\mathsf A_k\mathsf C_k}$.
If $\mathsf C_k>0$, the per-step drift $\Ek\|v_{k+1}-v_k\|^2=\mathsf A_k+2(1-\gamma)\mathsf B_k
+(1-\gamma)^2\mathsf C_k$ is a strictly convex quadratic in $\gamma$, minimized at
$\gamma_k^\star=1+\mathsf B_k/\mathsf C_k$ with minimum value $\mathsf A_k(1-\varphi_k^2)$. At
$\gamma=1$ the drift is $\mathsf A_k$, the increment of $\gU$ alone; at $\gamma=0$ it is
$\Ek\|\nabla F(x_{k+1})-\nabla F(x_k)\|^2$, which the black-box analysis bounds by
$L^2\eta^2\Ek\|G_k\|^2$; equality need not hold, so $L_v(0)\le L$ rather than $L_v(0)=L$.
\end{proposition}

Two qualifications apply. $L_v$ is a \emph{uniform} bound on the drift ratio along a trajectory
that itself depends on $\gamma$, so the proposition identifies the drift-minimizing weight at each
step rather than the minimizer of $L_v$; and the method needs $\gamma>0$, which
$\gamma_k^\star$ can fail when $\mathsf B_k<-\mathsf C_k$. With those caveats the gain has two
parts. Moving from $\gamma=0$ to $\gamma=1$ replaces the increment of the full gradient, at most
$L\eta\|G_k\|$, by that of the opaque component alone, which in performative prediction scales with
the performative sensitivity and can be far smaller; this is the substitution $L\to L_v$ of
Corollary~\ref{cor:B}, which \citet{hikima2025guided} and \citet{fischer2026blackbox} do not
obtain because their estimators are redrawn each iteration and their rates carry $L$. Moving from
$\gamma=1$ to $\gamma_k^\star$ reduces the per-step drift by a further factor $1-\varphi_k^2$ when
the two components' increments are correlated: $\gamma$ is the least-squares regression coefficient of $-\mathsf d^U_k$ on $\mathsf d^K_k$,
a control-variate coefficient built from uncentred second moments rather than a covariance in the
strict sense, and neither endpoint is drift-minimizing unless $\varphi_k=0$. The remaining parameters have prescriptions of the same
character, \eqref{eq:rhostar} for $\rho$ and Appendix~\ref{app:params} for $\lambda$ and $\rho_K$,
each settable once $\gamma$ is fixed.

\subsection{Deployments, and how the batch sizes control them}\label{sec:deploy}

In a performative problem a query is an intervention: evaluating $\widehat F$ at a point requires
\emph{deploying} that decision and observing outcomes under it. Each iteration deploys three
decisions, $x_k$ for the known-component oracle and $x_k\pm\mu_ku_k$ for the two function values,
so the total is $3T$ regardless of the batch sizes. Deployments and samples are therefore distinct
resources, and where a decision is a price schedule or a lending rule, the number of distinct
decisions released is the cost that matters.

The two batch sizes act on this cost in the same direction. Averaging $n_F$ samples per query point
divides $\sigma_{D,k}^2$ by $n_F$ and permits a proportionally larger $\beta$; averaging $n_K$
samples of $\ghK$ permits a larger $\beta_K$. Since $T=\Theta(\Delta_v/(\beta\varepsilon^2))$, both
enlarge the step and reduce $T$. How far this goes is governed by two caps on $\beta$ that sit a
factor $d$ apart. Below $\beta_1:=\chi_1^2/\varsigma_6^2=\Theta(1/d^2)$ the damped second-moment
bias $11\beta\Sigma_b^2$ stays dominated by the undamped mean bias, the probe radius balances as in
Corollary~\ref{cor:B}, and $\mathsf N$ is flat in $\beta$; between $\beta_1$ and the bias-analysis
cap $\beta_2:=1/(3d)$ the second-moment bias takes over the radius balance and $\mathsf N$ grows
linearly in $\beta$.
\begin{center}\small
\setlength{\tabcolsep}{3.5pt}
\begin{tabular}{@{}llccc@{}}
\toprule
& $\beta$ & deployments $3T$ & function evals.\ $2n_FT$ & known samples $n_KT$\\
\midrule
$n_F=n_K=1$ & $\varepsilon^4/C_1$ & $\Theta(d^2\varepsilon^{-6})$ & $\Theta(d^2\varepsilon^{-6})$ & $\Theta(d^2\varepsilon^{-6})$\\
batches raised until $\beta=\beta_1$ & $\Theta(1/d^2)$ & $\Theta(d^2\varepsilon^{-2})$ & $\Theta(d^2\varepsilon^{-6})$ & $\Theta(d\,\varepsilon^{-4})$\\
batches raised until $\beta=\beta_2$ & $\Theta(1/d)$ & $\Theta(d\,\varepsilon^{-2})$ & $\Theta(d^3\varepsilon^{-6})$ & $\Theta(d\,\varepsilon^{-4})$\\
\bottomrule
\end{tabular}
\end{center}
So enlarging the batches up to $\beta_1$ is free in function evaluations while cutting the
deployment count by $\varepsilon^{-4}$, from $\Theta(d^2\varepsilon^{-6})$ to
$\Theta(d^2\varepsilon^{-2})$, plus the additive $\gamma L_K\Delta_F\varepsilon^{-2}$ of the
known-component branch (Appendix~\ref{app:deploy}). Going further, to $\beta_2$, removes one more
factor $d$ from the deployment count at the price of a factor $d$ in observations, a trade a
practitioner may well accept when a deployment is an intervention on a population. Three points are
worth separating. The analysis permits $n_F=n_K=1$, so no minibatching is required, but $n_F=1$ is
the wrong operating point in a performative problem, since it maximizes the number of distinct
decisions released for the same observations. The two caps have different origins, and neither is the stability boundary: $\beta_2=1/(3d)$ is the
limit of the sharp bias-vector domination argument (Appendix~\ref{app:bias}), while the sketch
contraction itself holds up to $2/c_d$ and Theorem~\ref{thm:main} uses $\beta\le1/c_d$; $\beta_1$
reflects the generic bound $\Sigma_b^2\le\tfrac14L^2\mu^2\varsigma_6^2$, which we show is loose by a
factor $d$ on examples; a general bound $\Sigma_b^2=\mathcal O(L^2\mu^2d^2)$ would move
$\beta_1$ to $\Theta(1/d)$ and make the last row free as well, and we leave it open. And a
multi-direction method cannot trade this way at all: its $N$ probes are $2N$ \emph{distinct}
decisions per update, which repeated sampling cannot replace, so with $N=\Theta(d^2\varepsilon^{-4})$
its deployment count equals its sample count, $\Theta(d^2\varepsilon^{-6})$.

\section{Related work}\label{sec:related}

\paragraph{Methods that exploit a known gradient component.} Using a cheap or partial signal
alongside an expensive one is the subject of \emph{optimization with access to auxiliary
information} \citep{chayti2024auxiliary} and the helper framework extending it to second order
\citep{chayti2024unified}; our setting is an instance in which the complement of the auxiliary
signal is accessible only through function values. Within performative prediction,
\citet{ray2022decision} and \citet{liu2024twotimescale} use one-point estimators with constants in
$\sup|f|$, \citet{hikima2025zeroth} replace that by a variance bound, and
\citet{hikima2025guided,hikima2025estimators} are discussed above. Within MPECs,
\citet{cui2023mpec,cui2023hierarchicalgames} apply zeroth-order tools to the whole composite
objective, \citet{fischer2026blackbox} estimate only the follower Jacobian, and \citet{sow2022pzobo}
smooth the inner solution map. All of these estimate a gradient or a Jacobian at the current iterate,
leaving a second moment of order $c_d$ redrawn each iteration: in the bound of
\citet{fischer2026blackbox} the leading $\mathcal O(d)$ term is common to their partial estimator and
to the black-box baseline. Tracking a residual makes the direction error contract instead, and needs
no composite structure, which is why it extends to performative RL.

\paragraph{Performative prediction and RL.} \citet{perdomo2020performative} introduced performative
prediction and the stability/optimality distinction; \citet{mendlerdunner2020stochastic},
\citet{mofakhami2023neural}, \citet{li2024stochastic} and \citet{jagadeesan2022regret} study
variants. Methods targeting performative \emph{optima} assume a parametric distribution map
\citep{izzo2021perfgd,izzo2022stateful,miller2021outside} or plug in an estimate of it
\citep{lin2024plugin}; \citet{cyffers2026optimal} study the choice of regularizer under shift.
Strategic classification \citep{levanon2021strategic,cyffers2024optimal} and Stackelberg games with
algorithmic agents \citep{frankel2025finite} supply two of our settings. In performative RL most work
targets \emph{stability} \citep{mandal2023performative,rank2024gradually,mandal2025linear};
\citet{basu2025pepg} give the first policy-gradient method for optimality, whose guarantee needs a
softmax PeMDP with a known feature map so that the response derivatives are in closed form. Where
such a model exists its use is preferable; the settings that motivate our work are those where it
does not, and \piso assumes only that the response can be queried.

\paragraph{Zeroth-order estimation, nonsmoothness, and sketching.} The estimator machinery is
standard \citep{nesterov2017random,ghadimi2013stochastic,shamir2017optimal,flaxman2005online};
recent work studies comparison-based random search
\citep{chayti2025randomsearch,elbakkali2026comparison} and heavy-tailed value noise
\citep{elbakkali2026heavytailed}, all treating the objective as a single black box. The mean-versus-
second-moment distinction for the bias is due to \citet{berahas2022theoretical}. For nonsmooth
objectives Goldstein stationarity is the appropriate target
\citep{clarke1990optimization,lin2022gradientfree,kornowski2024optimal}, and
\citet{fischer2026blackbox} give a notion tailored to a smooth leader; equilibrium response maps are
generally only Lipschitz \citep{lindberg2015traffic}, so our smooth analysis is complementary there.
Recursion \eqref{eq:teaser} is structurally the EF21 mechanism \citep{richtarik2021ef21} and a
sketch-and-project iteration \citep{gower2015randomized}; \citet{huang2022accelerated} study
zeroth-order momentum without identifying it as a sketch step size.

\section{Experiments}\label{sec:experiments}

We evaluate on five settings spanning the three areas: multiproduct pricing (six weekly instances,
$d=10$, $40$ buyers, $100$ simulations, $5000$ samples) and strategic classification (credit data of
\citet{yeh2009comparisons}, $d=12$, five response strengths) for performative prediction; the
$8\times8$ gridworld of \citet{mandal2023performative} for performative RL; and the $225$ routing
and $427$ security instances of \citet{fischer2026blackbox} for bilevel games. Baselines are the
strongest published method for each area operating under \eqref{eq:split}: GZO-NS and GZO-HS
\citep{hikima2025guided}, ZO-OGVR \citep{hikima2025zeroth}, a Gaussian two-point method ZO-TG, ZOS
and PZOS \citep{fischer2026blackbox}, and a vanilla policy gradient. On the RL setting we
additionally report PePG \citep{basu2025pepg} as a reference with privileged information: it is given
exact reward and transition response Jacobians, so it differentiates the environment's response
instead of querying it, which is the assumption our setting is defined by not making. Its curve is
not a bound on what other methods can attain, but it indicates what a differentiable response model
is worth. We run \piso and \pisotwo with
Gaussian probes and with a variant cycling through the columns of a random orthonormal frame. The latter has the
sphere family as its one-step marginal but is dependent across consecutive steps, so
Assumption~\ref{as:dir} does not cover it and the theory applies to it only heuristically. Full protocols, the remaining figures, a metrics table and all
parameter ablations are in Appendix~\ref{app:experiments}.

\begin{figure}[t]
\centering
\includegraphics[width=0.86\textwidth]{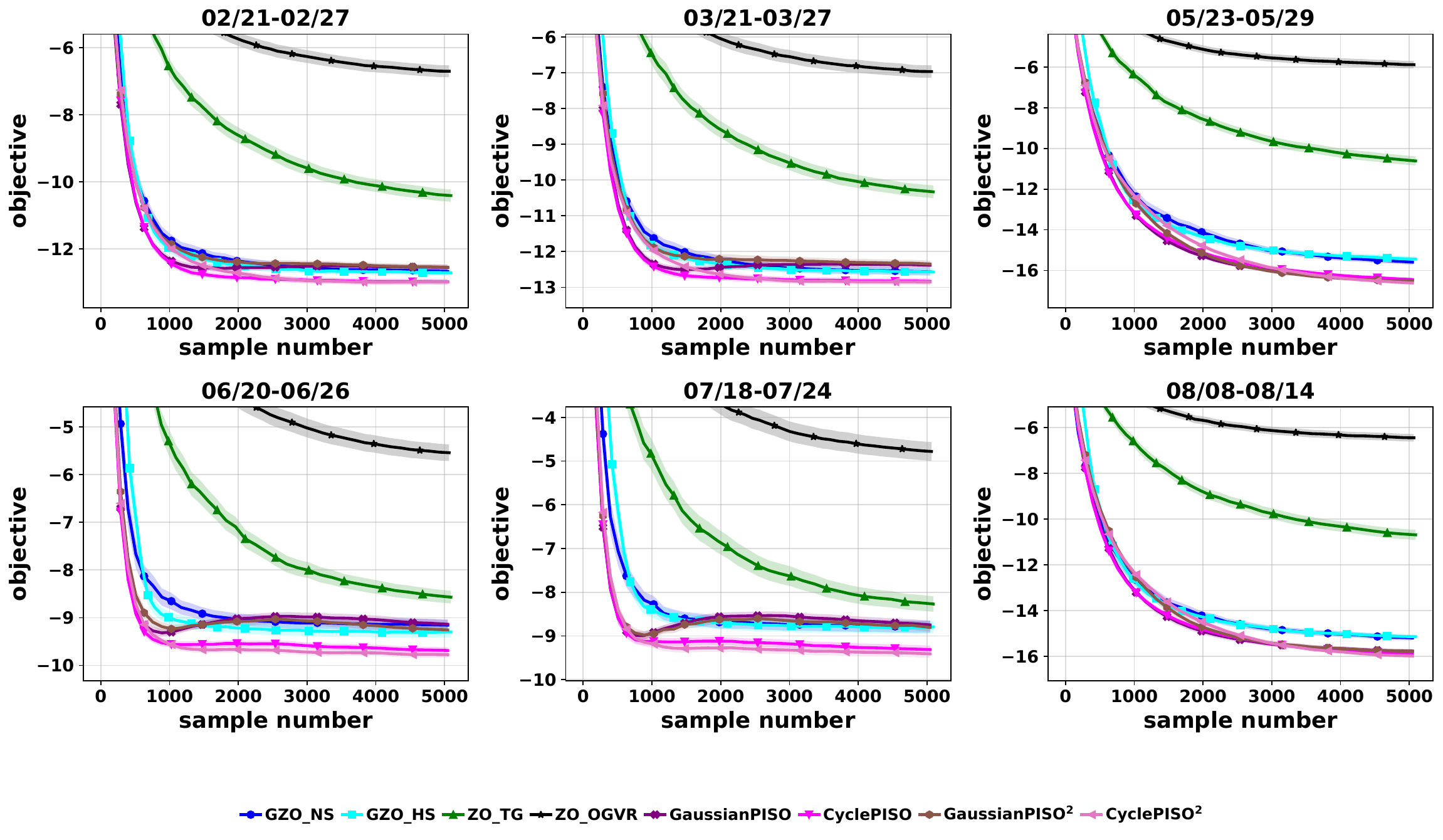}
\caption{Multiproduct pricing, six weekly instances. Objective (lower is better) against samples
consumed; means over $100$ simulations with one-standard-error bands. All four \piso variants
separate from the guided baselines early in the run and the cycling variants attain the lowest
objective on every instance. ZO-TG and ZO-OGVR, which do not use the computable component, remain
far behind.}
\label{fig:pricing}
\end{figure}

\begin{figure}[t]
\centering
\begin{minipage}{0.46\textwidth}\centering
\includegraphics[width=\textwidth]{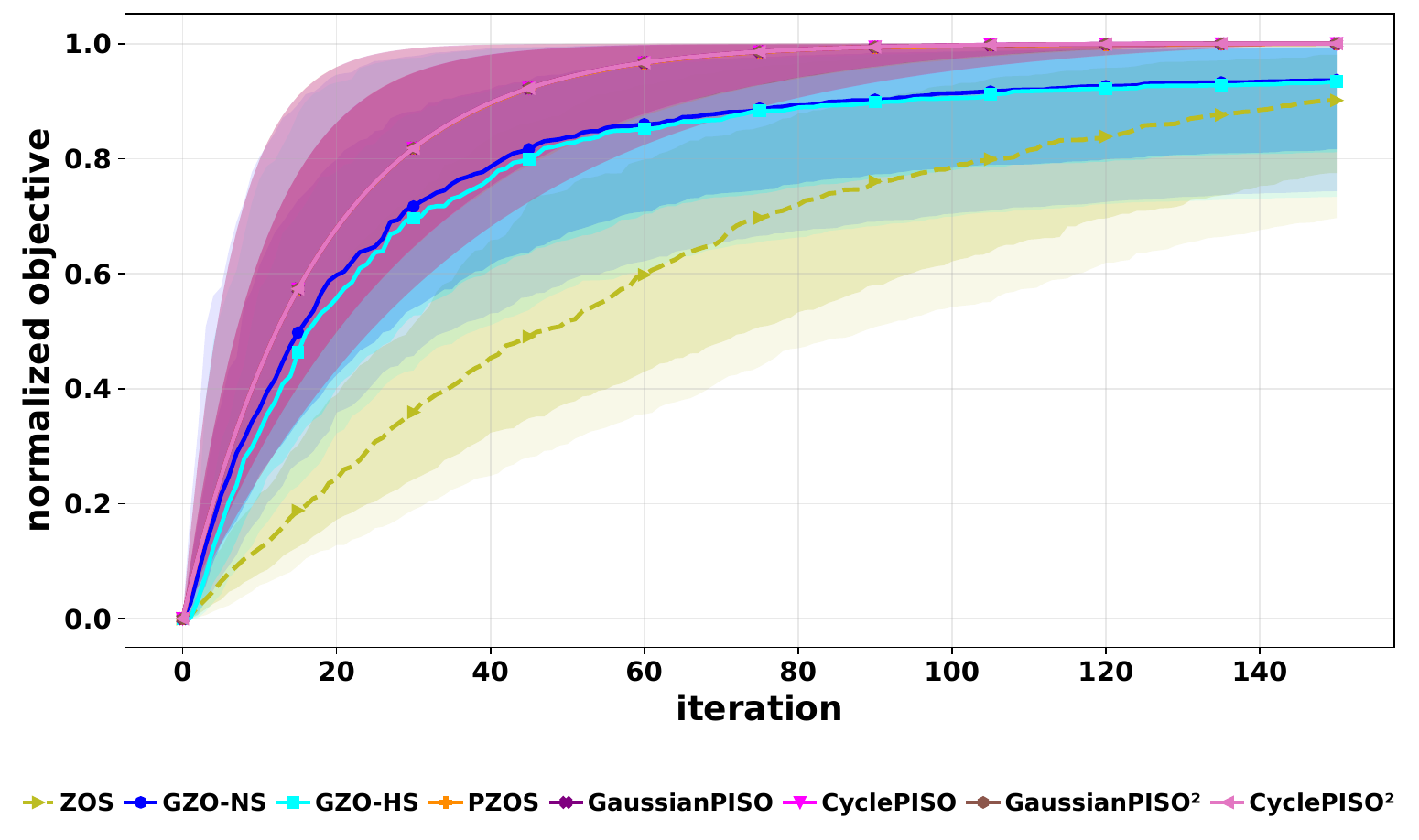}\\{\small(a) routing, $225$ instances}
\end{minipage}\hfill
\begin{minipage}{0.46\textwidth}\centering
\includegraphics[width=\textwidth]{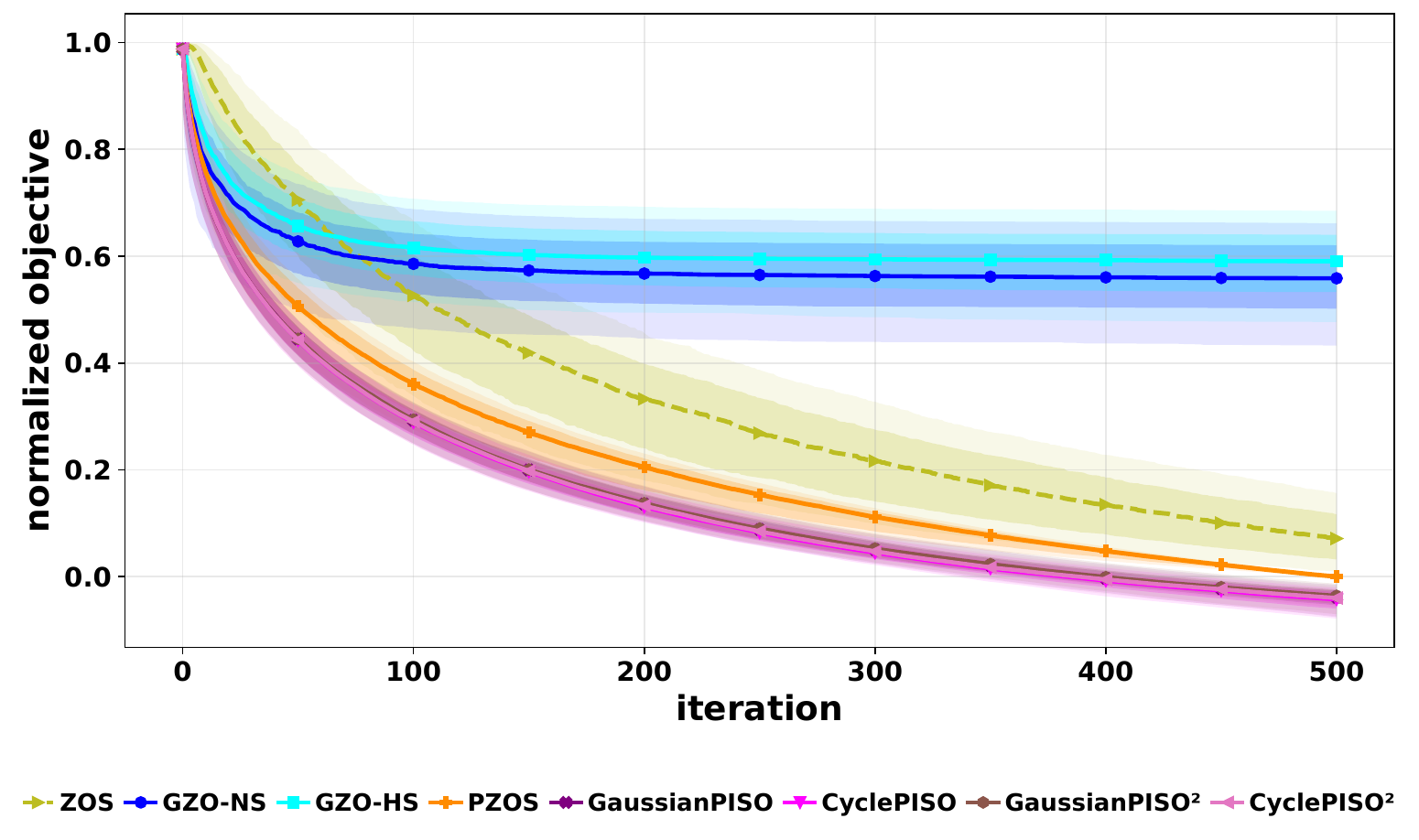}\\{\small(b) security, $427$ instances}
\end{minipage}
\caption{Bilevel games with an opaque lower level; per-instance normalized objective against
iteration, median with interquartile and $10$--$90$ percentile bands. (a) Routing is a maximization
problem: \piso attains the same endpoint as PZOS but reaches it well before the guided baselines.
(b) Security is a minimization problem: \piso separates from PZOS from the first iterations and ends
below it.}
\label{fig:bilevel}
\end{figure}

On pricing (Figure~\ref{fig:pricing}) the ordering is stable across all six instances: the cycling
variants of \piso and \pisotwo are best, the Gaussian variants follow, then the guided baselines, then
the two methods that ignore $\gK$. Strategic classification is more clear-cut: every \piso variant
attains a test loss between $0.618$ and $0.641$ at every response strength against $0.821$ for the
best baseline, with accuracy higher by four to ten points, and the methods that ignore $\gK$ degrade as
the response strengthens while the \piso variants are flat (Table~\ref{tab:classification}). On the
bilevel games (Figure~\ref{fig:bilevel}) two outcomes are worth separating: on routing \piso and PZOS
are indistinguishable at convergence, confirmed by per-instance scatter plots lying on the diagonal,
and what \piso gains is speed; on security \piso is below PZOS throughout and across the dimension
range $d\in[100,300]$. On performative RL, \piso is best among methods that query the environment,
with the skyline faster still (Figure~\ref{fig:rl}).

The ablations (Appendix~\ref{app:ablations}) test the analysis directly. The $\rho$ sweep locates the
optimum at $\rho=0.92$ in all four variants, the nearest grid point to the predicted
$\rho^\star\in\{0.900,0.917\}$, and every value at or below the predicted neutrality threshold
$1-2/c_d$ descends and then \emph{rises}, which is the signature of a residual iteration whose sketch
step size has left its stability window. The $\rho_K$ sweep is informative in the other direction:
with the growing batches used, $\rho_K=0$ is already among the best choices, which is what the
analysis predicts when $n_K$ is large. The components ablation confirms that neither the known
component alone nor an uncentred zeroth-order correction is stable on its own.

\section{Conclusion}\label{sec:conclusion}

We introduced PISO for stochastic optimization with a partially accessible gradient, a structure shared by performative prediction, performative reinforcement learning, and bilevel optimization with a black-box lower level. PISO tracks the residual and centres each zeroth-order observation on its current estimate, making the probing variance depend on the tracking error rather than the residual’s magnitude. Our analysis establishes nonconvex convergence guarantees with one probe direction per update and no uniform bound on the unavailable gradient component. These guarantees capture gains from slowly varying residuals, while full-objective smoothness still controls finite-difference bias. Experiments across five settings demonstrate improvements in solution quality or convergence speed over baselines with comparable oracle access.

Our analysis assumes a smooth objective, which equilibrium response maps need not produce. Extending residual tracking to Goldstein stationarity would broaden its applicability to nonsmooth bilevel problems. Analysing the empirically strong orthogonal cycling variants, whose dependent probes fall outside our assumptions, remains open.

\subsection*{AI use statement}

We used a large language model as an assistant for restructuring and editing the manuscript, checking
and tightening proofs, drafting parts of the experiment code, and suggesting related work. The authors
verified all such output. The research ideas, the method and the main results are the authors' own,
and the authors take full responsibility for the content of this work.

\section*{Reproducibility statement}
All proofs are in Appendices~\ref{app:structure}--\ref{app:queries}, with every constant derived
explicitly and the assumptions restated in Appendix~\ref{app:notation}. Experimental protocols,
problem definitions, baseline configurations, and per-setting hyperparameters are in
Appendix~\ref{app:experiments}. The classification data is the public credit dataset of
\citet{yeh2009comparisons}.

\bibliographystyle{iclr2027_conference}
\bibliography{refs}

\appendix
\clearpage
\section*{Appendix Contents}
\begingroup
\setcounter{tocdepth}{2}
\makeatletter
\@starttoc{apc}
\makeatother
\endgroup

\makeatletter
\let\PISO@orig@addcontentsline\addcontentsline
\renewcommand{\addcontentsline}[3]{%
  \PISO@orig@addcontentsline{#1}{#2}{#3}%
  \def\PISO@contentsfile{#1}%
  \def\PISO@maintoc{toc}%
  \ifx\PISO@contentsfile\PISO@maintoc
    \PISO@orig@addcontentsline{apc}{#2}{#3}%
  \fi
}
\makeatother

\newpage

\section{Overview of the appendix and summary of notation}\label{app:notation}

Appendix~\ref{app:structure} proves the structural identity (Lemma~\ref{lem:structure}) and explains
what centring buys. Appendix~\ref{app:families} collects the properties of the probe distribution and
proves the constants of Table~\ref{tab:families}. Appendix~\ref{app:bias} treats the
finite-difference bias, separating its mean from its second moment; this separation is what controls
the dimension dependence of the rates. Appendix~\ref{app:tracking} proves the tracking lemma, the
technical core of the analysis. Appendix~\ref{app:main} proves Theorem~\ref{thm:main}, and
Appendix~\ref{app:complexity} the complexity corollaries. Appendix~\ref{app:params} proves the
statements of Appendix~\ref{app:params}, and Appendix~\ref{app:queries} the probe-count claim of
Remark~\ref{rem:queries}.

Every symbol below has one meaning throughout. Pairs that look similar are listed together so they
can be checked at a glance.

\begin{center}\small
\setlength{\tabcolsep}{5pt}
\begin{tabular}{@{}llll@{}}
\toprule
$x_k$ & iterate & $u_k,s_k$ & probe and score\\
$m_k$ & residual state & $P_k=s_ku_k^\top$ & sketch matrix\\
$R_k$ & residual observation, \eqref{eq:centred} & $c_d$ & sketch constant, \eqref{eq:D2app}\\
$G_k$ & search direction & $\varsigma^2,\varsigma_4^2,\varsigma_6^2,\varsigma_8^2$ & probe moments, \eqref{eq:D3app}\\
$v(x)$ & residual, \eqref{eq:residual} & $\nu,\kappa$ & smoothing measure and scale, \eqref{eq:D4app}\\
$E_k=m_k-v(x_k)$ & tracking error & $\chi_1,\chi_2$ & smoothing moments, \eqref{eq:chidef}\\
$\Delta_k=G_k-\nabla F(x_k)$ & direction error & $b_k$ & finite-difference bias (scalar)\\
$\Phi_k$ & Lyapunov function & $\bar b_k,\bar B_k$ & bias vectors, \eqref{eq:biasvec}\\
$\omega$ & weight on $\|E_k\|^2$ in $\Phi_k$ & $\Sigma_{b,k}^2$ & $\Ek[b_k^2\|s_k\|^2]$\\
$\rho,\ \beta=1-\rho$ & momentum, sketch step & $\delta_k$ & known-oracle error, scaled by $\gamma$\\
$\lambda,\ \bt=1-(1-\lambda)\rho$ & mixing, readout gain & $\sigma_{K,k}^2,\sigma_{D,k}^2$ & oracle variances\\
$\theta(b)=1-2b+c_db^2$ & contraction factor & $\xi_k$ & directional-oracle noise\\
$\gamma$ & known-component weight & $\delta_{\mathrm{tol}}$ & solver tolerance (Regime C)\\
$\eta,\mu_k$ & step size, probe radius & $n_F,n_K$ & per-iteration batch sizes\\
$L,L_K,L_v,Q$ & smoothness constants & $\mathsf N$ & total function evaluations\\
\bottomrule
\end{tabular}
\end{center}

We abbreviate $\Ek[\,\cdot\,]:=\E[\,\cdot\mid\Fk]$, where $\Fk$ is the $\sigma$-field generated by
all randomness drawn before iteration $k$, so that $x_0,\dots,x_k$, $m_0,\dots,m_k$ and
$h_0,\dots,h_k$ are all $\Fk$-measurable while the iteration-$k$ draws $(u_k,s_k)$, $\ghK(x_k)$ and
the function-value noise are not. We write
$v_k:=v(x_k)$ and $\nabla F_k:=\nabla F(x_k)$, and set $\theta:=\theta(\bt)$.

\paragraph{Standing assumptions.} We restate them for reference.

\begin{assumption}[Smoothness; restatement of Assumption~\ref{as:smooth}]\label{as:smoothapp}
$F$ is $L$-smooth and bounded below by $F^\star$, with $\Delta_F:=F(x_0)-F^\star$; and $\gK$ is
$L_K$-Lipschitz. We write $L_v$ for the smallest constant satisfying
\begin{equation}\label{eq:pathwise}
\Ek\|v_{k+1}-v_k\|^2 \;\le\; L_v^2\,\eta^2\,\Ek\|G_k\|^2 .
\end{equation}
\end{assumption}

Finiteness of $L_v$ is implied rather than assumed: $v=\nabla F-\gamma\gK$ is
$(L+\gamma L_K)$-Lipschitz, so \eqref{eq:pathwise} holds with $L_v\le L+\gamma L_K$ because
$\|x_{k+1}-x_k\|=\eta\|G_k\|$. Using the smallest valid constant along the trajectory rather than the
global Lipschitz modulus only strengthens the results.

\begin{assumption}[Oracles; restatement of Assumption~\ref{as:oracles}]\label{as:oraclesapp}
$\Ek[\ghK(x_k)]=\gK(x_k)$ and $\Ek\|\ghK(x_k)-\gK(x_k)\|^2\le\sigma_{K,k}^2$, and $\ghK(x_k)$ is drawn
independently of $(u_k,s_k)$; consequently $\Ek[\delta_k]=0$ and
$\Ek\|\delta_k\|^2\le\gamma^2\sigma_{K,k}^2$. In \eqref{eq:twopoint}, $b_k$ is a deterministic
function of $(x_k,u_k,\mu_k)$, and
$\E[\xi_k\mid\Fk,u_k,s_k,\delta_k]=0$ with $\E[\xi_k^2\mid\Fk,u_k,s_k]\le\sigma_{D,k}^2$.
\end{assumption}

Note that $\xi_k$ is required to be conditionally mean-zero, not independent of $u_k$. This matters:
in a performative problem the noise of $\widehat F(x_k\pm\mu_ku_k)$ depends on $u_k$, because a
different decision is deployed at each of the two query points.

\begin{assumption}[Probe distribution; restatement of Assumption~\ref{as:dir}]\label{as:dirapp}
The pair $(u_k,s_k)$ is drawn i.i.d.\ given $\Fk$ and independently of $\delta_k$; no independence
from $\xi_k$ is assumed, only the conditional mean and variance conditions of
Assumption~\ref{as:oraclesapp}. It satisfies, for every $\Fk$-measurable $z$,
\begin{align}
\text{(D1)}\quad & \Ek\big[s_ku_k^\top\big]=\Id, \label{eq:D1app}\\
\text{(D2)}\quad & \Ek\big\|\inner{z}{u_k}s_k\big\|^2\le c_d\|z\|^2, \label{eq:D2app}\\
\text{(D3)}\quad & \Ek\|s_k\|^2\le\varsigma^2,\quad
\Ek\big[\|s_k\|^2\|u_k\|^2\big]\le\varsigma_4^2,\quad
\Ek\big[\|u_k\|^4\|s_k\|^2\big]\le\varsigma_6^2,\quad
\Ek\big[\|u_k\|^6\|s_k\|^2\big]\le\varsigma_8^2, \label{eq:D3app}\\
\text{(D4)}\quad & \E\Big[\tfrac{F(x+\mu u)-F(x-\mu u)}{2\mu}\,s\Big]=\nabla F_{\nu,\kappa\mu}(x),
\qquad F_{\nu,\tau}(x):=\E_{w\sim\nu}\big[F(x+\tau w)\big], \label{eq:D4app}
\end{align}
for some symmetric probability measure $\nu$ on $\R^d$ and scale $\kappa>0$. Given (D4) we set
\begin{equation}\label{eq:chidef}
\chi_1 \;:=\; \kappa\,\E_\nu\|w\| , \qquad \chi_2 \;:=\; \kappa^2\,\E_\nu\|w\|^2 .
\end{equation}
\end{assumption}

Only (D1) and (D2) enter the contraction; (D3) and (D4) enter only through noise and bias constants.

\section{The structural identity}\label{app:structure}

\begin{proof}[Proof of Lemma~\ref{lem:structure}]
\emph{(i) Two-timescale readout.} Substituting $m_{k+1}=\rho m_k+(1-\rho)R_k$ into the readout,
\[
\lambda R_k+(1-\lambda)m_{k+1}
=\lambda R_k+(1-\lambda)\big[\rho m_k+(1-\rho)R_k\big]
=(1-\lambda)\rho\,m_k+\big[\lambda+(1-\lambda)(1-\rho)\big]R_k .
\]
Expanding the bracket,
$\lambda+(1-\lambda)(1-\rho)=\lambda+1-\rho-\lambda+\lambda\rho=1-\rho(1-\lambda)=\bt$, while the
coefficient of $m_k$ is $(1-\lambda)\rho=1-\bt$. This proves (i).

\emph{(ii) State recursion.} We first identify the innovation. Substituting
$\gamma\ghK(x_k)=\gamma\gK(x_k)+\delta_k$ and \eqref{eq:twopoint} into line 4 of
Algorithm~\ref{alg:piso}, and using $\nabla F_k-\gamma\gK(x_k)=v_k$,
\begin{align}
D_k-\inner{\gamma\ghK(x_k)+m_k}{u_k}
&=\inner{\nabla F_k}{u_k}+b_k+\xi_k-\inner{\gamma\gK(x_k)+\delta_k+m_k}{u_k} \notag\\
&=\inner{\nabla F_k-\gamma\gK(x_k)-m_k-\delta_k}{u_k}+b_k+\xi_k \notag\\
&=\inner{-E_k-\delta_k}{u_k}+b_k+\xi_k , \label{eq:innov}
\end{align}
where the last line uses $E_k=m_k-v_k$. Therefore
\begin{equation}\label{eq:Rminusm}
R_k-m_k=\big(\inner{-E_k-\delta_k}{u_k}+b_k+\xi_k\big)s_k .
\end{equation}
Since $m_{k+1}-m_k=\beta(R_k-m_k)$ and $\inner{E_k}{u_k}s_k=s_ku_k^\top E_k=P_kE_k$,
\begin{equation}\label{eq:mkv}
m_{k+1}-v_k=(\Id-\beta P_k)E_k+\beta\big(\!-\!\inner{\delta_k}{u_k}+b_k+\xi_k\big)s_k .
\end{equation}
Adding and subtracting $v_{k+1}$ and using $E_{k+1}=m_{k+1}-v_{k+1}$ gives (ii).

\emph{(iii) Direction error.} Identity \eqref{eq:innov} does not involve $\beta$, so replacing
$\beta$ by $\bt$ in the derivation of \eqref{eq:mkv} and using (i) gives
\begin{equation}\label{eq:tmv}
\tilde m_{k+1}-v_k=(\Id-\bt P_k)E_k+\bt\big(\!-\!\inner{\delta_k}{u_k}+b_k+\xi_k\big)s_k ,
\end{equation}
where $\tilde m_{k+1}:=\lambda R_k+(1-\lambda)m_{k+1}$. Since $G_k=\gamma h_{k+1}+\tilde m_{k+1}$ and
$\nabla F_k=\gamma\gK(x_k)+v_k$,
\[
\Delta_k=\gamma\big(h_{k+1}-\gK(x_k)\big)+\big(\tilde m_{k+1}-v_k\big)
=\big[(1-\beta_K)\zeta_k+\beta_K\delta_k\big]+\big(\tilde m_{k+1}-v_k\big),
\]
using \eqref{eq:hshift} for the first bracket, and grouping the two $\delta_k$ terms as
$\beta_K\delta_k-\bt\inner{\delta_k}{u_k}s_k=(\beta_K\Id-\bt P_k)\delta_k$ with \eqref{eq:tmv} gives (iii). Note that $\delta_k$ no longer appears standing alone: with the
fresh estimate in the step one would obtain the extra term $\delta_k$, and hence
$\theta\gamma^2\sigma_{K,k}^2$ in $\Ek\|\Delta_k\|^2$.

\emph{(iv) Known-component recursion.} From $h_{k+1}=(1-\beta_K)h_k+\beta_K\ghK(x_k)$,
\begin{equation}\label{eq:hshift}
\gamma\big(h_{k+1}-\gK(x_k)\big)
=(1-\beta_K)\gamma\big(h_k-\gK(x_k)\big)+\beta_K\gamma\big(\ghK(x_k)-\gK(x_k)\big)
=(1-\beta_K)\zeta_k+\beta_K\delta_k ,
\end{equation}
which is the substitution used in (iii). Adding
$\gamma\big(\gK(x_k)-\gK(x_{k+1})\big)$ to both sides of \eqref{eq:hshift} gives
$\zeta_{k+1}=\gamma(h_{k+1}-\gK(x_{k+1}))=(1-\beta_K)\zeta_k+\beta_K\delta_k
+\gamma\big(\gK(x_k)-\gK(x_{k+1})\big)$, which is (iv).

We stress the indexing: $\zeta_k$ is defined from $h_k$, not $h_{k+1}$, so that it is
$\Fk$-measurable. The state $h_{k+1}$ used in $G_k$ involves the step-$k$ sample $\ghK(x_k)$, and
\eqref{eq:hshift} is what isolates that sample as the mean-zero term $\beta_K\delta_k$.
\end{proof}

\begin{lemma}[Contraction factor]\label{lem:theta}
For every $\Fk$-measurable $z$ and every $b\in\R$,
$\Ek\|(\Id-bP_k)z\|^2\le\theta(b)\|z\|^2$ with $\theta$ as in \eqref{eq:theta}, with equality
whenever \eqref{eq:D2app} holds with equality. Consequently $\theta(b)<1$ if and only if $0<b<2/c_d$, with $\theta(2/c_d)=1$ (mean-square neutral)
and $\theta(b)>1$ for $b>2/c_d$ or $b<0$; $\theta$ is minimized at $b=1/c_d$; and
$\theta(1/c_d)=1-1/c_d$.
\end{lemma}

\begin{proof}
Expanding and using \eqref{eq:D1app} for the cross term and \eqref{eq:D2app} for the last,
\[
\Ek\|(\Id-bP_k)z\|^2=\|z\|^2-2b\,z^\top\Ek[P_k]z+b^2\,\Ek\|\inner{z}{u_k}s_k\|^2
\le\|z\|^2-2b\|z\|^2+b^2c_d\|z\|^2 .
\]
The quadratic $\theta(b)-1=b(c_db-2)$ is negative exactly for $0<b<2/c_d$, giving the stated
window; its minimizer is $b=1/c_d$ with value $1-1/c_d$.
\end{proof}

\subsection{What centring buys}\label{app:centring}

We justify \S\ref{app:centring}. Consider the uncentred observation
$R_k^{\mathrm{fresh}}=\big(D_k-\inner{\gamma\ghK(x_k)}{u_k}\big)s_k$, in which $m_k$ does not appear
inside the inner product. Repeating the computation \eqref{eq:innov} without the $m_k$ term gives
$D_k-\inner{\gamma\ghK(x_k)}{u_k}=\inner{v_k-\delta_k}{u_k}+b_k+\xi_k$, hence
$R_k^{\mathrm{fresh}}=P_kv_k+(-\inner{\delta_k}{u_k}+b_k+\xi_k)s_k$ and, with
$m_{k+1}=\rho m_k+(1-\rho)R^{\mathrm{fresh}}_k$,
\begin{equation}\label{eq:uncentred}
E_{k+1}=\rho E_k+\beta\big(P_k-\Id\big)v_k
+\beta\big(\!-\!\inner{\delta_k}{u_k}+b_k+\xi_k\big)s_k+(v_k-v_{k+1}) .
\end{equation}
The sketch now multiplies $v_k$ rather than $E_k$. Taking second moments, the term
$\beta(P_k-\Id)v_k$ contributes $\beta^2(c_d-1)\|v_k\|^2$ at each step, and since the remaining
contraction is only $\rho=1-\beta$, its stationary contribution to $\E\|E_k\|^2$ is of order
$\beta c_d\|v_k\|^2$. With $G_U:=\sup_x\|v(x)\|$, driving this below $\varepsilon^2$ requires
$\beta\lesssim\varepsilon^2/(c_dG_U^2)$ in addition to the restrictions of
Theorem~\ref{thm:main}, so the complexity acquires a factor $c_dG_U^2/\varepsilon^2$. The centred
recursion of Lemma~\ref{lem:structure}(ii) has $E_k$ in place of $v_k$ and therefore carries no such
term. In particular, the uncentred variant is not the $\rho\to0$ limit of \piso; it is a different
algorithm whose complexity depends on the magnitude of the opaque component.

\section{The probe distribution}\label{app:families}

\begin{table}[t]
\centering\small
\begin{tabular}{@{}lcccc@{}}
\toprule
family & $c_d$ & $\varsigma^2$ & $\chi_1$ & $\chi_2$\\
\midrule
Gaussian, $u=s\sim\mathcal N(0,\Id)$ & $d+2$ & $d$ & $\sqrt d$ & $d$\\
Sphere, $u=s=\sqrt d\,w$, $\|w\|=1$ & $d$ & $d$ & $\sqrt d$ & $d$\\
Coordinate, $u=s=\sqrt d\,e_i$ & $d$ & $d$ & $d/2$ & $d^{3/2}/3$\\
\bottomrule
\end{tabular}
\caption{Constants for three standard families. Gaussian and sphere satisfy the smoothing identity
(D4), with $\nu$ Gaussian and $\kappa=1$ for the first and $\nu$ uniform on the unit ball and
$\kappa=\sqrt d$ for the second, and therefore share $\chi_1=\sqrt d$. The coordinate family \emph{fails} (D4): its mean estimator is
the coordinate-wise finite-difference gradient, which is not the gradient of any smoothing, and it
loses a factor $\sqrt d$ in $\chi_1$. This is the structural cause of the known gap between
coordinate and random directions.}
\label{tab:families}
\end{table}

\begin{lemma}[Constants of Table~\ref{tab:families}]\label{lem:families}
Let $d\ge2$.
\begin{enumerate}[label=(\alph*),leftmargin=1.8em,itemsep=2pt]
\item \emph{Gaussian}, $u=s\sim\mathcal N(0,\Id)$. Then \eqref{eq:D1app} holds; \eqref{eq:D2app} holds
with equality and $c_d=d+2$; $\varsigma^2=d$, $\varsigma_4^2=d(d+2)\le2d^2$,
$\varsigma_6^2=d(d+2)(d+4)\le9d^3$ and $\varsigma_8^2\le54d^4$; and \eqref{eq:D4app} holds with
$\nu=\mathcal N(0,\Id)$ and $\kappa=1$, so $\chi_1\le\sqrt d$ and $\chi_2=d$.
\item \emph{Sphere}, $u=s=\sqrt d\,w$ with $w$ uniform on the unit sphere. Then \eqref{eq:D1app}
holds; \eqref{eq:D2app} holds with equality and $c_d=d$; $\varsigma^2=d$, $\varsigma_4^2=d^2$,
$\varsigma_6^2=d^3$, $\varsigma_8^2=d^4$, all exactly; and \eqref{eq:D4app} holds with $\nu$ uniform
on the unit ball and $\kappa=\sqrt d$, so $\chi_1\le\sqrt d$ and $\chi_2\le d$.
\item \emph{Coordinate}, $u=s=\sqrt d\,e_i$ with $i$ uniform on $[d]$. Then \eqref{eq:D1app} holds,
\eqref{eq:D2app} holds with equality and $c_d=d$, and the $\varsigma$ moments coincide with (b).
Property \eqref{eq:D4app} \textbf{fails}, and the direct bounds are $\chi_1=d/2$ under
$L$-smoothness and $\chi_2=d^{3/2}/3$ under a $Q$-Lipschitz Hessian.
\end{enumerate}
\end{lemma}

\begin{proof}
\emph{(a)} $\E[uu^\top]=\Id$ gives \eqref{eq:D1app}. For \eqref{eq:D2app}, Isserlis' theorem gives
$\E[\|u\|^2uu^\top]=\sum_j\E[u_j^2uu^\top]=(d+2)\Id$, so
$\E\|\inner{z}{u}u\|^2=\E[\inner{z}{u}^2\|u\|^2]=z^\top(d+2)\Id\,z=(d+2)\|z\|^2$ exactly. The chi-square
moments $\E\|u\|^{2k}=d(d+2)\cdots(d+2k-2)$ give the $\varsigma$ values, with
$d(d+2)\le2d^2$, $d(d+2)(d+4)\le9d^3$ and $d(d+2)(d+4)(d+6)\le54d^4$ for $d\ge2$. For
\eqref{eq:D4app}, Gaussian integration by parts ($\E[g(u)u_i]=\E[\partial_ig(u)]$ for differentiable
$g$ of moderate growth) gives $\E[F(x+\mu u)u]=\mu\,\E[\nabla F(x+\mu u)]$; and since the law of $u$
is symmetric, substituting $u\mapsto-u$ gives
$\E[F(x-\mu u)u]=-\E[F(x+\mu u)u]=-\mu\E[\nabla F(x+\mu u)]$. Hence
\[
\begin{aligned}
\E\Big[\tfrac{F(x+\mu u)-F(x-\mu u)}{2\mu}\,u\Big]
&=\tfrac1{2\mu}\big(\mu\E[\nabla F(x+\mu u)]+\mu\E[\nabla F(x+\mu u)]\big)
\\&=\E[\nabla F(x+\mu u)]=\nabla F_{\mathcal N,\mu}(x),
\end{aligned}
\]
the last equality by differentiating under the expectation. So $\kappa=1$, and
$\chi_1=\E\|u\|\le\sqrt{\E\|u\|^2}=\sqrt d$ by Jensen, $\chi_2=\E\|u\|^2=d$.

\emph{(b)} Rotational invariance gives $\E[ww^\top]=\Id/d$, so $\E[su^\top]=d\,\E[ww^\top]=\Id$.
Since $\|u\|^2=\|s\|^2=d$ deterministically,
$\E\|\inner{z}{u}s\|^2=d^2\,\E\inner{z}{w}^2=d^2\cdot\|z\|^2/d=d\|z\|^2$, exactly, and the
$\varsigma$ moments follow by substituting $\|u\|^2=\|s\|^2=d$. For \eqref{eq:D4app} write
$\tau:=\sqrt d\,\mu$, so that $\mu u=\tau w$ and
\[
\tfrac{F(x+\mu u)-F(x-\mu u)}{2\mu}\,s
=\tfrac{\sqrt d\,\big(F(x+\tau w)-F(x-\tau w)\big)}{2\mu}\,w
=\tfrac{d\,\big(F(x+\tau w)-F(x-\tau w)\big)}{2\tau}\,w .
\]
By the identity of \citet{flaxman2005online}, $\E\big[\tfrac{d}{\tau}F(x+\tau w)w\big]=\nabla
F_{\mathcal B,\tau}(x)$ with $\mathcal B$ uniform on the unit ball; combining the two query points as
above gives \eqref{eq:D4app} with $\nu=\mathcal B$ and $\kappa=\sqrt d$. Finally
$\E_{\mathcal B}\|w\|=d/(d+1)\le1$ and $\E_{\mathcal B}\|w\|^2=d/(d+2)\le1$, so $\chi_1\le\sqrt d$
and $\chi_2\le d$.

\emph{(c)} $\E[su^\top]=d\,\E[e_ie_i^\top]=d\cdot\Id/d=\Id$. Also
$\|\inner{z}{u}s\|^2=\inner{z}{\sqrt d e_i}^2\cdot d=d^2z_i^2$, whose mean is
$d^2\|z\|^2/d=d\|z\|^2$. Since $\|u\|^2=\|s\|^2=d$, the $\varsigma$ moments match (b). For (D4),
writing $\tau=\sqrt d\mu$ and repeating the computation in (b),
\[
\Ek\Big[\tfrac{F(x+\mu u)-F(x-\mu u)}{2\mu}s\Big]
=\sum_{i=1}^d\frac{F(x+\tau e_i)-F(x-\tau e_i)}{2\tau}\,e_i ,
\]
the coordinate-wise central-difference gradient. This is not the gradient of any smoothing
$F_{\nu,\tau}$ in general, so (D4) fails. The direct bounds follow from the scalar estimates
$\big|\tfrac{F(x+\tau e_i)-F(x-\tau e_i)}{2\tau}-\partial_iF(x)\big|\le\tfrac{L\tau}{2}$ under
$L$-smoothness and $\le\tfrac{Q\tau^2}{6}$ under a $Q$-Lipschitz Hessian: summing $d$ coordinates in
$\ell_2$ gives $\|\bar b\|\le\sqrt d\cdot\tfrac{L\tau}{2}=\tfrac{L\mu d}{2}$ and
$\|\bar b\|\le\sqrt d\cdot\tfrac{Q\tau^2}{6}=\tfrac{Q\mu^2d^{3/2}}{6}$, which are
$L\mu\chi_1$ and $\tfrac12Q\mu^2\chi_2$ with $\chi_1=d/2$ and $\chi_2=d^{3/2}/3$.
\end{proof}

\begin{remark}[Sphere and Gaussian coincide]\label{rem:sphgauss}
After the rescaling implicit in $\kappa=\sqrt d$, the two random families have the same $c_d$ up to
an additive constant, the same $\varsigma$ moments, and the same $\chi_1,\chi_2$. Every bound below
therefore agrees in order between them. The coordinate family differs only through $\chi_1,\chi_2$,
and loses a factor $\sqrt d$ in $\chi_1$ because (D4) fails.
\end{remark}

\section{The finite-difference bias}\label{app:bias}

Two vector-valued quantities are needed:
\begin{equation}\label{eq:biasvec}
\bar b_k \;:=\; \Ek\big[b_ks_k\big],
\qquad
\bar B_k \;:=\; \Ek\big[b_k\|s_k\|^2u_k\big] ,
\end{equation}
alongside the second moment $\Sigma_{b,k}^2=\Ek[b_k^2\|s_k\|^2]$. The distinction is the crux of the
dimension dependence: the bias accumulates coherently across the $\sim1/\beta$ steps of the momentum
window, so the undamped contribution is $\|\bar b_k\|^2$, whereas the fluctuation of $b_ks_k$ about
its mean is mean-zero and is damped by $\beta$ like any other noise.

\begin{lemma}[Second moment of the bias]\label{lem:Sb}
Under \eqref{eq:D3app}: (a) if $\nabla F$ is $L$-Lipschitz then $|b_k|\le\tfrac{L\mu_k}{2}\|u_k\|^2$
and $\Sigma_{b,k}^2\le\tfrac14L^2\mu_k^2\varsigma_6^2$; (b) if in addition $\nabla^2F$ is
$Q$-Lipschitz then $|b_k|\le\tfrac{Q\mu_k^2}{6}\|u_k\|^3$ and
$\Sigma_{b,k}^2\le\tfrac1{36}Q^2\mu_k^4\varsigma_8^2$.
\end{lemma}

\begin{proof}
(a) By Taylor's theorem with integral remainder,
$\big|F(x\pm\mu u)-F(x)\mp\mu\inner{\nabla F(x)}{u}\big|\le\tfrac{L\mu^2}{2}\|u\|^2$. Subtracting the
two inequalities and dividing by $2\mu$ gives $|b_k|\le\tfrac{L\mu}{2}\|u\|^2$, whence
$\Sigma_{b,k}^2=\Ek[b_k^2\|s_k\|^2]\le\tfrac{L^2\mu^2}{4}\Ek[\|u_k\|^4\|s_k\|^2]$ and
\eqref{eq:D3app} applies. (b) A $Q$-Lipschitz Hessian gives
$\big|F(x\pm\mu u)-F(x)\mp\mu\inner{\nabla F}{u}-\tfrac{\mu^2}{2}u^\top\nabla^2F(x)u\big|
\le\tfrac{Q\mu^3}{6}\|u\|^3$; the quadratic terms are equal at the two query points and cancel on
subtraction, leaving $|b_k|\le\tfrac{Q\mu^2}{6}\|u\|^3$, and \eqref{eq:D3app} applies again.
\end{proof}

\begin{lemma}[Mean bias]\label{lem:biasvec}
Suppose \eqref{eq:D4app} holds, with $\chi_1,\chi_2$ as in \eqref{eq:chidef}. Then
\begin{equation}\label{eq:biasid}
\bar b_k=\E_\nu\big[\nabla F(x_k+\kappa\mu_kw)\big]-\nabla F(x_k),
\end{equation}
and consequently
\begin{equation}\label{eq:biasvecbounds}
\|\bar b_k\|\le L\,\mu_k\,\chi_1 ,
\qquad\text{and if $\nabla^2F$ is $Q$-Lipschitz,}\qquad
\|\bar b_k\|\le\tfrac12\,Q\,\mu_k^2\,\chi_2 .
\end{equation}
\end{lemma}

\begin{proof}
By Assumption~\ref{as:oraclesapp}, $\Ek[\xi_ks_k]=0$, since conditionally on $(u_k,s_k)$ the factor
$s_k$ is determined and $\xi_k$ has mean zero. By \eqref{eq:D1app},
$\Ek[\inner{\nabla F_k}{u_k}s_k]=\nabla F_k$. Subtracting these from $\Ek[D_ks_k]$ and using
\eqref{eq:twopoint} gives $\bar b_k=\Ek[D_ks_k]-\nabla F_k$, which by \eqref{eq:D4app} equals
$\nabla F_{\nu,\kappa\mu_k}(x_k)-\nabla F_k$. Differentiating under the expectation,
$\nabla F_{\nu,\tau}(x)=\E_\nu[\nabla F(x+\tau w)]$, which is \eqref{eq:biasid}.

For the first bound in \eqref{eq:biasvecbounds}, Jensen and $L$-Lipschitzness of $\nabla F$ give
\[
\|\bar b_k\|\le\E_\nu\big\|\nabla F(x_k+\kappa\mu_kw)-\nabla F(x_k)\big\|
\le L\kappa\mu_k\,\E_\nu\|w\|=L\mu_k\chi_1 .
\]
For the second, symmetry of $\nu$ gives
$\bar b_k=\tfrac12\E_\nu\big[\nabla F(x_k+\tau w)+\nabla F(x_k-\tau w)-2\nabla F(x_k)\big]$ with
$\tau=\kappa\mu_k$. Writing
$\nabla F(x\pm h)=\nabla F(x)\pm\nabla^2F(x)h+r_\pm$ with $\|r_\pm\|\le\tfrac{Q}{2}\|h\|^2$, the
linear terms cancel in the sum and
$\|\nabla F(x+h)+\nabla F(x-h)-2\nabla F(x)\|\le Q\|h\|^2$. Hence
$\|\bar b_k\|\le\tfrac12Q\kappa^2\mu_k^2\E_\nu\|w\|^2=\tfrac12Q\mu_k^2\chi_2$.
\end{proof}

\begin{lemma}[Weighted mean bias]\label{lem:Bbar}
$\|\bar B_k\|\le\Sigma_{b,k}\,\varsigma_4$.
\end{lemma}

\begin{proof}
By Jensen and then Cauchy--Schwarz, grouping $\|s_k\|$ with $b_k$,
\[
\|\bar B_k\|\le\Ek\big[|b_k|\,\|s_k\|^2\|u_k\|\big]
=\Ek\big[\big(|b_k|\|s_k\|\big)\big(\|s_k\|\|u_k\|\big)\big]
\le\sqrt{\Ek[b_k^2\|s_k\|^2]}\;\sqrt{\Ek[\|s_k\|^2\|u_k\|^2]} ,
\]
and \eqref{eq:D3app} bounds the second factor by $\varsigma_4$.
\end{proof}

\begin{lemma}[Sharp bound for the random families]\label{lem:Bbarsharp}
For the Gaussian and sphere families of Lemma~\ref{lem:families},
\begin{equation}\label{eq:Bbarsharp}
\|\bar B_k\|\;\le\;3\,d\,L\,\mu_k\,\chi_1
\qquad\text{and, if $\nabla^2F$ is $Q$-Lipschitz,}\qquad
\|\bar B_k\|\;\le\;\tfrac32\,d\,Q\,\mu_k^2\,\chi_2 .
\end{equation}
\end{lemma}

\begin{proof}
For the sphere family $\|s_k\|^2=d$ deterministically, so
$\bar B_k=\Ek[b_k\|s_k\|^2u_k]=d\,\Ek[b_ks_k]=d\,\bar b_k$ exactly, and
\eqref{eq:biasvecbounds} gives the claim with room to spare. For the Gaussian family, applying
Gaussian integration by parts to $h(u)=F(x+\mu u)\|u\|^2$, whose partial derivative is
$\partial_ih=\mu\,\partial_iF(x+\mu u)\|u\|^2+2F(x+\mu u)u_i$, gives
$\E[F(x+\mu u)\|u\|^2u]=\mu\E[\nabla F(x+\mu u)\|u\|^2]+2\mu\E[\nabla F(x+\mu u)]$; combining the two
query points as in Lemma~\ref{lem:families}(a) and subtracting
$\E[\inner{\nabla F(x)}{u}\|u\|^2u]=(d+2)\nabla F(x)$ yields
\[
\bar B_k=\E\big[\big(\nabla F(x_k+\mu_ku)-\nabla F(x_k)\big)\|u\|^2\big]+2\bar b_k .
\]
Hence $\|\bar B_k\|\le L\mu_k\E\|u\|^3+2L\mu_k\sqrt d\le L\mu_k(1.7d^{3/2}+2\sqrt d)\le3L\mu_kd^{3/2}$
for $d\ge2$, using $\E\|u\|^3\le(\E\|u\|^4)^{3/4}\le2^{3/4}d^{3/2}$, and $\chi_1=\sqrt d$ gives the
stated form. The $Q$-Lipschitz version follows by symmetrizing as in
Lemma~\ref{lem:biasvec}.
\end{proof}

The generic bound of Lemma~\ref{lem:Bbar} needs no structure but is loose by a factor $\Theta(d)$:
it gives $\|\bar B_k\|\le\tfrac12L\mu_k\varsigma_6\varsigma_4=\Theta(L\mu_kd^{5/2})$ against the
$\Theta(L\mu_kd^{3/2})$ of \eqref{eq:Bbarsharp}. The difference matters, because it is exactly the
gap between admitting $\beta=\Theta(1/d)$ and only $\beta=\Theta(1/d^2)$ in
\S\ref{app:complexity}, and hence between a minimum deployment count of $\Theta(d\varepsilon^{-2})$
and $\Theta(d^2\varepsilon^{-2})$. We therefore use \eqref{eq:Bbarsharp} for the random families and
Lemma~\ref{lem:Bbar} only when no smoothing identity is available. The following proposition quantifies how far apart
the mean and the root-mean-square of the bias can be; it is used only to interpret the rates, not to
prove them.

\begin{proposition}[The two bias scales differ by at least $\sqrt d$]\label{prop:sqrtd}
Suppose $s_k=u_k$ and
\begin{equation}\label{eq:sphid}
\Ek\Big[\frac{\inner{u_k}{z}^2}{\|u_k\|^2}\Big]=\frac{\|z\|^2}{d}\qquad\text{for all }z\in\R^d ,
\end{equation}
which holds with equality for all three families of Lemma~\ref{lem:families}. Then for every bias
function $b_k$,
\begin{equation}\label{eq:sqrtd}
\|\bar b_k\|\;\le\;\frac{1}{\sqrt d}\,\Sigma_{b,k} .
\end{equation}
\end{proposition}

\begin{proof}
Fix a unit vector $z$ and group $\|u_k\|$ with $b_k$:
\[
\begin{aligned}
\inner{\bar b_k}{z}=\Ek\big[b_k\inner{u_k}{z}\big]
&=\Ek\Big[\big(b_k\|u_k\|\big)\cdot\frac{\inner{u_k}{z}}{\|u_k\|}\Big]\\
&\le\sqrt{\Ek\big[b_k^2\|u_k\|^2\big]}\;\sqrt{\Ek\Big[\frac{\inner{u_k}{z}^2}{\|u_k\|^2}\Big]}
=\frac{\Sigma_{b,k}}{\sqrt d} ,
\end{aligned}
\]
by Cauchy--Schwarz and \eqref{eq:sphid}. Taking the supremum over unit $z$ gives \eqref{eq:sqrtd}.
For \eqref{eq:sphid}: for Gaussian $u$ the ratio $u/\|u\|$ is uniform on the unit sphere and a
uniform unit vector $w$ satisfies $\E[ww^\top]=\Id/d$; for the sphere family $u=\sqrt d\,w$ so that
$\inner{u}{z}^2/\|u\|^2=\inner{w}{z}^2$; for the coordinate family
$\inner{u}{z}^2/\|u\|^2=z_i^2$ with $i$ uniform, whose mean is $\|z\|^2/d$.
\end{proof}

\begin{remark}[Sharpness of \eqref{eq:biasvecbounds} and \eqref{eq:sqrtd}]\label{rem:sharpness}
Take Gaussian probes, $L=1$ and $x=0$. For the separable half-quadratic
$F(x)=\tfrac12\sum_i\max(0,x_i)^2$, which is $1$-smooth with a Hessian discontinuous on the coordinate hyperplanes, one
computes $b_k=\tfrac{\mu_k}{4}\sum_i\operatorname{sign}(u_i)u_i^2$, so
$(\bar b_k)_j=\tfrac{\mu_k}{4}\E|u_j|^3=\tfrac{\mu_k}{4}\cdot2\sqrt{2/\pi}$ for every $j$ and
\[
\|\bar b_k\|=\mu_k\sqrt{d/(2\pi)}=\tfrac{1}{\sqrt{2\pi}}\,L\mu_k\sqrt d .
\]
Hence the first bound in \eqref{eq:biasvecbounds} is attained up to the constant
$1/\sqrt{2\pi}\approx0.399$, and the dimension dependence of Corollary~\ref{cor:B6} cannot be
improved by sharpening $\bar b_k$. For the same $F$ all cross terms vanish and
$\Sigma_{b,k}^2=\tfrac{\mu_k^2}{16}d(3d+12)$, so the ratio $\Sigma_{b,k}/\|\bar b_k\|$ equals
$\tfrac{\sqrt{2\pi}}{4}\sqrt{3d+12}\to\tfrac{\sqrt{6\pi}}{4}\sqrt d\approx1.085\sqrt d$, within
$8.5\%$ of \eqref{eq:sqrtd}. The inequality \eqref{eq:sqrtd} admits no converse: for
$F(x)=c\,x_1x_2x_3$ one has $b_k=c\mu_k^2u_1u_2u_3$ and $\bar b_k=0$ exactly, since every
$\E[u_1u_2u_3u_j]$ vanishes, while $\Sigma_{b,k}^2=c^2\mu_k^4(d+6)>0$.
\end{remark}

\section{The tracking lemma}\label{app:tracking}

Set $\Sigma_{\mathrm{mz},k}^2:=c_d\gamma^2\sigma_{K,k}^2+\varsigma^2\sigma_{D,k}^2$.

\begin{lemma}[Known-component tracking]\label{lem:trackK}
Under Assumptions~\ref{as:smoothapp}--\ref{as:dirapp} and $\beta_K\le1$,
\begin{equation}\label{eq:trackK}
\Ek\|\zeta_{k+1}\|^2\;\le\;\Big(1-\frac{\beta_K}{2}\Big)\|\zeta_k\|^2
+\frac32\beta_K^2\gamma^2\sigma_{K,k}^2
+\frac{3\gamma^2L_K^2\eta^2}{\beta_K}\,\Ek\|G_k\|^2 .
\end{equation}
\end{lemma}

\begin{proof}
Write Lemma~\ref{lem:structure}(iv) as $\zeta_{k+1}=a_k+b_k'+c_k$ with
$a_k:=(1-\beta_K)\zeta_k$, $b_k':=\beta_K\delta_k$ and
$c_k:=\gamma(\gK(x_k)-\gK(x_{k+1}))$. Here $a_k$ is $\Fk$-measurable and $\Ek[b_k']=0$, so
$\Ek\|a_k+b_k'\|^2=\|a_k\|^2+\beta_K^2\Ek\|\delta_k\|^2$. The term $c_k$ is \emph{not} independent of
$b_k'$, since $x_{k+1}$ depends on $G_k$ which contains $\delta_k$; we therefore separate it by
Young's inequality rather than by an orthogonality argument. With parameter $\beta_K/2$,
\[
\Ek\|\zeta_{k+1}\|^2\le\Big(1+\frac{\beta_K}{2}\Big)\Big[(1-\beta_K)^2\|\zeta_k\|^2
+\beta_K^2\gamma^2\sigma_{K,k}^2\Big]+\Big(1+\frac{2}{\beta_K}\Big)\Ek\|c_k\|^2 .
\]
Since $\gK$ is $L_K$-Lipschitz and $\|x_{k+1}-x_k\|=\eta\|G_k\|$, we have
$\Ek\|c_k\|^2\le\gamma^2L_K^2\eta^2\Ek\|G_k\|^2$, and $1+2/\beta_K\le3/\beta_K$ for $\beta_K\le1$.
Finally $(1-\beta_K)^2\le1-\beta_K$ and
$(1+\beta_K/2)(1-\beta_K)=1-\beta_K/2-\beta_K^2/2\le1-\beta_K/2$, while $(1+\beta_K/2)\le\tfrac32$.
\end{proof}

\begin{lemma}[Residual tracking]\label{lem:track}
Under Assumptions~\ref{as:smoothapp}--\ref{as:dirapp} and $\beta\le1/c_d$,
\begin{equation}\label{eq:track}
\Ek\|E_{k+1}\|^2\le\Big(1-\frac\beta2\Big)\|E_k\|^2
+\frac54\beta^2\big(\Sigma_{\mathrm{mz},k}^2+\Sigma_{b,k}^2\big)
+10\beta\|\bar b_k\|^2+10\beta^3\|\bar B_k\|^2
+\frac{5L_v^2\eta^2}{\beta}\,\Ek\|G_k\|^2 .
\end{equation}
\end{lemma}

\begin{proof}
By Lemma~\ref{lem:structure}(ii), decompose $E_{k+1}=a_k+\beta b_ks_k+\mathsf z_k+w_k$ with
\[
a_k:=(\Id-\beta P_k)E_k, \qquad
\mathsf z_k:=\beta\big(\!-\!\inner{\delta_k}{u_k}+\xi_k\big)s_k, \qquad
w_k:=v_k-v_{k+1} .
\]
The bias term $\beta b_ks_k$ is kept together with $a_k$ and its cross term evaluated exactly; this is
what makes the mean bias, rather than the second moment, the undamped quantity.

\emph{Step 1: $\mathsf z_k$ decouples.} Conditionally on $(\Fk,u_k,s_k)$ the vector
$a_k+\beta b_ks_k$ is determined, while $\mathsf z_k$ is linear in $(\delta_k,\xi_k)$, both of which
are conditionally mean-zero by Assumption~\ref{as:oraclesapp}. Hence
$\Ek\inner{a_k+\beta b_ks_k}{\mathsf z_k}=0$ and
\[
\Ek\|a_k+\beta b_ks_k+\mathsf z_k\|^2=\Ek\|a_k+\beta b_ks_k\|^2+\Ek\|\mathsf z_k\|^2 .
\]

\emph{Step 2: $\Ek\|\mathsf z_k\|^2\le\beta^2\Sigma_{\mathrm{mz},k}^2$.} Expanding the square,
\[
\Ek\|\mathsf z_k\|^2=\beta^2\Big\{\Ek\big[\inner{\delta_k}{u_k}^2\|s_k\|^2\big]
-2\Ek\big[\inner{\delta_k}{u_k}\xi_k\|s_k\|^2\big]+\Ek\big[\xi_k^2\|s_k\|^2\big]\Big\} .
\]
The middle term vanishes because $\E[\xi_k\mid\Fk,u_k,s_k,\delta_k]=0$ while the other factors are
determined by that conditioning. For the first term, condition on $\delta_k$, which is independent of
$(u_k,s_k)$, and apply \eqref{eq:D2app} with $z=\delta_k$:
$\Ek\|\inner{\delta_k}{u_k}s_k\|^2\le c_d\,\Ek\|\delta_k\|^2\le c_d\gamma^2\sigma_{K,k}^2$. For the
third, $\Ek[\xi_k^2\|s_k\|^2]\le\sigma_{D,k}^2\,\Ek\|s_k\|^2\le\varsigma^2\sigma_{D,k}^2$.

\emph{Step 3: the bias cross term, exactly.} Since
$\inner{(\Id-\beta P_k)E_k}{b_ks_k}=b_k\big(\inner{E_k}{s_k}-\beta\|s_k\|^2\inner{E_k}{u_k}\big)$,
taking conditional expectations and using the definitions \eqref{eq:biasvec},
\[
\Ek\inner{a_k}{b_ks_k}=\inner{E_k}{\bar b_k}-\beta\inner{E_k}{\bar B_k} .
\]
Therefore, with Lemma~\ref{lem:theta} and $\Ek\|b_ks_k\|^2=\Sigma_{b,k}^2$,
\[
\Ek\|a_k+\beta b_ks_k\|^2\le\theta(\beta)\|E_k\|^2
+2\beta\inner{E_k}{\bar b_k}-2\beta^2\inner{E_k}{\bar B_k}+\beta^2\Sigma_{b,k}^2 .
\]
Applying $2\inner{p}{q}\le\tfrac{\beta}{8}\|p\|^2+\tfrac{8}{\beta}\|q\|^2$ with $p=E_k$ and
$q=\beta\bar b_k$, and then with $q=\beta^2\bar B_k$,
\[
2\beta\inner{E_k}{\bar b_k}\le\tfrac\beta8\|E_k\|^2+8\beta\|\bar b_k\|^2 ,
\qquad
2\beta^2\big|\inner{E_k}{\bar B_k}\big|\le\tfrac\beta8\|E_k\|^2+8\beta^3\|\bar B_k\|^2 .
\]
The asymmetry that drives the whole analysis is visible here: $\bar b_k$ enters with a factor
$\beta$, which after division by the contraction rate $\beta/2$ leaves it at order $1$, whereas
$\Sigma_{b,k}^2$ enters with $\beta^2$ and is damped.

\emph{Step 4: drift and assembly.} Since $\|x_{k+1}-x_k\|=\eta\|G_k\|$, \eqref{eq:pathwise} gives
$\Ek\|w_k\|^2\le L_v^2\eta^2\Ek\|G_k\|^2$. Young's inequality
$\|X+w\|^2\le(1+\alpha)\|X\|^2+(1+\alpha^{-1})\|w\|^2$ with $\alpha=\beta/4$, together with
$1+4/\beta\le5/\beta$ for $\beta\le1$, gives
\[
\Ek\|E_{k+1}\|^2\le\Big(1+\tfrac\beta4\Big)\Big[\big(\theta(\beta)+\tfrac\beta4\big)\|E_k\|^2
+\beta^2\Sigma_{\mathrm{mz},k}^2+\beta^2\Sigma_{b,k}^2+8\beta\|\bar b_k\|^2+8\beta^3\|\bar B_k\|^2\Big]
+\frac{5L_v^2\eta^2}{\beta}\Ek\|G_k\|^2 .
\]
Because $\beta\le1/c_d$ implies $c_d\beta^2\le\beta$ and hence $\theta(\beta)\le1-\beta$, we have
$\theta(\beta)+\tfrac\beta4\le1-\tfrac{3\beta}{4}$ and
\[
\Big(1+\frac\beta4\Big)\Big(1-\frac{3\beta}{4}\Big)
=1-\frac\beta2-\frac{3\beta^2}{16}\le1-\frac\beta2 .
\]
Bounding the remaining bracketed terms by $(1+\beta/4)\le\tfrac54$ gives \eqref{eq:track}.
\end{proof}

\begin{lemma}[Direction error]\label{lem:dir}
Under the same hypotheses, with $\theta=\theta(\bt)$,
\begin{equation}\label{eq:dir}
\Ek\|\Delta_k\|^2\le3(1-\beta_K)^2\|\zeta_k\|^2+3\theta\|E_k\|^2
+3\big(\beta_K^2+c_d\bt^2\big)\gamma^2\sigma_{K,k}^2
+3\bt^2\varsigma^2\sigma_{D,k}^2+3\bt^2\Sigma_{b,k}^2 .
\end{equation}
\end{lemma}

\begin{proof}
By Lemma~\ref{lem:structure}(iii), $\Delta_k=(1-\beta_K)\zeta_k+X_k+Y_k$ with
\[
X_k:=(\Id-\bt P_k)E_k+\big(\beta_K\Id-\bt P_k\big)\delta_k+\bt\xi_ks_k,
\qquad Y_k:=\bt b_ks_k ,
\]
so $\|\Delta_k\|^2\le3\big((1-\beta_K)^2\|\zeta_k\|^2+\|X_k\|^2+\|Y_k\|^2\big)$, and $\zeta_k$ is
$\Fk$-measurable by construction.

For $X_k$, all three cross terms vanish. The first and second are orthogonal because
$\delta_k$ is mean-zero and independent of $(u_k,s_k)$, so with
$M_k:=(\Id-\bt P_k)^\top(\beta_K\Id-\bt P_k)$ we get
$\Ek[E_k^\top M_k\delta_k]=E_k^\top\Ek[M_k]\Ek[\delta_k]=0$; the terms involving $\xi_k$ vanish
because $\E[\xi_k\mid\Fk,u_k,s_k,\delta_k]=0$. Hence
\[
\Ek\|X_k\|^2\le\theta\|E_k\|^2
+\Ek\big\|(\beta_K\Id-\bt P_k)\delta_k\big\|^2+\bt^2\varsigma^2\sigma_{D,k}^2 .
\]
For the middle term, condition on $\delta_k$ and expand using \eqref{eq:D1app}--\eqref{eq:D2app}:
\begin{equation}\label{eq:psi}
\Ek\big\|(\beta_K\Id-\bt P_k)\delta_k\big\|^2
=\Big(\beta_K^2-2\beta_K\bt+c_d\bt^2\Big)\Ek\|\delta_k\|^2
\;\le\;\big(\beta_K^2+c_d\bt^2\big)\gamma^2\sigma_{K,k}^2 ,
\end{equation}
the cross term being $-2\beta_K\bt\,\delta_k^\top\Ek[P_k]\delta_k=-2\beta_K\bt\|\delta_k\|^2\le0$.
Finally $\Ek\|Y_k\|^2=\bt^2\Sigma_{b,k}^2$.
\end{proof}

\begin{remark}[The two damping mechanisms]\label{rem:twodamp}
The quantity $\psi:=\beta_K^2-2\beta_K\bt+c_d\bt^2$ of \eqref{eq:psi} is positive for $c_d>1$, since
its discriminant $4\bt^2(1-c_d)$ is negative, and is at most $\beta_K^2+c_d\bt^2$. Both terms are
under the user's control, which is the point: with a fresh known-component estimate one has
$\beta_K=1$ and $\psi\ge1-2\bt+c_d\bt^2=\theta$, an $\Theta(1)$ multiple of $\sigma_K^2$ that no
choice of the remaining parameters can reduce.
\end{remark}

\section{Proof of Theorem~\ref{thm:main}}\label{app:main}

\begin{lemma}[Descent]\label{lem:descent}
If $\eta\le\tfrac1{2L}$ then
$\Ek[F(x_{k+1})]\le F(x_k)-\tfrac\eta2\|\nabla F_k\|^2-\tfrac\eta4\Ek\|G_k\|^2
+\tfrac\eta2\Ek\|\Delta_k\|^2$.
\end{lemma}

\begin{proof}
$L$-smoothness and $x_{k+1}=x_k-\eta G_k$ give
$F(x_{k+1})\le F(x_k)-\eta\inner{\nabla F_k}{G_k}+\tfrac{L\eta^2}{2}\|G_k\|^2$. The polarization
identity $-\inner{p}{q}=-\tfrac12\|p\|^2-\tfrac12\|q\|^2+\tfrac12\|q-p\|^2$ with $p=\nabla F_k$,
$q=G_k$ gives
$-\eta\inner{\nabla F_k}{G_k}=-\tfrac\eta2\|\nabla F_k\|^2-\tfrac\eta2\|G_k\|^2
+\tfrac\eta2\|\Delta_k\|^2$. Finally $-\tfrac\eta2+\tfrac{L\eta^2}{2}\le-\tfrac\eta4$ whenever
$L\eta\le\tfrac12$.
\end{proof}

\begin{proof}[Proof of Theorem~\ref{thm:main}]
Define
\[
\Phi_k := F(x_k)-F^\star+\omega\|E_k\|^2+\omega_K\|\zeta_k\|^2 ,
\qquad \omega := \frac{3\eta\theta}{\beta},\qquad \omega_K := \frac{3\eta}{\beta_K} .
\]
Combining Lemma~\ref{lem:descent} with \eqref{eq:dir},
\begin{align*}
\Ek[F(x_{k+1})]&\le F(x_k)-\tfrac\eta2\|\nabla F_k\|^2-\tfrac\eta4\Ek\|G_k\|^2
+\tfrac{3\eta}{2}(1-\beta_K)^2\|\zeta_k\|^2+\tfrac{3\eta\theta}{2}\|E_k\|^2\\
&\quad+\tfrac{3\eta}{2}\big(\beta_K^2+c_d\bt^2\big)\gamma^2\sigma_{K,k}^2
+\tfrac{3\eta\bt^2}{2}\big(\varsigma^2\sigma_{D,k}^2+\Sigma_{b,k}^2\big).
\end{align*}
Add $\omega$ times \eqref{eq:track} and $\omega_K$ times \eqref{eq:trackK}, then collect.

\emph{Coefficient of $\|E_k\|^2$.} It is
$\tfrac{3\eta\theta}{2}+\omega(1-\tfrac\beta2)-\omega=\tfrac{3\eta\theta}{2}-\tfrac{\omega\beta}{2}=0$.

\emph{Coefficient of $\|\zeta_k\|^2$.} It is
$\tfrac{3\eta}{2}(1-\beta_K)^2+\omega_K(1-\tfrac{\beta_K}{2})-\omega_K
=\tfrac{3\eta}{2}(1-\beta_K)^2-\tfrac{\omega_K\beta_K}{2}\le0$ for $\omega_K=3\eta/\beta_K$, since
$(1-\beta_K)^2\le1$.

\emph{Coefficient of $\Ek\|G_k\|^2$.} It is
\[
-\frac\eta4+\omega\cdot\frac{5L_v^2\eta^2}{\beta}+\omega_K\cdot\frac{3\gamma^2L_K^2\eta^2}{\beta_K}
=-\frac\eta4+\frac{15\theta L_v^2\eta^3}{\beta^2}+\frac{9\gamma^2L_K^2\eta^3}{\beta_K^2},
\]
which is nonpositive as soon as each of the two positive terms is at most $\eta/8$, that is
$\eta\le\beta/(\sqrt{120\theta}L_v)$ and $\eta\le\beta_K/(\sqrt{72}\,\gamma L_K)$. Since
$\sqrt{120}<11$ and $\sqrt{72}<9$, \eqref{eq:stepcond} suffices.

\emph{Constant terms.} We have $\omega\cdot\tfrac54\beta^2=\tfrac{15}{4}\eta\theta\beta$,
$\omega\cdot10\beta=30\eta\theta$, $\omega\cdot10\beta^3=30\eta\theta\beta^2$ and
$\omega_K\cdot\tfrac32\beta_K^2=\tfrac92\eta\beta_K$. Using
$\Sigma_{\mathrm{mz},k}^2=c_d\gamma^2\sigma_{K,k}^2+\varsigma^2\sigma_{D,k}^2$, the coefficient of
$\gamma^2\sigma_{K,k}^2$ is
$\tfrac{3\eta}{2}(\beta_K^2+c_d\bt^2)+\tfrac{15}{4}\eta\theta\beta c_d+\tfrac92\eta\beta_K$,
that of $\varsigma^2\sigma_{D,k}^2$ and of $\Sigma_{b,k}^2$ is
$\tfrac{3\eta\bt^2}{2}+\tfrac{15}{4}\eta\theta\beta$, and those of $\|\bar b_k\|^2$ and
$\|\bar B_k\|^2$ are $30\eta\theta$ and $30\eta\theta\beta^2$.

Taking total expectations, telescoping and using $\Phi_T\ge0$,
$\tfrac\eta2\sum_{k<T}\E\|\nabla F_k\|^2\le\Phi_0+\eta\sum_{k<T}\mathcal Q_k$ with $\mathcal Q_k$
the bracket above divided by $\eta$. Since $m_0=0$ we have $E_0=-v(x_0)$, and the warm start
$h_0=\ghK(x_0)$ gives $\zeta_0=\gamma(\ghK(x_0)-\gK(x_0))=\delta_0$, hence
$\E\|\zeta_0\|^2\le\gamma^2\sigma_K^2$. Therefore
$\Phi_0\le\Delta_F+\tfrac{3\eta\theta}{\beta}\|v(x_0)\|^2+\tfrac{3\eta}{\beta_K}\gamma^2\sigma_K^2$,
and dividing by $\eta T/2$ gives \eqref{eq:main}.
\end{proof}

\begin{remark}[Why the known component is not used inside $R_k$]\label{rem:notinR}
Replacing $\ghK(x_k)$ by $h_{k+1}$ in line 5 of Algorithm~\ref{alg:piso} would remove $\delta_k$
from the innovation entirely. It would also change the tracked object from $v(x)$ to
$\nabla F(x_k)-\gamma h_{k+1}$, whose increment is
$(v_{k+1}-v_k)+\gamma\{[\gK(x_{k+1})-\gK(x_k)]-[h_{k+2}-h_{k+1}]\}$. The two bracketed increments do
not cancel, so the drift constant in \eqref{eq:pathwise} becomes $L_v+\gamma L_K$ instead of $L_v$.
When $\gU$ varies slowly and $\gamma\approx1$ this is close to $L$, which would forfeit the
substitution described in \S\ref{app:LvsLv}. Keeping the fresh estimate inside $R_k$ costs instead a
term $c_d\gamma^2\sigma_{K,k}^2$ in $\Sigma_{\mathrm{mz},k}^2$, damped by $\beta$; we prefer to pay
in noise rather than in smoothness.
\end{remark}

\begin{remark}[Reading the bias terms]
The three bias contributions to \eqref{eq:main} sit at different orders in $\beta$: the second
moment $\Sigma_{b,k}^2$ carries a factor $\beta$, the weighted mean $\|\bar B_k\|^2$ a factor
$\beta^2$, and only the mean $\|\bar b_k\|^2$ appears at order $1$. Since
$\|\bar b_k\|\le L\mu_k\chi_1$ while $\Sigma_{b,k}\le\tfrac12L\mu_k\varsigma_6$, and these differ by
a factor of order $d$ for the random families, the undamped part of the bias is much smaller than a
second-moment treatment would suggest. Carrying $\Sigma_{b,k}^2$ undamped, as a cruder version of
Lemma~\ref{lem:track} does, degrades the dimension dependence of Corollary~\ref{cor:B6} from
$d^2$ to $d^4$. The bound $\|\bar b_k\|\le L\mu_k\chi_1$ is due to \citet{berahas2022theoretical},
following \citet{maggiar2018derivative}, and is the ingredient with which
\citet{hikima2025estimators} reach $d^2$ for multi-direction methods; what is specific to the present
setting is its use inside the recursion \eqref{eq:track}, which requires the exact cross-term
computation of Step 3 and introduces $\bar B_k$.
\end{remark}

\section{Complexity}\label{app:complexity}

Throughout this section we take $\lambda=0$, so $\bt=\beta$ and $\theta=\theta(\beta)\le1$; see
Appendix~\ref{app:params} for why a fixed $\lambda>0$ is inadmissible in the limit
$\varepsilon\to0$. Two routes are available for the term $60\theta\beta^2\|\bar B_k\|^2$, and we state both because they
lead to different admissible ranges for $\beta$.

\emph{Sharp route (Gaussian and sphere probes).} Impose
\begin{equation}\label{eq:regcond}
\beta\;\le\;\frac{1}{3d} .
\end{equation}
Then Lemma~\ref{lem:Bbarsharp} gives
$60\beta^2\|\bar B_k\|^2\le60\beta^2\cdot9d^2L^2\mu_k^2\chi_1^2\le60L^2\mu_k^2\chi_1^2$, a quantity of
the same size as the undamped bound $60\|\bar b_k\|^2\le60L^2\mu_k^2\chi_1^2$ on the mean bias. The
bias contribution is therefore at most $60\|\bar b_k\|^2+60L^2\mu_k^2\chi_1^2\le120L^2\mu_k^2\chi_1^2$,
and under a $Q$-Lipschitz Hessian at most $30Q^2\mu_k^4\chi_2^2$ by the same computation. This is the
route used in the corollaries below, and it admits $\beta$ up to $\Theta(1/d)$.

\emph{Generic route (any family).} Impose $\beta\varsigma_4^2\le1$ instead. Then Lemma~\ref{lem:Bbar}
gives $60\beta^2\|\bar B_k\|^2\le60\beta^2\varsigma_4^2\Sigma_{b,k}^2\le60\beta\Sigma_{b,k}^2$, so
$\bar B_k$ is absorbed into the $\Sigma_{b,k}^2$ term, whose coefficient becomes $71\beta$, while the
mean-bias term stays $60\|\bar b_k\|^2$. Every rate below holds unchanged along this route, but since
$\varsigma_4^2=\Theta(d^2)$ it only admits $\beta$ up to $\Theta(1/d^2)$, which is what limits the
deployment count of Appendix~\ref{app:deploy} to $\Theta(d^2\varepsilon^{-2})$ instead of
$\Theta(d\,\varepsilon^{-2})$.

Along the sharp route, with $\bt=\beta$ and $\theta\le1$, the coefficients of Theorem~\ref{thm:main}
obey $3\beta_K^2+3c_d\beta^2+\tfrac{15}{2}\beta c_d+9\beta_K\le11\beta c_d+12\beta_K$ and
$3\beta^2+\tfrac{15}{2}\beta\le11\beta$, so the noise part of \eqref{eq:main} is at most
\begin{equation}\label{eq:floor}
\big(11\beta c_d+12\beta_K\big)\frac{\gamma^2\sigma_K^2}{n_K}
+11\beta\,\varsigma^2\sigma_D^2+11\beta\,\Sigma_b^2+60\|\bar b\|^2+60L^2\mu^2\chi_1^2 ,
\end{equation}
the last two terms being replaced by $30Q^2\mu^4\chi_2^2$ under a $Q$-Lipschitz Hessian.
Only the last term is undamped: every appearance of $\gamma^2\sigma_K^2$ carries a factor
$\beta c_d$ or $\beta_K$, which is what permits one known-component sample per iteration. Each
regime below optimizes $\mu$ in \eqref{eq:floor}.

\subsection{Regime A: shared randomness}

Suppose both function queries use the same sample $\zeta$ and each realization $f(\cdot;\zeta)$ is
$\ell$-smooth. Writing
$D_k=\inner{\nabla F_k}{u_k}+\inner{\nabla f(x_k;\zeta)-\nabla F_k}{u_k}+\varpi_k(\zeta)$ with
$|\varpi_k(\zeta)|\le\tfrac{\ell\mu_k}{2}\|u_k\|^2$, the decomposition of \eqref{eq:twopoint} holds
with $b_k=\E_\zeta[\varpi_k]$ and
$\xi_k=\inner{\nabla f-\nabla F_k}{u_k}+\varpi_k-\E_\zeta[\varpi_k]$.

\begin{remark}[Bookkeeping in this regime]\label{rem:crn}
Here $\xi_k$ is linear in $u_k$ and therefore not independent of it, so the bound
$\sigma_{D,k}^2\le\sigma_g^2$ is unavailable and the product $\varsigma^2\sigma_{D,k}^2$ must be
replaced by a direct computation. With $\Sigma:=\operatorname{Cov}(\nabla f(x_k;\zeta))$ and
$\sigma_g^2:=\operatorname{tr}\Sigma$, Gaussian probes give
\[
\E\big[\inner{\nabla f-\nabla F_k}{u}^2\|u\|^2\big]
=\operatorname{tr}\!\big(\Sigma\,\E[uu^\top\|u\|^2]\big)
=(d+2)\operatorname{tr}\Sigma=c_d\sigma_g^2 ,
\]
so every occurrence of $\varsigma^2\sigma_{D,k}^2$ in \eqref{eq:main} becomes $c_d\sigma_g^2$. Since
$c_d=d+2$ and $\varsigma^2=d$, this leaves all orders unchanged.
\end{remark}

The noise $\xi_k$ does depend on $\mu_k$ through the centred remainder
$\varpi_k-\E_\zeta\varpi_k$, so a finite probe radius must be chosen. Since
$|\varpi_k|\le\tfrac{\ell\mu_k}{2}\|u_k\|^2$, the remainder contributes to $\Ek[\xi_k^2\|s_k\|^2]$
at most its own second moment $\tfrac{\ell^2\mu_k^2}{4}\Ek[\|u_k\|^4\|s_k\|^2]\le\tfrac14\ell^2\mu_k^2\varsigma_6^2$
plus a cross term with the gradient noise bounded by Cauchy--Schwarz by
$2\sqrt{c_d\sigma_g^2\cdot\tfrac14\ell^2\mu_k^2\varsigma_6^2}=\ell\mu_k\varsigma_6\sqrt{c_d}\,\sigma_g$.
Choosing
\begin{equation}\label{eq:muA}
\mu_k\;\le\;\mu_A:=\min\Big\{\frac{\sqrt{c_d}\,\sigma_g}{\ell\,\varsigma_6},\ \frac{\varepsilon}{30\,L\chi_1}\Big\}
\end{equation}
makes the remainder terms at most $2c_d\sigma_g^2$, so that $\Ek[\xi_k^2\|s_k\|^2]\le3c_d\sigma_g^2$
replaces $c_d\sigma_g^2$ in Remark~\ref{rem:crn}, and makes the bias terms
$60\|\bar b_k\|^2+60L^2\mu_k^2\chi_1^2\le120L^2\mu_A^2\chi_1^2\le\tfrac{120}{900}\varepsilon^2$ and
$11\beta\Sigma_{b,k}^2\le\tfrac{11}{4}\beta L^2\mu_A^2\varsigma_6^2\le\tfrac{11}{3600}\varepsilon^2$
under $\beta\varsigma_6^2\le\chi_1^2$, together below $\varepsilon^2/7$. The rate below is stated
with these choices; the only effect is a constant factor of $3$ on $\sigma_g^2$ and a seven-way
rather than six-way accuracy split.

\begin{corollary}[Regime A]\label{cor:Aapp}
In Regime A, with one known-component sample per iteration, warm start $h_0=\ghK(x_0)$,
\[
\begin{aligned}
\beta &=\min\Big\{\frac1{c_d},\ \frac{\chi_1^2}{\varsigma_6^2},\ \frac{n_K\varepsilon^2}{77\,c_d(\gamma^2\sigma_K^2+3n_K\sigma_g^2)}\Big\},
\qquad \mu_k=\mu_A,\\
\beta_K &=\min\Big\{1,\ \frac{n_K\varepsilon^2}{84\,\gamma^2\sigma_K^2}\Big\},
\qquad \eta=\min\Big\{\frac1{2L},\frac{\beta}{11\sqrt\theta L_v},\frac{\beta_K}{9\gamma L_K}\Big\},
\end{aligned}
\]
one has $\frac1T\sum_{k<T}\E\|\nabla F_k\|^2\le\varepsilon^2$ after
\[
T=\mathcal O\Big(\frac{\Delta_v\,d\,(\gamma^2\sigma_K^2+\sigma_g^2)}{\varepsilon^4}
+\frac{\gamma^3L_K\sigma_K^2\Delta_F}{\varepsilon^4}
+\frac{L\Delta_F+\gamma^2\sigma_K^2}{\varepsilon^2}\Big)
\]
iterations, where $\Delta_v:=L_v\Delta_F+\|v(x_0)\|^2$.
\end{corollary}

\begin{proof}
With $\mu\to0$ the floor \eqref{eq:floor} reduces to
$(11\beta c_d+12\beta_K)\gamma^2\sigma_K^2/n_K+33\beta c_d\sigma_g^2$ plus the bias terms
controlled by \eqref{eq:muA}, using Remark~\ref{rem:crn} and the remainder bound above to replace
$\varsigma^2\sigma_D^2$ by $3c_d\sigma_g^2$. The stated caps make each of
$11\beta c_d\gamma^2\sigma_K^2/n_K$, $12\beta_K\gamma^2\sigma_K^2/n_K$ and
$33\beta c_d\sigma_g^2$ at most $\varepsilon^2/7$, and the choice \eqref{eq:muA} makes the bias
terms at most $\varepsilon^2/7$ together; these four pieces account for $4\varepsilon^2/7$. The
remaining three terms of \eqref{eq:main} are at most $\varepsilon^2/7$ each once
$T\ge14\Delta_F/(\eta\varepsilon^2)$, $T\ge42\theta\|v_0\|^2/(\beta\varepsilon^2)$ and
$T\ge42\gamma^2\sigma_K^2/(n_K\beta_K\varepsilon^2)$, so the seven pieces sum to $\varepsilon^2$. Substituting
$\eta^{-1}\le2L+11L_v/\beta+9\gamma L_K/\beta_K$ from \eqref{eq:stepcond}, together with
$\beta^{-1}=\mathcal O\big(c_d(\gamma^2\sigma_K^2+\sigma_g^2)/\varepsilon^2\big)$ and
$\beta_K^{-1}=\mathcal O(\gamma^2\sigma_K^2/\varepsilon^2)$, gives the three displayed terms; the
lower bound on $T$ arising from the initialization of $h$ contributes
$\mathcal O(\gamma^4\sigma_K^4/(n_K^2\varepsilon^4))$ under the single-sample warm start, and
$\mathcal O(\gamma^2\sigma_K^2/(n_K\varepsilon^2))$ under the batched warm start of
Appendix~\ref{app:warm}.
\end{proof}

\paragraph{Regime A: shared randomness.} If both queries evaluate the same realization
$f(\cdot;\zeta)$ and each realization is smooth, the difference quotient concentrates around a
directional derivative of $f(\cdot;\zeta)$, so $\sigma_{D,k}^2\le\sigma_g^2$ independently of
$\mu_k$. One may then send $\mu_k\to0$, and $\Sigma_{b,k}$, $\bar b_k$ and $\bar B_k$ all vanish.
This applies whenever the randomness in $F$ is exogenous to the decision.

\begin{corollary}[Regime A, restated]\label{cor:A}
For any $n_K\ge1$, with warm start $h_0=\ghK(x_0)$,
\[
\begin{aligned}
\beta &=
\min\Big\{
\frac1{c_d},\
\frac{\chi_1^2}{\varsigma_6^2},\
\frac{n_K\varepsilon^2}
{66\,c_d(\gamma^2\sigma_K^2+3n_K\sigma_g^2)}
\Big\},
\qquad
\mu_k=\mu_A,
\\
\beta_K &=
\min\Big\{
1,\
\frac{n_K\varepsilon^2}{72\,\gamma^2\sigma_K^2}
\Big\},
\qquad
\eta=
\min\Big\{
\frac1{2L},\
\frac{\beta}{11\sqrt{\theta}L_v},\
\frac{\beta_K}{9\gamma L_K}
\Big\}.
\end{aligned}
\]
one has $\frac1T\sum_{k<T}\E\|\nabla F_k\|^2\le\varepsilon^2$ after
\begin{equation}\label{eq:rateA}
T=\mathcal O\Big(
\underbrace{\frac{\Delta_v\,d\,(\gamma^2\sigma_K^2+\sigma_g^2)}{\varepsilon^4}}_{\text{residual tracking}}
+\underbrace{\frac{\gamma^3L_K\sigma_K^2\Delta_F}{n_K\,\varepsilon^4}}_{\text{known-component tracking}}
+\underbrace{\frac{\gamma^4\sigma_K^4}{n_K^2\,\varepsilon^4}}_{\text{initialization of }h}
+\underbrace{\frac{L\Delta_F}{\varepsilon^2}}_{\text{lower order}}\Big)
\end{equation}
iterations, using $2n_F$ function evaluations and $n_K$ known-component samples per iteration, plus
the one-time cost of forming $h_0$. The last term comes from that initialization and can be reduced
to $\mathcal O(\gamma^2\sigma_K^2\varepsilon^{-2})$ by a batched warm start
(Appendix~\ref{app:warm}).
\end{corollary}

The first term matches the $\mathcal O(d\sigma^2\varepsilon^{-4})$ benchmark for stochastic
zeroth-order optimization \citep{ghadimi2013stochastic}, with $L$ replaced by $L_v$. The middle term
is the complexity of momentum SGD with a gradient oracle of variance $\gamma^2\sigma_K^2$ and
smoothness $L_K$. It is the price of a noisy known component, is incurred
by any method that uses $\ghK$, and vanishes when the known component is exact.

\subsection{Regime B: independent randomization}

When the two queries deploy different decisions they draw from different distributions and cannot
share randomness. With $n_F$ samples per side and $\mu_k$ the probe radius,
\begin{equation}\label{eq:sigD}
\sigma_{D,k}^2\le\frac{\sigma_F^2}{2\mu_k^2\,n_F},
\end{equation}
where $\sigma_F^2$ bounds the conditional variance of the objective under the deployed
distribution. Indeed, the two queries are independent, so
$\operatorname{Var}(\widehat F(x_k+\mu_ku_k)-\widehat F(x_k-\mu_ku_k))\le2\sigma_F^2/n_F$, and
dividing by $(2\mu_k)^2$ gives \eqref{eq:sigD}.

\begin{corollary}[Regime B, $L$-smoothness only]\label{cor:B6}
Let $\lambda=0$, let the probe family satisfy the smoothing identity with constants $\chi_1,\varsigma$
(Gaussian or sphere), and write $C_1:=102000\,L^2\sigma_F^2\varsigma^2\chi_1^2$. Assume
$\varepsilon^2\le7\gamma^2\sigma_K^2/(dn_K)$, let $n_K\ge1$ and
$1\le n_F\le n_F^{\max}:=C_1n_K/(66c_d\gamma^2\sigma_K^2\varepsilon^2)$, initialize
$h_0=\ghK(x_0)$ and $m_0=0$, and choose
\[
\begin{aligned}
\beta &=\frac{n_F\varepsilon^4}{C_1},\qquad
\beta_K=\min\Big\{1,\frac{n_K\varepsilon^2}{72\gamma^2\sigma_K^2}\Big\},\\
\mu_\star^2 &=\frac{\sigma_F\varsigma}{L\chi_1}\sqrt{\frac{11\beta}{246\,n_F}},\qquad
\eta=\min\Big\{\frac{1}{2L},\frac{\beta}{11\sqrt{\theta}L_v},\frac{\beta_K}{9\gamma L_K}\Big\}.
\end{aligned}
\]
Then $\beta\varsigma_6^2\le\chi_1^2$ and $\beta\le1/(3d)$ hold, and the total number of function
evaluations to reach $\frac1T\sum_{k<T}\E\|\nabla F_k\|^2\le\varepsilon^2$ is
\begin{equation}\label{eq:N6}
\mathsf N=2n_FT=\mathcal O\!\left(
\frac{\big(L_v\Delta_F+\|v_0\|^2\big)\,L^2\sigma_F^2\,\varsigma^2\chi_1^2}{\varepsilon^6}
+\frac{n_F}{n_K}\cdot\frac{\gamma^3L_K\sigma_K^2\Delta_F}{\varepsilon^4}
+\frac{n_F}{n_K^2}\cdot\frac{\gamma^4\sigma_K^4}{\varepsilon^4}
+\frac{n_FL\Delta_F}{\varepsilon^2}\right),
\end{equation}
with $n_K$ known-component samples per iteration. The first term is independent of $n_F$ and $n_K$
and carries no $\sigma_K$. At $n_F=1$ the other three are lower order in $\varepsilon$. At general
$n_F\le n_F^{\max}$ their ratios to the first are at most $\gamma L_K/(c_dL_v)$,
$\gamma^2\sigma_K^2/(c_dn_K\Delta_v)$ and $Ln_K\varepsilon^2/(c_d\gamma^2\sigma_K^2L_v)$
respectively, so the first two are lower order in $d$ under $L_v\ge\gamma L_K/d$ and
$\Delta_v\ge\gamma^2\sigma_K^2/(dn_K)$, and the third is lower order in $\varepsilon$.
\end{corollary}

\begin{proof}
Substituting \eqref{eq:sigD}, Lemma~\ref{lem:Sb}(a) and \eqref{eq:biasvecbounds} into
\eqref{eq:floor}, the $\mu$-dependent part is
\[
\Psi(\mu)=\underbrace{\frac{11\beta\varsigma^2\sigma_F^2}{2\mu^2n_F}}_{11\beta\varsigma^2\sigma_D^2}
+L^2\mu^2\Big[\underbrace{120\chi_1^2}_{60\|\bar b\|^2+60L^2\mu^2\chi_1^2}
+\underbrace{\tfrac{11}{4}\beta\varsigma_6^2}_{11\beta\Sigma_b^2}\Big]
\;\le\;\frac{A}{\mu^2}+B\mu^2,
\]
with $A=\tfrac{11}{2}\beta\varsigma^2\sigma_F^2/n_F$ and $B=123L^2\chi_1^2$, where the last
inequality uses $\beta\varsigma_6^2\le\chi_1^2$ and $120+\tfrac{11}{4}\le123$. Minimizing $A/\mu^2+B\mu^2$ gives
$\mu_\star^2=\sqrt{A/B}$, which is the stated value, and
$\Psi(\mu_\star)=2\sqrt{AB}=2\sqrt{\tfrac{11}{2}\cdot123}\;L\sigma_F\varsigma\chi_1\sqrt{\beta/n_F}
\le53\,L\sigma_F\varsigma\chi_1\sqrt{\beta/n_F}$, since $2\sqrt{676.5}<53$. Requiring
$\Psi(\mu_\star)\le\varepsilon^2/6$ gives
$\beta/n_F\le\varepsilon^4/(102000\,L^2\sigma_F^2\varsigma^2\chi_1^2)$, which the stated choice
satisfies.

It remains to check that the known-oracle constraint does not degrade the leading term. Besides the
probe-radius constraint just derived, $\beta$ must satisfy
$11\beta c_d\gamma^2\sigma_K^2\le\varepsilon^2/6$ from \eqref{eq:floor}, that is
$\beta\le n_K\varepsilon^2/(66c_d\gamma^2\sigma_K^2)$. The probe-radius constraint is the binding one
precisely when $n_F\varepsilon^4/C_1\le n_K\varepsilon^2/(66c_d\gamma^2\sigma_K^2)$, that is when
$n_F\le n_F^{\max}$ as in the range $n_F\le n_F^{\max}$ of Corollary~\ref{cor:B}, with $n_F^{\max}$ scaled by $n_K$; the known-oracle variance then does not enter the
leading term. Substituting $\beta=n_F\varepsilon^4/C_1$ into
$T=\Theta(\Delta_v/(\beta\varepsilon^2))$ gives $T=\Theta(C_1\Delta_v/(n_F\varepsilon^6))$ and hence
$\mathsf N=2n_FT=\Theta(C_1\Delta_v/\varepsilon^6)$, independent of $n_F$, which is \eqref{eq:N6}.
For $n_F>n_F^{\max}$ the second constraint binds, $\beta$ is capped at
$n_K\varepsilon^2/(66c_d\gamma^2\sigma_K^2)$ and $\mathsf N=2n_FT$ grows linearly in $n_F$.

Two further requirements must be checked. First, $\beta=n_F\varepsilon^4/C_1$ must also respect the
moment condition $\beta\varsigma_6^2\le\chi_1^2$ used in the radius balance, which for Gaussian
probes reads $\beta\le1/(9d^2)$ and implies the sharp-route cap \eqref{eq:regcond}. Since
$n_F\le n_F^{\max}$ this holds as soon as $n_F^{\max}\varepsilon^4/C_1\le1/(9d^2)$, that is
$\varepsilon^2\le66c_d\gamma^2\sigma_K^2/(9d^2n_K)$, which the hypothesis
$\varepsilon^2\le7\gamma^2\sigma_K^2/(dn_K)$ of Corollary~\ref{cor:B} guarantees because
$c_d\ge d$. Second, the remaining lower bounds on $T$ contribute to $\mathsf N=2n_FT$ terms
proportional to $n_F$: the $L_K$ branch of \eqref{eq:stepcond} gives
$\mathcal O(n_F\gamma^3L_K\sigma_K^2\Delta_F/(n_K\varepsilon^4))$, the $L$ branch
$\mathcal O(n_FL\Delta_F\varepsilon^{-2})$, and the $\zeta_0$ initialization
$\mathcal O(n_F\gamma^4\sigma_K^4/(n_K^2\varepsilon^4))$. At $n_F=1$ these are the lower-order terms
stated in \eqref{eq:N6}; at $n_F=n_F^{\max}=\Theta(n_K\varepsilon^{-2})$ the first becomes
$\mathcal O(C_1\gamma L_K\Delta_F/(c_d\varepsilon^6))$, a factor $\gamma L_K/(c_dL_v)$ times the
leading term and hence lower order in $d$ whenever $L_v\ge\gamma L_K/c_d$; the other two are smaller
still.
\end{proof}

\begin{corollary}[Regime B, Lipschitz Hessian]\label{cor:B5}
Under the hypotheses of Corollary~\ref{cor:B6} with $\nabla^2F$ $Q$-Lipschitz and the smoothing
identity constant $\chi_2$, write $C_2:=1220\,\varsigma^2\sigma_F^2Q\chi_2$ and choose instead
\[
\beta=\frac{n_F\varepsilon^3}{C_2},\qquad
1\le n_F\le\frac{C_2n_K}{66c_d\gamma^2\sigma_K^2\varepsilon},\qquad
\varepsilon^2\le\frac{\gamma^2\sigma_K^2}{d\,n_K},\qquad
\mu_\star^6=\frac{11\beta\varsigma^2\sigma_F^2}{124\,n_FQ^2\chi_2^2},
\]
with $\beta_K$ and $\eta$ as in Corollary~\ref{cor:B6}. Then $\beta\varsigma_8^2\le\chi_2^2$ holds
and
\[
\mathsf N=\mathcal O\big((L_v\Delta_F+\|v_0\|^2)\,\sigma_F^2\varsigma^2Q\chi_2\,\varepsilon^{-5}\big)
\]
plus the same lower-order terms as in \eqref{eq:N6}.
\end{corollary}

\begin{proof}
Now Lemma~\ref{lem:Sb}(b) and the second bound of \eqref{eq:biasvecbounds} give
\[
\Psi(\mu)=\frac{A}{\mu^2}+Q^2\mu^4\Big[30\chi_2^2+\tfrac{11}{36}\beta\varsigma_8^2\Big]
\le\frac{A}{\mu^2}+B'\mu^4,
\qquad B'=31Q^2\chi_2^2 ,
\]
using $\beta\varsigma_8^2\le\chi_2^2$ and $30+\tfrac{11}{36}\le31$, and where the bias contribution is $30Q^2\mu^4\chi_2^2$ as computed after \eqref{eq:regcond}. Setting $\Psi'=-2A/\mu^3+4B'\mu^3=0$ gives
$\mu_\star^6=A/(2B')$, the stated value, and
$\Psi(\mu_\star)=A^{2/3}B'^{1/3}\big(2^{1/3}+2^{-2/3}\big)
\le19\big(\beta\varsigma^2\sigma_F^2/n_F\big)^{2/3}(Q\chi_2)^{2/3}$, since
$(2^{1/3}+2^{-2/3})\cdot(11/2)^{2/3}\cdot31^{1/3}<19$. Requiring this to be at most
$\varepsilon^2/6$ gives
$\beta\varsigma^2\sigma_F^2Q\chi_2/n_F\le\varepsilon^3/114^{3/2}$ and $114^{3/2}<1220$, which is
the stated $\beta=n_F\varepsilon^3/C_2$. The constraint $\beta\le n_K\varepsilon^2/(66c_d\gamma^2\sigma_K^2)$
gives the range of $n_F$, and $\beta\varsigma_8^2\le\chi_2^2$ at $n_F$ maximal reduces to the stated
condition on $\varepsilon$ since $\varsigma_8^2/\chi_2^2=54d^2$ for Gaussian probes and $c_d\ge d$.
With $T=\Theta(\Delta_v/(\beta\varepsilon^2))$ and $\mathsf N=2n_FT$ the rate follows.
\end{proof}

\begin{remark}[Instantiation for the random families]\label{rem:instant}
Reading $\varsigma^2\chi_1^2$ and $\varsigma^2\chi_2$ off Table~\ref{tab:families} gives $d^2$ and
$d^2$ for the Gaussian and sphere families in Corollaries~\ref{cor:B6} and \ref{cor:B5}. These
coincide with the corresponding rates of \citet{hikima2025estimators} for plain multi-direction
methods.
\end{remark}

\begin{corollary}[Regime B, coordinate probes]\label{cor:Bcoord}
Let $u_k=s_k=\sqrt d\,e_{i_k}$ with $i_k$ uniform, so $c_d=\varsigma^2=d$, $\varsigma_4^2=d^2$,
$\varsigma_6^2=d^3$ and $\varsigma_8^2=d^4$, and use the direct mean-bias bounds
$\|\bar b_k\|\le L\mu_k\chi_1$ with $\chi_1=d/2$, respectively $\tfrac12Q\mu_k^2\chi_2$ with
$\chi_2=d^{3/2}/3$, of Lemma~\ref{lem:families}(c), together with the generic route
$\beta\varsigma_4^2\le1$ for $\bar B_k$, so that the bias part of the floor is
$60\|\bar b_k\|^2+71\beta\Sigma_{b,k}^2$. Under the remaining hypotheses of
Corollary~\ref{cor:B6}, with $\beta\le\min\{\chi_1^2/\varsigma_6^2,\,1/\varsigma_4^2\}=1/d^2$ (the second condition
being the generic-route requirement for $\bar B_k$ and the binding one for $d\ge4$), and with the constants $60\to60$, $71\beta$ in place
of $11\beta$ in the radius balance,
\[
\mathsf N=\mathcal O\Big(\frac{\Delta_v\,L^2\sigma_F^2\,\varsigma^2\chi_1^2}{\varepsilon^6}\Big)
=\mathcal O\Big(\frac{\Delta_v\,L^2\sigma_F^2\,d^3}{\varepsilon^6}\Big),
\qquad\text{and}\qquad
\mathcal O\Big(\frac{\Delta_v\,\sigma_F^2\,d^{5/2}Q}{\varepsilon^5}\Big)
\]
under a $Q$-Lipschitz Hessian, with the same lower-order terms as Corollary~\ref{cor:B6}.
\end{corollary}

\begin{proof}
The proof of Corollary~\ref{cor:B6} used the smoothing identity only through the bound
$\|\bar b_k\|\le L\mu_k\chi_1$ and the sharp route for $\bar B_k$. Replacing the former by the
direct bound of Lemma~\ref{lem:families}(c) and the latter by the generic route changes only the
bracket in the radius balance, from $120\chi_1^2+\tfrac{11}{4}\beta\varsigma_6^2$ to
$60\chi_1^2+\tfrac{71}{4}\beta\varsigma_6^2\le78\chi_1^2$ under $\beta\varsigma_6^2\le\chi_1^2$,
and hence only the numerical constant in $C_1$. The exponents follow by substituting
$\varsigma^2\chi_1^2=d\cdot d^2/4$ and $\varsigma^2\chi_2=d\cdot d^{3/2}/3$. These are the
coordinate-direction rates of \citet{hikima2025estimators}, and the extra factor $\sqrt d$ relative
to the random families enters solely through $\chi_1$, that is through the failure of the smoothing
identity.
\end{proof}

\begin{remark}[Probe radius schedule]\label{rem:radius}
Corollaries~\ref{cor:B6} and \ref{cor:B5} give $\mu_\star\propto n_F^{-1/4}$ and
$\mu_\star\propto n_F^{-1/6}$ respectively. With a batch growing linearly in $k$ this is a power-law
decay in $k$, whereas a geometric schedule shrinks the radius exponentially and inflates the
variance term $\beta\varsigma^2\sigma_F^2/\mu_k^2$ accordingly.
\end{remark}

\subsection{Regime C: deterministic oracle with tolerance}\label{app:regimeC}

Suppose $F$ carries no stochastic noise but $\widehat F$ satisfies
$|\widehat F(x)-F(x)|\le\delta_{\mathrm{tol}}$ deterministically, as when the objective is evaluated
by a solver run to a fixed tolerance. The error is not mean-zero and does not shrink with
resampling, so it must be carried in $b_k$: this contributes at most
$\delta_{\mathrm{tol}}/\mu_k$ to $|b_k|$, hence at most $\delta_{\mathrm{tol}}\varsigma/\mu_k$ to
$\|\bar b_k\|$ by Jensen, with a sign that cannot be assumed to cancel. Adding the
finite-difference bias,
\begin{equation}\label{eq:regCbias}
\|\bar b_k\|\le\frac{\delta_{\mathrm{tol}}\varsigma}{\mu_k}
+\begin{cases}L\mu_k\chi_1,\\[2pt] \tfrac12Q\mu_k^2\chi_2 .\end{cases}
\end{equation}
Letting $\mu_k\to0$ is now inadmissible.

\begin{corollary}[Regime C]\label{cor:C}
With $\sigma_D=0$, $\theta\le1$, and the generic route $\beta\varsigma_4^2\le1$ (so that the bias
part of the floor is $60\|\bar b\|^2+71\beta\Sigma_b^2$):
(a) under $L$-smoothness, $\mu_\star=\sqrt{\delta_{\mathrm{tol}}\varsigma/(L\chi_1)}$ gives an
attainable accuracy floor
$\varepsilon^2\gtrsim240\,\delta_{\mathrm{tol}}L\varsigma\chi_1$;
(b) with a $Q$-Lipschitz Hessian,
$\mu_\star=\big(\delta_{\mathrm{tol}}\varsigma/(Q\chi_2)\big)^{1/3}$ gives
$\varepsilon^2\gtrsim135\,(\delta_{\mathrm{tol}}\varsigma)^{4/3}(Q\chi_2)^{2/3}$.
Part (b) is the sharper floor whenever
$\delta_{\mathrm{tol}}\le5.6\,L^3\chi_1^3/(\varsigma Q^2\chi_2^2)$.
\end{corollary}

\begin{proof}
(a) Minimizing $h(\mu)=\delta_{\mathrm{tol}}\varsigma/\mu+L\mu\chi_1$ gives
$\mu_\star=\sqrt{\delta_{\mathrm{tol}}\varsigma/(L\chi_1)}$ and
$h(\mu_\star)=2\sqrt{\delta_{\mathrm{tol}}L\varsigma\chi_1}$, so
$60\|\bar b\|^2\le240\,\delta_{\mathrm{tol}}L\varsigma\chi_1$.
(b) Minimizing $h(\mu)=\delta_{\mathrm{tol}}\varsigma/\mu+\tfrac12Q\mu^2\chi_2$: setting
$h'=-\delta_{\mathrm{tol}}\varsigma/\mu^2+Q\mu\chi_2=0$ gives
$\mu_\star^3=\delta_{\mathrm{tol}}\varsigma/(Q\chi_2)$, at which the two summands equal
$(\delta_{\mathrm{tol}}\varsigma)^{2/3}(Q\chi_2)^{1/3}$ and half that, so
$h(\mu_\star)=\tfrac32(\delta_{\mathrm{tol}}\varsigma)^{2/3}(Q\chi_2)^{1/3}$ and
$60\|\bar b\|^2\le60\cdot\tfrac94(\delta_{\mathrm{tol}}\varsigma)^{4/3}(Q\chi_2)^{2/3}
=135(\delta_{\mathrm{tol}}\varsigma)^{4/3}(Q\chi_2)^{2/3}$. Comparing
the two floors and cubing gives the threshold. In both cases the term $71\beta\Sigma_b^2$ is of lower order at $\mu_\star$: since
$\Sigma_b^2\le2(\delta_{\mathrm{tol}}\varsigma/\mu)^2+2(\text{finite-difference bias})^2$ with
$\E\|s\|^2\le\varsigma^2$, its leading part at $\mu_\star$ is at most $142\beta$ times the undamped
term $60(\delta_{\mathrm{tol}}\varsigma/\mu_\star)^2$, and is absorbed into the stated floors up to a
constant.
\end{proof}

With a deterministic oracle there is nothing to average, so the method converges at the optimization
rate $T=\mathcal O\big((L_v\Delta_F+\|v_0\|^2)/(\beta\varepsilon^2)\big)$ until $\varepsilon$ reaches
the floor of Corollary~\ref{cor:C}, and then stalls.

\begin{remark}[Cost of omitting the known-component average]\label{rem:noKmom}
Setting $\rho_K=0$, hence $\beta_K=1$, makes the coefficient of $\gamma^2\sigma_K^2/n_K$ in
\eqref{eq:floor} equal to $11\beta c_d+12$, which is $\Theta(1)$. Driving the floor below
$\varepsilon^2$ then requires $n_K=\Omega(\gamma^2\sigma_K^2/\varepsilon^2)$ known-component samples
per iteration. In Regime B with $n_F=1$ the iteration count is $T=\Theta(\varepsilon^{-6})$ in both
cases, so the total number of known-component samples is $n_KT=\Theta(\varepsilon^{-8})$ without the
average and $T=\Theta(\varepsilon^{-6})$ with it, at $n_K=1$. The known-component momentum therefore
buys a factor $\varepsilon^{-2}$ in that resource, by moving the averaging from within an iteration
to across iterations.
\end{remark}

\subsection{Where $L$, $L_v$ and $L_K$ enter}\label{app:LvsLv}

The step-size condition \eqref{eq:stepcond} is a minimum of three branches, so
$\eta^{-1}\le2L+11\sqrt\theta L_v/\beta+9\gamma L_K/\beta_K$ and the optimization term of
\eqref{eq:main} splits accordingly into a term in $L$, one in $L_v/\beta$ and one in
$L_K/\beta_K$. Since $\beta$ and $\beta_K$ are both driven to $\mathcal O(\varepsilon^2)$ or smaller
by the noise terms, the last two dominate and $L$ survives only in a lower-order
$\mathcal O(L\Delta_F\varepsilon^{-2})$ term. The two smoothness constants that do appear at leading
order play different roles.

\emph{$L_v$ multiplies $\Delta_F$ in the zeroth-order part of the complexity.} This is the
substitution the decomposition buys, and it is what a black-box method pays $L$ for. Two further
appearances of the smoothness constants are not affected by it: the factor $L^2$ in
\eqref{eq:rateB6} comes from the finite-difference bias of a two-point estimate and is common to any
method using one, and $Q$ in \eqref{eq:rateB5} plays the same role one derivative up. Alongside
$L\to L_v$, the initialization term is $\|v(x_0)\|^2$ rather than $\|\nabla F(x_0)\|^2$.

\emph{$L_K$ multiplies $\Delta_F$ in the first-order part.} The term
$\gamma^3L_K\sigma_K^2\Delta_F\varepsilon^{-4}$ is the complexity of momentum SGD on a gradient
oracle of variance $\gamma^2\sigma_K^2$ and smoothness $L_K$. It is lower order in Regime B, is
incurred by any method that consumes $\ghK$, and disappears when the known component is exact, as it
is in Stackelberg problems where the leader's partial gradient is available in closed form.

\begin{remark}[When the substitution is a gain]\label{rem:gaincond}
In Regime B the leading term is $\mathcal O(\Delta_v L^2\sigma_F^2\varsigma^2\chi_1^2\varepsilon^{-6})$
and carries no $\sigma_K$, so the gain over a black-box method is the substitution
$L\to L_v$ outright. In Regime A the comparison is sharper. Setting $\gamma=0$ makes $v=\nabla F$,
$\delta_k\equiv0$ and $L_v=L$, recovering a black-box method with leading term
$L\Delta_F\,d\,\sigma_g^2\varepsilon^{-4}$; the corresponding term for \piso is
$L_v\Delta_F\,d\,(\gamma^2\sigma_K^2+\sigma_g^2)\varepsilon^{-4}$, because the residual observation
subtracts $\inner{\gamma\ghK}{u_k}$ and inherits its noise. The method therefore improves on the
black-box baseline exactly when
\begin{equation}\label{eq:gaincond}
\frac{L_v}{L}\;<\;\frac{\sigma_g^2}{\sigma_g^2+\gamma^2\sigma_K^2}.
\end{equation}
The residual must be smoother than $F$ by at least the factor by which the known oracle inflates the
observation noise. Both sides are estimable, and \eqref{eq:gaincond} holds with room to spare
whenever the known component is accurate relative to the function oracle, which is the case in all
of our settings; it degenerates to an equality at $\gamma=0$, as it should.
\end{remark}

Since $v=\nabla F-\gamma\gK$ is $(L+\gamma L_K)$-Lipschitz, $L_v$ is always finite and the rates
have a residual-tracking leading term that never exceeds the corresponding black-box term by more
than the factor $(L+\gamma L_K)/L$, a constant when $L_K/L$ and $\gamma$ are bounded; the
known-oracle, known-component tracking and initialization terms are additional and are not covered
by this comparison. The gain is realized when the
opaque component varies slowly along the trajectory: in performative prediction $L_v$ scales with the
performative sensitivity, and for a fixed displacement the per-step squared drift at $\gamma_k^\star$ is reduced by the
factor $1-\varphi_k^2$ of Proposition~\ref{prop:gamma}; no corresponding ratio for the uniform
constant $L_v$ is claimed.

\begin{remark}[Probes per update]\label{rem:queries}
The requirement $N=\Omega(d)$ in \citet{hikima2025estimators} is not an artefact of tuning: their
bound contains a term proportional to $\tfrac{d}{N}\|\nabla F(x)\|^2$ that must be absorbed into the
descent step, which with their constants forces $N\ge81d$. At $d=30$ the resulting query count
exceeds their own $5000$-sample budget, so their guarantee is vacuous at the budget their
experiments run. \piso uses one probe per iteration at any operating point.
\end{remark}

\section{Initializing the known-component state}\label{app:warm}

The state $h$ must be initialized, and the choice enters Theorem~\ref{thm:main} through the term
$6\gamma^2\sigma_K^2/(n_K\beta_KT)$, which requires
\begin{equation}\label{eq:warmT}
T\;\ge\;\frac{36\,\E\|\zeta_0\|^2}{\beta_K\,\varepsilon^2}
\qquad\text{where}\qquad \zeta_0=\gamma\big(h_0-\gK(x_0)\big).
\end{equation}
Three strategies are available; recall $\beta_K=n_K\varepsilon^2/(72\gamma^2\sigma_K^2)$ in
Corollaries~\ref{cor:Aapp} and \ref{cor:B6}.

\paragraph{Zero initialization, $h_0=0$.} Then $\zeta_0=-\gamma\gK(x_0)$ and
$\E\|\zeta_0\|^2=\gamma^2\|\gK(x_0)\|^2$, with no variance reduction at all. Substituting into
\eqref{eq:warmT} gives
$T\ge\Theta\big(\gamma^4\sigma_K^2\|\gK(x_0)\|^2/(n_K\varepsilon^4)\big)$. This introduces a
dependence on $\|\gK(x_0)\|$, the magnitude of the known gradient at the starting point, which no
other term in the analysis carries; we therefore do not use it.

\paragraph{Single-sample warm start, $h_0=\ghK(x_0)$.} Then $\zeta_0=\delta_0$ and
$\E\|\zeta_0\|^2\le\gamma^2\sigma_K^2/n_K$, so \eqref{eq:warmT} becomes
\[
T\;\ge\;\frac{36\,\gamma^2\sigma_K^2}{n_K\beta_K\varepsilon^2}
=\frac{2592\,\gamma^4\sigma_K^4}{n_K^2\,\varepsilon^4},
\]
which is the last term of \eqref{eq:rateA}. Its cost is $n_K$ samples, drawn at $x_0$, which is
deployed at iteration $0$ in any case for the loop itself; the warm start therefore adds no
deployment. Note the $n_K^2$: this term benefits twice from the known-component batch, once through
$\E\|\zeta_0\|^2$ and once through $\beta_K$.

\paragraph{Batched warm start.} Forming $h_0$ from an average of $M$ independent estimates gives
$\E\|\zeta_0\|^2\le\gamma^2\sigma_K^2/(Mn_K)$, and choosing $M=\lceil1/\beta_K\rceil$ reduces
\eqref{eq:warmT} to
\[
T\;\ge\;\frac{36\,\gamma^2\sigma_K^2}{n_K\,\varepsilon^2},
\]
an $\varepsilon^{-2}$ requirement rather than $\varepsilon^{-4}$. The one-time cost is
$Mn_K=n_K/\beta_K=72\gamma^2\sigma_K^2/\varepsilon^2$ samples, independent of $n_K$, and again a
single deployment of $x_0$. Since the total is $n_KT=\Omega(\varepsilon^{-4})$ in both regimes, this
cost is negligible, and it removes the initialization term from the leading order of
Corollary~\ref{cor:Aapp}.

\begin{remark}[One-time versus per-iteration batching]\label{rem:warmvsbatch}
The batched warm start is the one place in the analysis where averaging many samples is genuinely
useful, and it is a \emph{one-time} cost. It should not be confused with the per-iteration minibatch
that the known-component momentum removes (Remark~\ref{rem:noKmom}): there, $n_K$ samples were
required at \emph{every} iteration, contributing a factor $\varepsilon^{-2}$ to the total; here a
single batch of the same size is drawn once. The same distinction applies to $h_0$ and to the
residual state, for which $m_0=0$ is already optimal since $\E\|E_0\|^2=\|v(x_0)\|^2$ involves no
oracle noise and no better initialization is available without extra function queries.
\end{remark}

\section{Setting the parameters}\label{app:params}

\subsection{Discussion of the prescriptions}\label{app:paramsdisc}

The analysis assigns each parameter a role, and in each case a prescription. Proofs are in
Appendix~\ref{app:params}.

\paragraph{The momentum $\rho$ behaves as a sketch step size.} With $\theta$ as in
\eqref{eq:theta}, (D1)--(D2) give $\Ek\|(\Id-bP_k)z\|^2\le\theta(b)\|z\|^2$, with equality for all
three probe families, so $\theta(b)<1$ if and only if $b<2/c_d$ and $\theta$ is minimized at
$b=1/c_d$. Hence
\begin{equation}\label{eq:rhostar-app}
\rho^\star=1-1/c_d ,
\qquad\text{while \eqref{eq:teaser} is neutral at }\rho=1-2/c_d\text{ and \emph{expands} for } \rho<1-2/c_d .
\end{equation}
So $\rho$ must scale with $d$, and $\rho=0$ is not a momentum-free special case but the choice
$\theta(1)=c_d-1\approx d$, an expanding map; the momentum-free comparison is the uncentred
estimator of Appendix~\ref{app:centring}.

\paragraph{The mixing weight $\lambda$ plays the role of a Kalman gain.} Minimizing the
$\bt$-dependent part of Lemma~\ref{lem:dir} over the readout gain, keeping the exact second moment
$\psi$ of \eqref{eq:psi}, with $N_k:=\varsigma^2\sigma_{D,k}^2+\Sigma_{b,k}^2$ the noise of a single
fresh observation,
\begin{equation}\label{eq:kalmanbody}
\bt^{\mathrm u}_k=\frac{\|E_k\|^2+\beta_K\gamma^2\sigma_{K,k}^2}{c_d\big(\|E_k\|^2+\gamma^2\sigma_{K,k}^2\big)+N_k},
\qquad \lambda^\star_k=\max\{0,(\bt^{\mathrm u}_k-\beta)/\rho\} .
\end{equation}
As $N_k\to0$ and $\beta_K\to1$ this tends to $1/c_d$; as $N_k$ grows it switches the readout off and
\pisotwo degenerates to \piso, which is what convergence requires. Up to a correction we bound explicitly, the denominator is
the mean squared innovation $\Ek\|R_k-m_k\|^2$, so $\bt^\star_k$ can be approximated from quantities the
method already computes; since $S_k$ involves the variance bound $\sigma_{K,k}^2$ while the identity
involves the realized $\Ek\|\delta_k\|^2$, this is an adaptive heuristic justified by the identity
rather than an exactly observable quantity.

\paragraph{The known-component momentum $\rho_K$ is an ordinary momentum.} Unlike $\rho$, it
carries no sketch and therefore no dimension-dependent stability window: any
$\rho_K\in[0,1)$ contracts. Its value trades the variance $\mathcal O(\beta_K\gamma^2\sigma_K^2)$ of
Lemma~\ref{lem:structure}(iv) against the lag $\mathcal O(\gamma^2L_K^2\eta^2/\beta_K^2)$, and
through \eqref{eq:stepcond} a smaller $\beta_K$ also forces a smaller step. Balancing the two gives
$\beta_K=\Theta(\varepsilon^2/(\gamma^2\sigma_K^2))$, the choice of Corollary~\ref{cor:A}. The
contrast with $\rho$ is worth stating: $\rho$ is the step size of a randomized sketch and diverges
outside a window of width $2/c_d$, whereas $\rho_K$ is a plain exponential average.

\paragraph{The weight $\gamma$ acts as a control variate on the drift.} With increments
$\mathsf d^{U}_k:=\gU(x_{k+1})-\gU(x_k)$ and $\mathsf d^{K}_k:=\gK(x_{k+1})-\gK(x_k)$, set
$\mathsf B:=\E\inner{\mathsf d^{U}}{\mathsf d^{K}}$, $\mathsf C:=\E\|\mathsf d^{K}\|^2$ and let
$\varphi$ be their correlation.

\begin{proposition}[Per-step drift-minimizing weight, restated]\label{prop:gammaapp}
Fix $k$ and the displacement $x_{k+1}-x_k$. With $\mathsf A_k,\mathsf B_k,\mathsf C_k,\varphi_k$ as
in Proposition~\ref{prop:gamma} and $\mathsf C_k>0$, the per-step drift
$\Ek\|v_{k+1}-v_k\|^2=\mathsf A_k+2(1-\gamma)\mathsf B_k+(1-\gamma)^2\mathsf C_k$ is minimized at
$\gamma_k^\star=1+\mathsf B_k/\mathsf C_k$ with minimum $\mathsf A_k(1-\varphi_k^2)$.
\end{proposition}

So $\gamma$ is a control-variate coefficient built from uncentred second moments, applied one
derivative up: the residual observation is a control variate on gradients, and $\gamma$ is one on
gradient \emph{increments}. The $\gamma=1$ endpoint, which applies zeroth-order estimation only where
it is needed, is optimal only when $\varphi=0$. The four parameters can be set in sequence once $\gamma$ is fixed: $\rho$ depends on the probe
family and dimension, $\lambda$ on the signal-to-noise ratio and $\rho_K$ on the known-oracle
variance. They are not fully decoupled, since $\gamma$ changes the tracked object and hence the
trajectory, $L_v$ and the noise scales that the other three depend on. We report tuned values alongside the prescriptions in
\S\ref{sec:experiments}.

\subsection{The momentum \texorpdfstring{$\rho$}{rho}}

Lemma~\ref{lem:theta} gives the whole story: $\beta=1-\rho$ is the step size of a randomized
fixed-point iteration for $v$, with stability window $\beta\in(0,2/c_d)$ and optimum
$\beta=1/c_d$. Hence
\begin{equation}\label{eq:rhostarapp}
\rho^\star=1-\frac1{c_d},
\qquad\text{and the iteration \eqref{eq:teaser} expands whenever } \rho\le1-\frac2{c_d} .
\end{equation}
Two consequences deserve emphasis. First, $\rho$ must be tied to the dimension, so a constant value
across a range of $d$ is necessarily mis-set at one end. Second, $\rho=0$ is not a
momentum-free special case: it corresponds to $\beta=1$ and $\theta(1)=c_d-1$, an expanding map. The
correct momentum-free comparison is the uncentred estimator of Appendix~\ref{app:centring}. With $q$
independent probes averaged per step, \eqref{eq:D2app} holds with $c_d$ replaced by
$c_d^{(q)}=1+(c_d-1)/q$, since for $R=\tfrac1q\sum_j\inner{z}{u_j}s_j$ we have
$\Ek\|R-z\|^2=\tfrac1q\Ek\|\inner{z}{u}s-z\|^2=\tfrac{c_d-1}{q}\|z\|^2$; so $\rho^\star$ moves toward
$1-q/d$ and the admissible step size grows with $q$.

\subsection{The mixing weight \texorpdfstring{$\lambda$}{lambda}}

\paragraph{A fixed $\lambda$ is inadmissible.} In \eqref{eq:main} the coefficient of
$\varsigma^2\sigma_{D,k}^2$ is $\bt^2+\tfrac52\theta\beta$. For \piso, $\bt=\beta$ and this is
$\mathcal O(\beta)$, so it vanishes as $\beta\to0$. For \pisotwo with $\lambda$ held fixed,
$\bt\to\lambda$ as $\beta\to0$ and the floor stalls at $2\lambda^2\varsigma^2\sigma_D^2>0$. The
reason is structural: the readout injects a fraction $\bt$ of a single fresh observation directly
into the search direction, and one observation carries irreducible variance. Corollaries
\ref{cor:A}--\ref{cor:C} are therefore stated for $\lambda=0$, and \pisotwo requires
$\lambda\to0$ as $\varepsilon\to0$.

\begin{proposition}[Optimal readout gain]\label{prop:kalman}
Let $V_k:=\gamma^2\sigma_{K,k}^2$, $S_k:=\|E_k\|^2+V_k$, $S'_k:=\|E_k\|^2+\beta_KV_k$ and
$N_k:=\varsigma^2\sigma_{D,k}^2+\Sigma_{b,k}^2$, and recall that $\lambda\in[0,1]$ constrains
$\bt=\beta+\rho\lambda$ to $[\beta,1]$. Keeping the exact second moment
$\psi(\beta_K,\bt)=\beta_K^2-2\beta_K\bt+c_d\bt^2$ of \eqref{eq:psi} rather than its upper bound, the
$\bt$-dependent part of the direction-error bound is minimized without constraint at
\begin{equation}\label{eq:kalmanapp}
\begin{aligned}
&\bt^{\mathrm u}_k=\frac{S'_k}{c_dS_k+N_k}
=\frac{\|E_k\|^2+\beta_K\gamma^2\sigma_{K,k}^2}{c_d\big(\|E_k\|^2+\gamma^2\sigma_{K,k}^2\big)+N_k},\\
&\text{and the feasible optimum is}\qquad
\bt^\star_k=\min\big\{1,\max\{\beta,\bt^{\mathrm u}_k\}\big\},
\end{aligned}
\end{equation}
with $\lambda^\star_k=(\bt^\star_k-\beta)/\rho$. The numerator collapses to $S_k$ only when
$\beta_K=1$, that is without the known-component average. As $N_k\to0$ and $\beta_K\to1$,
$\bt^{\mathrm u}_k\to1/c_d$; as $N_k/S_k\to\infty$, $\bt^{\mathrm u}_k\to0$ and the lower clip is
active, giving $\lambda^\star_k=0$, that is \piso.
\end{proposition}

\begin{proof}
From Lemma~\ref{lem:dir}, retaining $\psi$ exactly in place of $\beta_K^2+c_d\bt^2$, the
$\bt$-dependent part is
\[
\begin{aligned}
g(\bt)&=\theta(\bt)\|E_k\|^2+\psi(\beta_K,\bt)V_k+\bt^2N_k\\
&=\big(\|E_k\|^2+\beta_K^2V_k\big)-2\bt\big(\|E_k\|^2+\beta_KV_k\big)+\bt^2\big(c_d\|E_k\|^2+c_dV_k+N_k\big),
\end{aligned}
\]
a strictly convex quadratic in $\bt$ with unconstrained minimizer $S'_k/(c_dS_k+N_k)$. On the interval
$[\beta,1]$ the minimizer of a convex quadratic is its unconstrained minimizer clipped to the
interval. The limits are immediate.
\end{proof}

The form $S'/(c_dS+N)$ is a scalar Kalman gain: correctable signal over total signal plus noise, with
$c_d$ the variance inflation of the sketch and the known-component variance counted at weight
$\beta_K$ in the numerator because only that fraction of it reaches the step. Its denominator is approximately observable from the innovation, because it is essentially the mean
squared innovation of the method.

\begin{proposition}[Innovation identity]\label{prop:innov}
Let $\widetilde N_k:=\Sigma_{b,k}^2+\Ek[\xi_k^2\|s_k\|^2]$. If \eqref{eq:D2app} holds with equality
then
\begin{equation}\label{eq:innovnorm}
\Ek\big\|R_k-m_k\big\|^2
=c_d\Big(\|E_k\|^2+\Ek\|\delta_k\|^2\Big)+\widetilde N_k-2\inner{E_k}{\bar B_k} ,
\end{equation}
and the correction term obeys
$2|\inner{E_k}{\bar B_k}|\le 2\varsigma_4\Sigma_{b,k}\|E_k\|$, so its size relative to the leading
term is at most $2\varsigma_4\Sigma_{b,k}/(c_d\|E_k\|)$.
\end{proposition}

\begin{proof}
By \eqref{eq:Rminusm},
$\|R_k-m_k\|^2=\big(\inner{E_k+\delta_k}{u_k}-b_k-\xi_k\big)^2\|s_k\|^2$. Expanding the square gives
six terms, which we evaluate individually.
\begin{enumerate}[label=(\roman*),leftmargin=1.8em,itemsep=1pt]
\item $\Ek[\inner{E_k}{u_k}^2\|s_k\|^2]=c_d\|E_k\|^2$, by \eqref{eq:D2app} with equality.
\item $\Ek[\inner{\delta_k}{u_k}^2\|s_k\|^2]=c_d\Ek\|\delta_k\|^2$, by conditioning on $\delta_k$,
which is independent of $(u_k,s_k)$, and applying \eqref{eq:D2app} with equality.
\item $2\Ek[\inner{E_k}{u_k}\inner{\delta_k}{u_k}\|s_k\|^2]=0$, since the expression is linear in the
mean-zero $\delta_k$, which is independent of $(u_k,s_k)$.
\item $\Ek[b_k^2\|s_k\|^2]+\Ek[\xi_k^2\|s_k\|^2]=\widetilde N_k$, by definition.
\item $-2\Ek[\inner{E_k+\delta_k}{u_k}b_k\|s_k\|^2]
=-2\inner{E_k}{\Ek[b_k\|s_k\|^2u_k]}=-2\inner{E_k}{\bar B_k}$, the $\delta_k$ part vanishing as in
(iii).
\item $-2\Ek[\inner{E_k+\delta_k}{u_k}\xi_k\|s_k\|^2]=0$ and $2\Ek[b_k\xi_k\|s_k\|^2]=0$, because
$\E[\xi_k\mid\Fk,u_k,s_k,\delta_k]=0$ and the remaining factors are measurable with respect to that
conditioning.
\end{enumerate}
Summing gives \eqref{eq:innovnorm}; the final bound is Cauchy--Schwarz together with
Lemma~\ref{lem:Bbar}.
\end{proof}

Consequently $c_dS_k+N_k$ can be approximated by the running mean of $\|R_k-m_k\|^2$, which the
method already computes, up to the correction of Proposition~\ref{prop:innov}, which is small whenever
the tracking error exceeds the finite-difference bias scale. This gives the denominator of
\eqref{eq:kalmanapp} directly. The numerator $S'_k=\|E_k\|^2+\beta_KV_k$ differs from
$S_k=\|E_k\|^2+V_k$ by $(1-\beta_K)V_k$, so an estimate or bound $\widehat V_k$ of the
known-component variance is also needed when $\beta_K<1$: given $\widehat N_k$ and $\widehat V_k$, set
$\widehat S_k=(\widehat{\|R-m\|^2}-\widehat N_k)/c_d$, $\widehat S'_k=\widehat S_k-(1-\beta_K)\widehat V_k$,
and read off $\bt^\star_k$ from \eqref{eq:kalmanapp}. The gain is thus approximately estimable rather
than directly observable. Since $\|E_k\|$
decreases along a successful run, $\bt^\star_k$ decays, so a constant $\lambda$ tuned on a fixed
budget approximates a time-average of a decaying schedule.

\subsection{The weight \texorpdfstring{$\gamma$}{gamma}}

\begin{proof}[Proof of Proposition~\ref{prop:gamma}]
By \eqref{eq:residual}, $v_{k+1}-v_k=\mathsf d^U_k+(1-\gamma)\mathsf d^K_k$ exactly. With $k$ and the
displacement fixed, the increments do not depend on $\gamma$ and expanding gives
\[
\E\big\|\mathsf d^U+(1-\gamma)\mathsf d^K\big\|^2
=\mathsf A+2(1-\gamma)\mathsf B+(1-\gamma)^2\mathsf C ,
\]
with $\mathsf A:=\E\|\mathsf d^U\|^2$, $\mathsf B:=\E\inner{\mathsf d^U}{\mathsf d^K}$ and
$\mathsf C:=\E\|\mathsf d^K\|^2$. This is a strictly convex quadratic in $\gamma$ when
$\mathsf C>0$; setting the derivative with respect to $(1-\gamma)$ to zero gives
$(1-\gamma)^\star=-\mathsf B/\mathsf C$, that is $\gamma^\star=1+\mathsf B/\mathsf C$, with minimum
value $\mathsf A-\mathsf B^2/\mathsf C=\mathsf A(1-\varphi^2)$ where
$\varphi=\mathsf B/\sqrt{\mathsf{AC}}$. At $\gamma=1$ the drift is $\mathsf A$, and at $\gamma=0$
the tracked object is $\nabla F$ and the drift is $\E\|\nabla F(x_{k+1})-\nabla F(x_k)\|^2$, which is
at most $L^2\eta^2\E\|G_k\|^2$ by $L$-smoothness.

We stress what the proposition does not say. The constant $L_v$ of \eqref{eq:pathwise} is a uniform
bound over the trajectory, and the trajectory depends on $\gamma$, so the per-step minimizer
$\gamma_k^\star$ is not the minimizer of $L_v(\gamma)$ and no statement of the form
$L_v(\gamma^\star)/L_v(1)=\sqrt{1-\varphi^2}$ is claimed for the uniform constant. Moreover
$\gamma_k^\star>0$ requires $\mathsf B_k>-\mathsf C_k$.
\end{proof}

Thus $\gamma$ is a control-variate coefficient built from uncentred second moments, applied to
gradient increments rather than to gradients. It is estimable: running with $\gamma=1$ so that $m_k$
tracks $\gU$, recording $\mathsf d^K_k$ and $m_{k+1}-m_k$ and regressing the latter on the former
estimates $\mathsf B/\mathsf C$, while the coefficient of determination estimates $\varphi^2$.

\begin{remark}[Where $\gamma$ enters, and where it does not]\label{rem:gammaplaces}
Inspecting \eqref{eq:main}, $\gamma$ appears in exactly three places: the noise term
$\tfrac72\theta\gamma^2\sigma_{K,k}^2$, which increases in $\gamma$; the initialization
$\|v(x_0)\|^2=\|\gU(x_0)+(1-\gamma)\gK(x_0)\|^2$; and $L_v$, through the step-size restriction.
It does not affect $c_d$, and therefore neither $\rho^\star$ nor $\bt^\star$. Including the noise
term, minimizing the per-step drift proxy $\mathsf A_k+2(1-\gamma)\mathsf B_k+(1-\gamma)^2\mathsf C_k$
plus $\nu\gamma^2$ with
$\nu\propto\sigma_K^2/(n_F\varepsilon^2)$ shrinks the regression estimate to
$\gamma^\star_\nu=\tfrac{\mathsf C}{\mathsf C+\nu}\gamma^\star$, recovering $\gamma^\star$ when the
known-component oracle is exact and $\gamma\to0$ when it is very noisy.
\end{remark}

\begin{remark}[A time-varying $\gamma_k$]\label{rem:adaptgamma}
Allowing $\gamma$ to vary changes the tracked object, adding
$-(\gamma_{k+1}-\gamma_k)\gK(x_{k+1})$ to $v_{k+1}-v_k$, so that
$\Ek\|v_{k+1}-v_k\|^2\le2L_v^2\eta^2\Ek\|G_k\|^2
+2|\gamma_{k+1}-\gamma_k|^2\Ek\|\gK(x_{k+1})\|^2$. The first term is absorbed as before with $\eta$
reduced by $\sqrt2$; the second contributes an additional floor
$\tfrac{20\theta}{\beta^2}\cdot\tfrac1T\sum_k|\gamma_{k+1}-\gamma_k|^2\E\|\gK\|^2$, so any schedule
with $\sum_k|\gamma_{k+1}-\gamma_k|^2=\mathcal O(\beta^2\varepsilon^2T/\theta)$ leaves the rate
unchanged.
\end{remark}

\section{Probes per update}\label{app:queries}

We justify Remark~\ref{rem:queries}. In the analysis of \citet{hikima2025estimators} for the
Gaussian multi-direction estimator with $N$ probes per iteration, the bound on the estimator error
contains a term proportional to $\tfrac{d}{N}\|\nabla F(x)\|^2$, with the stated constant giving
$\tfrac{54d}{N}\|\nabla F(x)\|^2$ after their Lemma 1. For a descent argument to close, this term
must be absorbed into the $-\|\nabla F(x)\|^2$ produced by the descent step, which requires
$54d/N<1$, that is $N=\Omega(d)$; obtaining their stated constant requires $N\ge81d$. Each probe
costs two function queries, so at least $162d$ queries precede every parameter update. At $d=30$
this is $4860$ queries, against a total budget of $5000$ in their own experiments.

By contrast, \piso attains the rates of Corollaries~\ref{cor:B6} and \ref{cor:B5} with a single probe
per iteration, because the averaging that reduces the sketch variance is performed across iterations
by the contraction \eqref{eq:teaser} rather than within one iteration. The two designs reach the same
$\varepsilon$ and $d$ dependence with different granularity: \piso updates after $\mathcal O(1)$
deployments, whereas a multi-direction method must release $\Omega(d)$ decisions between updates.

\section{Deployment accounting}\label{app:deploy}

\subsection{Deployments, the binding cost in performative prediction}\label{app:deploydisc}

In Regime B a query is not an observation but an intervention: evaluating $\widehat F$ at a point
requires \emph{deploying} that decision and collecting outcomes from the induced distribution. Each
iteration of Algorithm~\ref{alg:piso} deploys three decisions, namely $x_k$ for the known-component
oracle and $x_k\pm\mu_ku_k$ for the two function values, so the total number of deployments is
$3T$ regardless of the batch sizes. Deployments and samples are therefore distinct resources, and in
an application where a decision is a price schedule, a lending rule or a policy, the number of
distinct decisions released is the cost that matters.

The two batch sizes act on this cost in the same direction, through $\beta$ and $\beta_K$. Averaging
$n_F$ samples at each query point divides $\sigma_{D,k}^2$ by $n_F$ and permits a proportionally
larger $\beta$; averaging $n_K$ samples of $\ghK$ divides $\sigma_{K,k}^2$ by $n_K$ and permits a
proportionally larger $\beta_K$. Since $T=\Theta(\Delta_v/(\beta\varepsilon^2))$, both enlarge the
step per iteration and so reduce $T$. Two structural caps limit how far this can go: the sketch stability window $\beta\le1/c_d$ of
Appendix~\ref{app:params}, and the condition $\beta\le1/(3d)$ under which the weighted bias vector
$\bar B_k$ of Lemma~\ref{lem:structure} stays dominated by the mean bias
(Appendix~\ref{app:bias}). Both are $\Theta(1/d)$, so with
$C_1:=102000L^2\sigma_F^2\varsigma^2\chi_1^2=\Theta(d^2)$ the two extreme operating points are as
follows.

the two extreme operating points are compared in the table below.

\textbf{Two caps on $\beta$ govern how far this goes.} The first is
$\beta_1:=\chi_1^2/\varsigma_6^2$, the hypothesis of Corollary~\ref{cor:B6} under which the damped
second-moment bias $11\beta\Sigma_b^2\le\tfrac{11}{4}\beta L^2\mu^2\varsigma_6^2$ stays dominated by
the undamped mean bias $120L^2\mu^2\chi_1^2$; for Gaussian probes $\beta_1=1/(9d^2)$. The second is
the cap $\beta_2:=1/(3d)$ of \eqref{eq:regcond}, which is the limit of the sharp bias-vector
domination argument of Lemma~\ref{lem:Bbarsharp} and not a stability boundary: the sketch contraction
itself holds up to $2/c_d$, and Theorem~\ref{thm:main} uses the conservative range $\beta\le1/c_d$. For $\beta\le\beta_1$ the radius balance is as in Corollary~\ref{cor:B6} and
$\mathsf N=2n_FT$ is independent of $\beta$; deployments are $3T=\Theta(\Delta_v/(\beta\varepsilon^2))$
and reach $\Theta(d^2\Delta_v\varepsilon^{-2})$ at $\beta_1$, plus the additive term
$\gamma L_K\Delta_F\varepsilon^{-2}$ from the known-component branch at $\beta_K=\Theta(1)$, which
is absorbed whenever $L_v\ge\gamma L_K/d^2$. For $\beta_1<\beta\le\beta_2$ the bracket in the
radius balance becomes $120\chi_1^2+\tfrac{11}{4}\beta\varsigma_6^2=\Theta(\beta d^3)$, so the
required $n_F$ grows as $\beta^2d^4\varepsilon^{-4}$ and
$\mathsf N=2n_FT=\Theta(\beta d^4\Delta_v\varepsilon^{-6})$ grows linearly in $\beta$, reaching
$\Theta(d^3\Delta_v\varepsilon^{-6})$ at $\beta_2$, where deployments are
$\Theta(d\Delta_v\varepsilon^{-2})$. The minimum deployment count attainable without loss in
$\mathsf N$ is therefore $\Theta(d^2\Delta_v\varepsilon^{-2})$, and $\Theta(d\Delta_v\varepsilon^{-2})$
costs a factor $d$ in observations. The gap between the two caps is exactly the looseness of the
generic bound $\Sigma_b^2\le\tfrac14L^2\mu^2\varsigma_6^2$: on the half-quadratic example of
Remark~\ref{rem:sharpness} one computes $\Sigma_b^2=\tfrac{\mu^2}{16}d(3d+12)=\Theta(L^2\mu^2d^2)$, a
factor $d$ below the bound, and a general bound of that order would place $\beta_1$ at
$\Theta(1/d)$ and make the last operating point free. We do not have such a bound and leave it open. Third, at that point
$\beta_K=\Theta(1)$: the known-component average is close to inactive at the
deployment-optimal end, and its role is to make the opposite end, $n_K=1$, admissible at all
(Appendix~\ref{app:params}). The choice of $n_K$ therefore interpolates between minimizing
known-component samples and minimizing deployments.

The operating points compare as follows.

\begin{center}\small
\setlength{\tabcolsep}{3.5pt}
\begin{tabular}{@{}llccc@{}}
\toprule
& $\beta$ & deployments $3T$ & function evals.\ $2n_FT$ & known samples $n_KT$\\
\midrule
$n_F=n_K=1$ & $\varepsilon^4/C_1$ & $\Theta(d^2\varepsilon^{-6})$ & $\Theta(d^2\varepsilon^{-6})$ & $\Theta(d^2\varepsilon^{-6})$\\
batches raised until $\beta=\beta_1$ & $\Theta(1/d^2)$ & $\Theta(d^2\varepsilon^{-2})$ & $\Theta(d^2\varepsilon^{-6})$ & $\Theta(d\,\varepsilon^{-4})$\\
batches raised until $\beta=\beta_2$ & $\Theta(1/d)$ & $\Theta(d\,\varepsilon^{-2})$ & $\Theta(d^3\varepsilon^{-6})$ & $\Theta(d\,\varepsilon^{-4})$\\
\bottomrule
\end{tabular}
\end{center}

\emph{Comparison.} A multi-direction method draws $N$ probes per update, which is $2N$
\emph{distinct} deployed decisions, and this cannot be replaced by repeated sampling because
distinct directions are the object. With $N=\Theta(d^2\varepsilon^{-4})$ and
$T=\Theta(\varepsilon^{-2})$ \citep{hikima2025estimators}, its deployment count coincides with its
sample count, $\Theta(d^2\varepsilon^{-6})$. Against this, \piso at $\beta_1$ deploys a factor $\varepsilon^{-4}$ fewer decisions for the same
total number of observations, and at $\beta_2$ a factor $d\,\varepsilon^{-4}$ fewer at the cost of a
factor $d$ more observations.

\section{Additional experimental results}\label{app:experiments}

\begin{figure}[h]
\centering
\includegraphics[width=\textwidth]{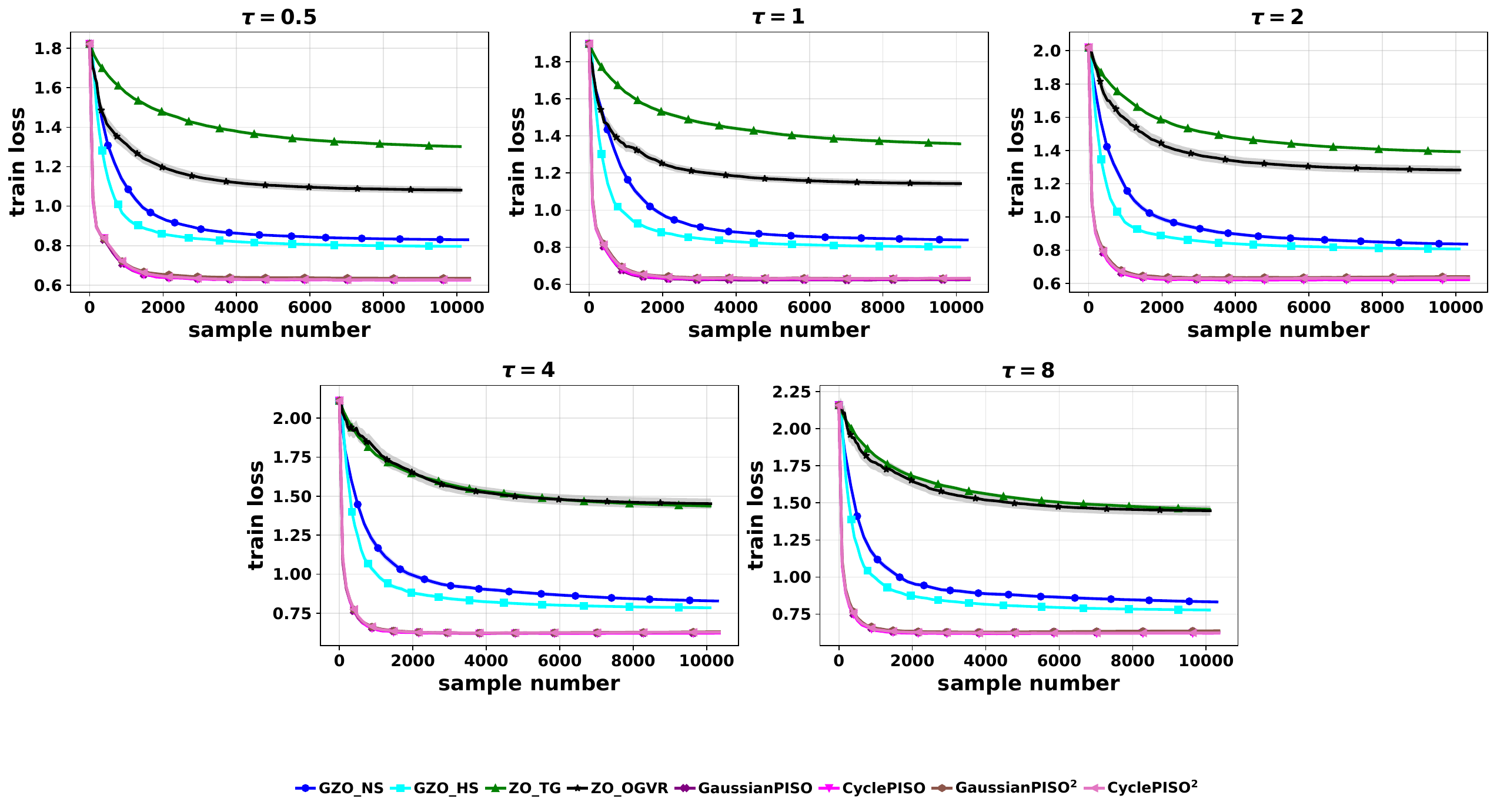}
\caption{Strategic classification on the credit dataset, $d=12$, five strategic-response strengths
$\tau$. Training loss against samples consumed; means over $100$ simulations with one-standard-error
bands. All four \piso variants reach a loss of about $0.63$ at every $\tau$, while the strongest
baseline plateaus near $0.80$; the two methods that ignore the computable component degrade sharply
as $\tau$ increases.}
\label{fig:classification}
\end{figure}

\begin{table}[h]
\centering\small
\setlength{\tabcolsep}{4pt}
\caption{Strategic classification: test loss and test accuracy (\%) at the end of the sample budget,
averaged over $100$ simulations, for each strategic-response strength $\tau$. Best per column in
bold. All four \piso variants dominate every baseline on both metrics at every $\tau$, by roughly
$0.2$ in test loss and $4$--$10$ accuracy points. The two methods that ignore the computable
component, ZO-TG and ZO-OGVR, degrade as $\tau$ grows, that is as the opaque component becomes more
influential, whereas the \piso variants are essentially flat in $\tau$; the guided baselines improve
with $\tau$ but never close the gap. Training metrics are in the accompanying archive.}
\label{tab:classification}
\begin{tabular}{@{}lcccccccccc@{}}
\toprule
& \multicolumn{2}{c}{$\tau=0.5$} & \multicolumn{2}{c}{$\tau=1$} & \multicolumn{2}{c}{$\tau=2$}
& \multicolumn{2}{c}{$\tau=4$} & \multicolumn{2}{c}{$\tau=8$}\\
\cmidrule(lr){2-3}\cmidrule(lr){4-5}\cmidrule(lr){6-7}\cmidrule(lr){8-9}\cmidrule(lr){10-11}
Method & loss & acc. & loss & acc. & loss & acc. & loss & acc. & loss & acc.\\
\midrule
GZO-NS & 0.891 & 59.1 & 0.912 & 60.8 & 0.899 & 63.2 & 0.879 & 65.0 & 0.878 & 65.7 \\
GZO-HS & 0.846 & 57.2 & 0.862 & 59.8 & 0.865 & 62.0 & 0.835 & 63.7 & 0.821 & 65.0 \\
ZO-TG & 1.428 & 45.4 & 1.487 & 46.0 & 1.505 & 45.8 & 1.558 & 44.6 & 1.580 & 44.1 \\
ZO-OGVR & 1.181 & 51.4 & 1.243 & 52.8 & 1.383 & 54.3 & 1.539 & 54.9 & 1.534 & 56.4 \\
GaussianPISO & 0.624 & \textbf{66.6} & \textbf{0.620} & 67.6 & 0.632 & 67.7 & 0.622 & \textbf{69.3} & 0.631 & 69.0 \\
CyclePISO & \textbf{0.622} & 66.6 & 0.627 & \textbf{67.6} & \textbf{0.620} & \textbf{68.3} & \textbf{0.619} & 69.0 & \textbf{0.618} & 69.5 \\
GaussianPISO$^2$ & 0.632 & 66.2 & 0.629 & 66.8 & 0.641 & 67.6 & 0.628 & 68.2 & 0.636 & 69.0 \\
CyclePISO$^2$ & 0.623 & 66.5 & 0.632 & 67.0 & 0.632 & 68.0 & 0.627 & 68.9 & 0.621 & \textbf{69.7} \\
\bottomrule
\end{tabular}
\end{table}

\begin{figure}[h]
\centering
\includegraphics[width=\textwidth]{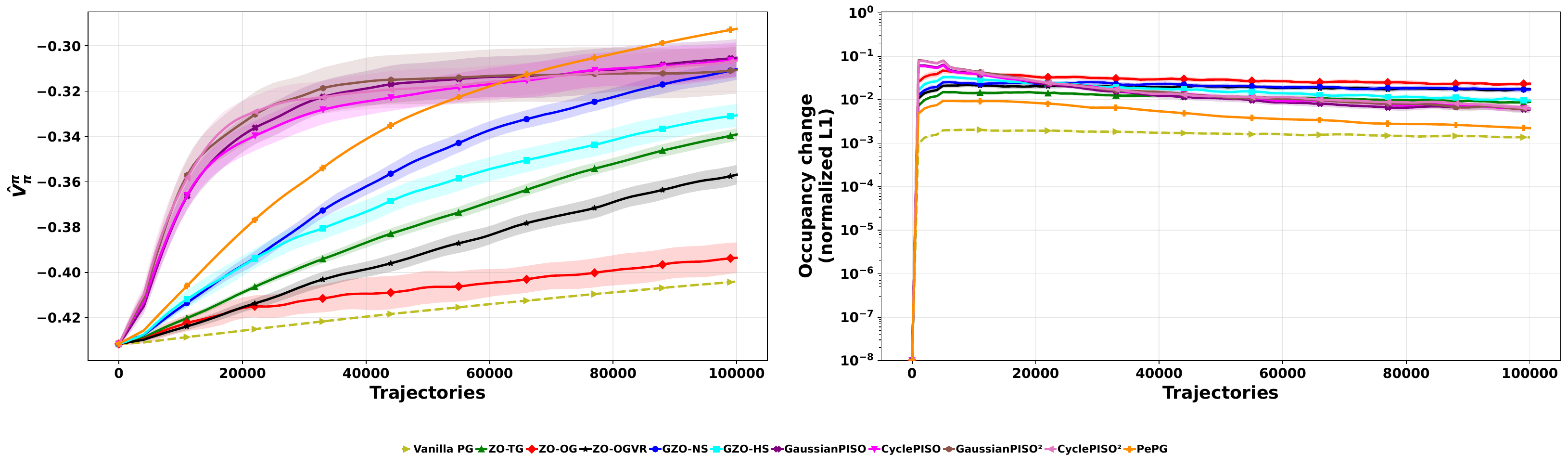}
\caption{Performative reinforcement learning on an $8\times8$ gridworld, $20$ seeds. Left:
performative return $\widehat V^\pi_\pi$ against the number of trajectories consumed, with
$\pm1.96$ standard-error bands. Right: normalized $L_1$ change of the state occupancy between
consecutive updates, on a log scale. \piso and \pisotwo make the fastest early progress among the
methods that operate under \eqref{eq:split}, while GZO-NS closes part of the gap late in the run.
PePG is a reference with privileged information: it is supplied with exact reward and transition
response Jacobians and exact tabular advantages, so it differentiates the environment's response
rather than querying it. Its curve is not a bound on what other methods can attain, but it indicates
what a differentiable response model is worth.}
\label{fig:rl}
\end{figure}

\paragraph{Performative RL.} Among methods that only query the environment
(Figure~\ref{fig:rl}) the ordering is \piso${}\approx{}$\pisotwo, with GZO-NS closing part of the
gap late in the run, followed by the remaining zeroth-order baselines and the vanilla policy gradient.
We regard PePG as the informative reference rather than a competitor: it is handed the derivative
that a query-based method must estimate, and what the figure shows is how much of the remaining range
is recovered without any model of the response. The right panel is a diagnostic of how far each
method moves the policy per update. \piso's occupancy change is of order $10^{-1}$ early in the run
and decays to about $10^{-2}$ at the end, while PePG and the vanilla policy gradient make smaller
late-stage moves. So \piso takes larger policy steps early in the run, and part of its advantage over
the vanilla policy gradient may reflect that; every baseline's step schedule was tuned independently
(Appendix~\ref{app:hyperparameters}), so this is not an artefact of a shared setting, but we do not
claim the gap is attributable to the estimator alone. The decay of \piso's occupancy change over the
run indicates that the larger steps do not destabilize the policy.

\begin{figure}[h]
\centering
\includegraphics[width=\textwidth]{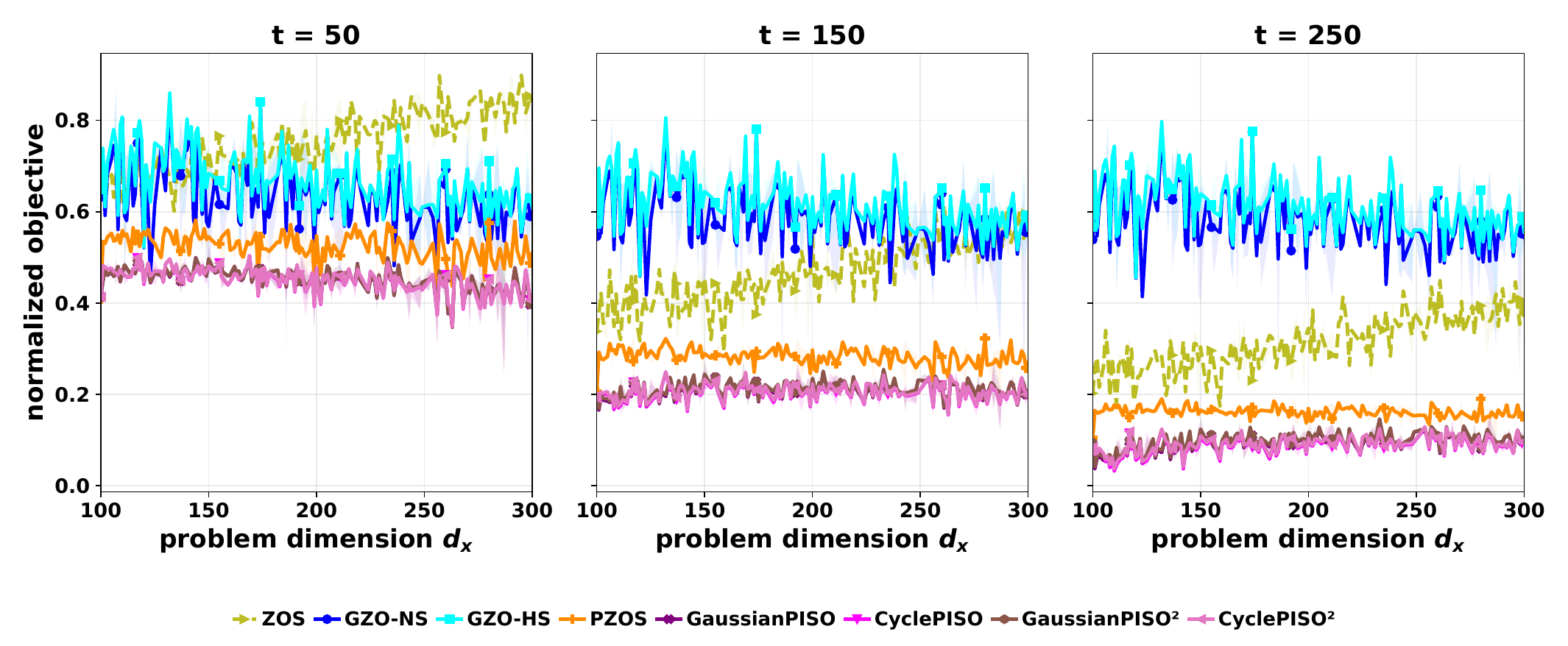}
\caption{Security game, normalized objective against problem dimension $d_x\in[100,300]$ at three
iteration counts. The separation between \piso and PZOS is stable across the range, and both
separate from the guided baselines, whose objective does not decrease appreciably after $t=50$.}
\label{fig:securitydim}
\end{figure}

\begin{figure}[h]
\centering
\includegraphics[width=0.92\textwidth]{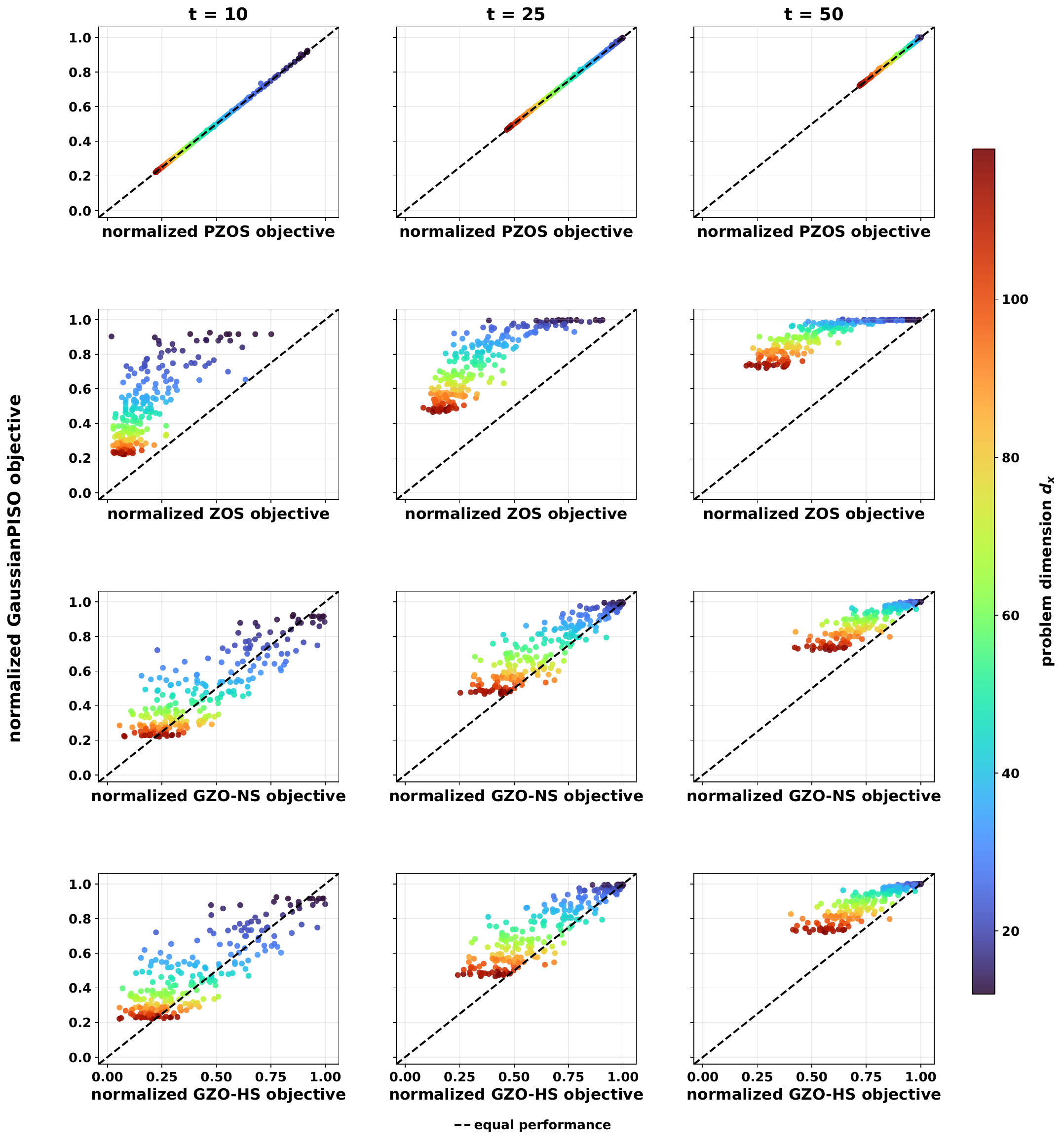}
\caption{Routing, per-instance comparison of GaussianPISO against each baseline at $t=10,25,50$,
coloured by problem dimension. Points above the diagonal favour \piso, which is a maximization
problem here. The PZOS row lies on the diagonal at every horizon, so the two methods return the same
solution per instance; the ZOS, GZO-NS and GZO-HS rows lie above it, with the margin largest at
small $t$ and at small dimension.}
\label{fig:routingpair}
\end{figure}

\subsection{Ablations}\label{app:ablations}

All ablations use the pricing problem of \S\ref{sec:experiments}, so $d=10$ and hence $c_d=12$ for
Gaussian probes and $c_d=10$ for the cycling variant. The batch size starts at $30$ and increases by
$2$ per iteration, so the known-component batch $n_K$ is large throughout; this matters for
Figure~\ref{fig:rhoK}.

\begin{figure}[h]
\centering
\includegraphics[width=\textwidth]{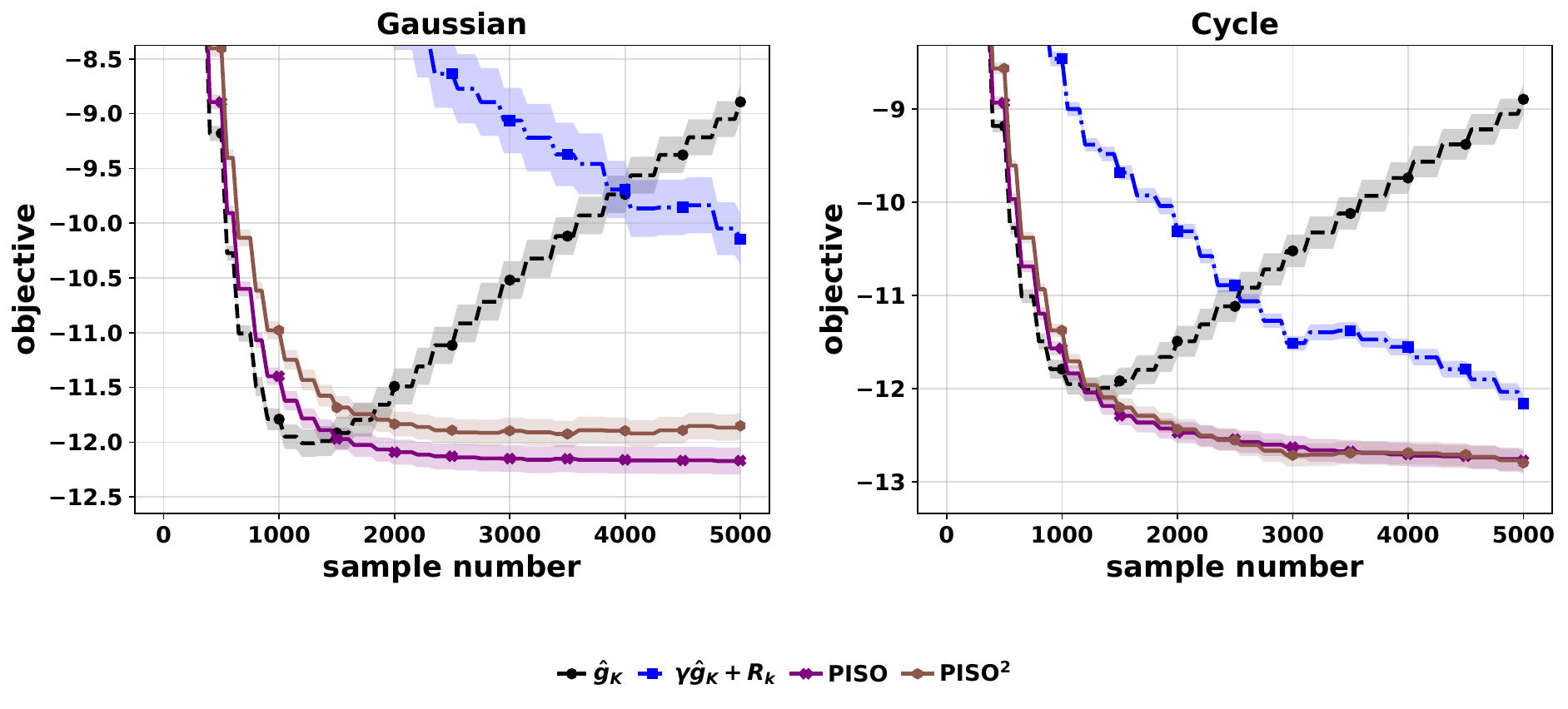}
\caption{Components. $\ghK$ alone (black) uses only the computable part and diverges once the opaque
part matters; $\gamma\ghK+R_k$ (blue) adds the \emph{uncentred} residual observation of
Appendix~\ref{app:centring} with $\rho=0$, and is unstable; \piso and \pisotwo add the centring and
the momentum. The two unstable arms are the two ways of omitting a piece of \eqref{eq:centred}, and
confirm that neither the known component nor an uncentred zeroth-order correction suffices on its
own.}
\label{fig:components}
\end{figure}

\begin{figure}[h]
\centering
\includegraphics[width=\textwidth]{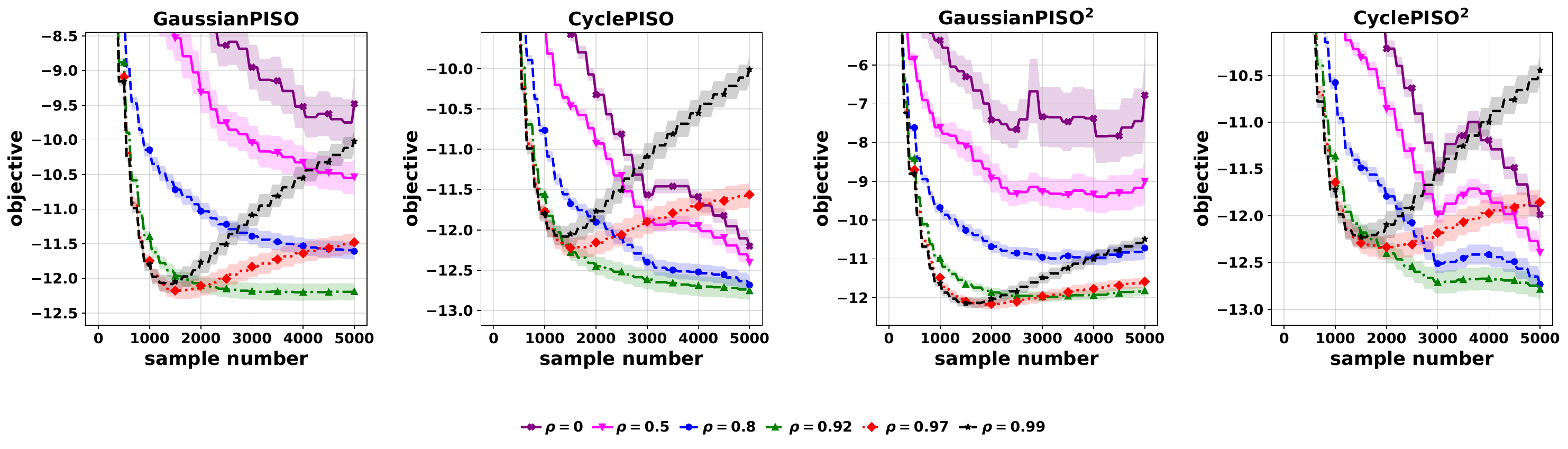}
\caption{Momentum $\rho$, the sharpest test of the analysis. Theory predicts an optimum at
$\rho^\star=1-1/c_d$, that is $0.917$ for the Gaussian panels and $0.900$ for the cycling panels, and
predicts that the residual iteration is neutral at $\rho=1-2/c_d$ and \emph{expands} below it, that is at $0.833$ and
$0.800$ respectively (Appendix~\ref{app:params}). Both are visible. The best curve is $\rho=0.92$ in all
four panels, the nearest grid point to the prediction; $\rho\in\{0,0.5,0.8\}$ all descend initially
and then \emph{rise}, the signature of a tracking error that grows once the sketch step size leaves
its stability window; and $\rho=0.99$ is over-smoothed, descending furthest before turning up as the
state lags the moving residual.}
\label{fig:rho}
\end{figure}

\begin{figure}[h]
\centering
\includegraphics[width=\textwidth]{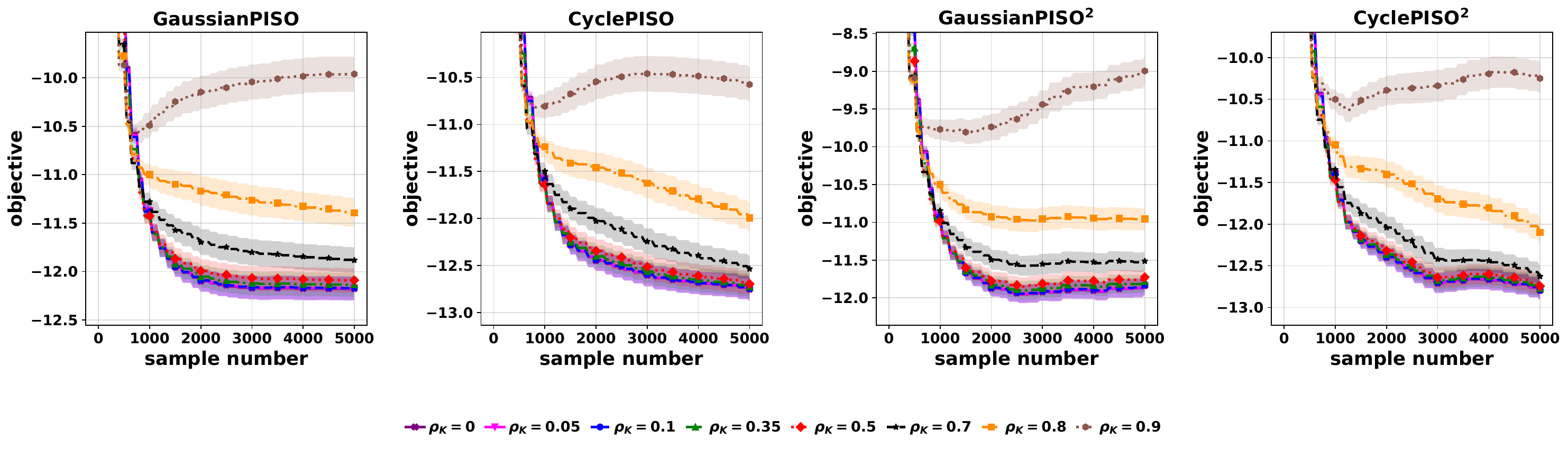}
\caption{Known-component momentum $\rho_K$. Here $\rho_K=0$, that is $\beta_K=1$ and a fresh
estimate, is among the best choices, and performance degrades monotonically for
$\rho_K\ge0.7$. This is what the analysis predicts in this regime: with the growing batch the
known-component variance $\sigma_K^2/n_K$ is already small, so $\beta_K=1$ satisfies the floor
condition of \eqref{eq:floor} and a smaller $\beta_K$ only tightens the step-size cap
$\eta\le\beta_K/(9\gamma L_K)$ of \eqref{eq:stepcond} and adds lag. The known-component average earns
its place when $n_K$ is small, where it is what makes $n_K=1$ admissible
(Remark~\ref{rem:noKmom}); it is not needed when one can afford a large batch.}
\label{fig:rhoK}
\end{figure}

\begin{figure}[h]
\centering
\includegraphics[width=\textwidth]{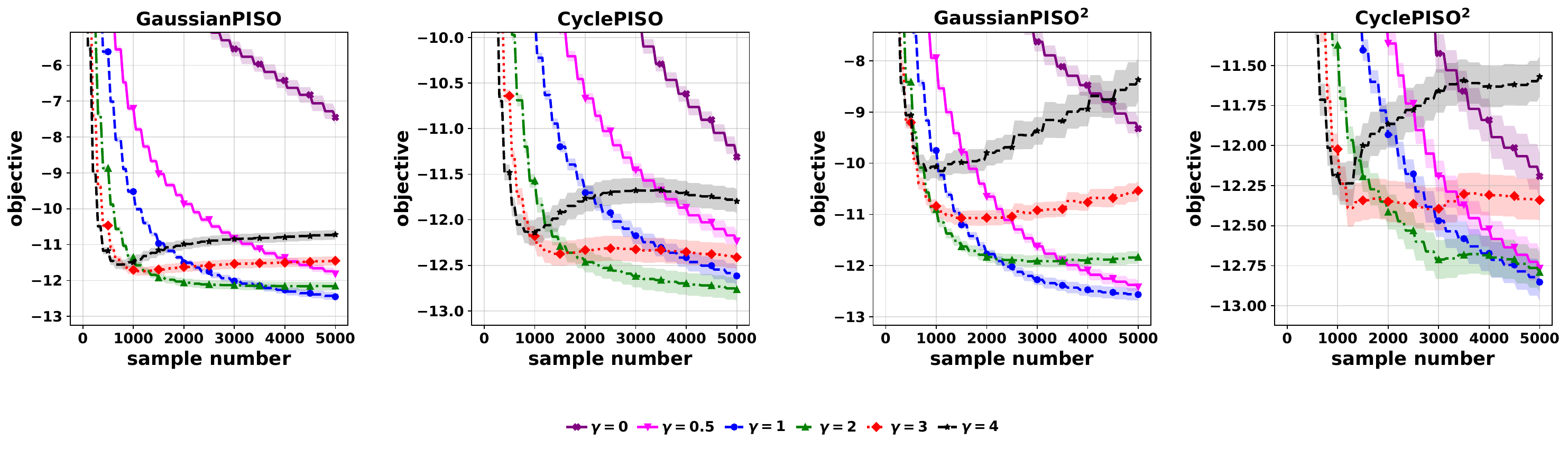}
\caption{Known-component weight $\gamma$. The endpoints behave as
\S\ref{app:LvsLv} describes: $\gamma=0$ tracks $\nabla F$ and is worst, and large $\gamma$
degrades as the noise term $\gamma^2\sigma_K^2$ and the drift constant $L_v(\gamma)$ both grow. The
interior optimum is $\gamma\in\{1,2\}$ depending on the variant, consistent with
Proposition~\ref{prop:gamma} placing $\gamma^\star=1+\mathsf B/\mathsf C$ away from both endpoints,
though the grid is too coarse to resolve $\gamma^\star$ precisely and we do not attempt to
estimate it directly.}
\label{fig:gamma}
\end{figure}

\begin{figure}[h]
\centering
\includegraphics[width=0.8\textwidth]{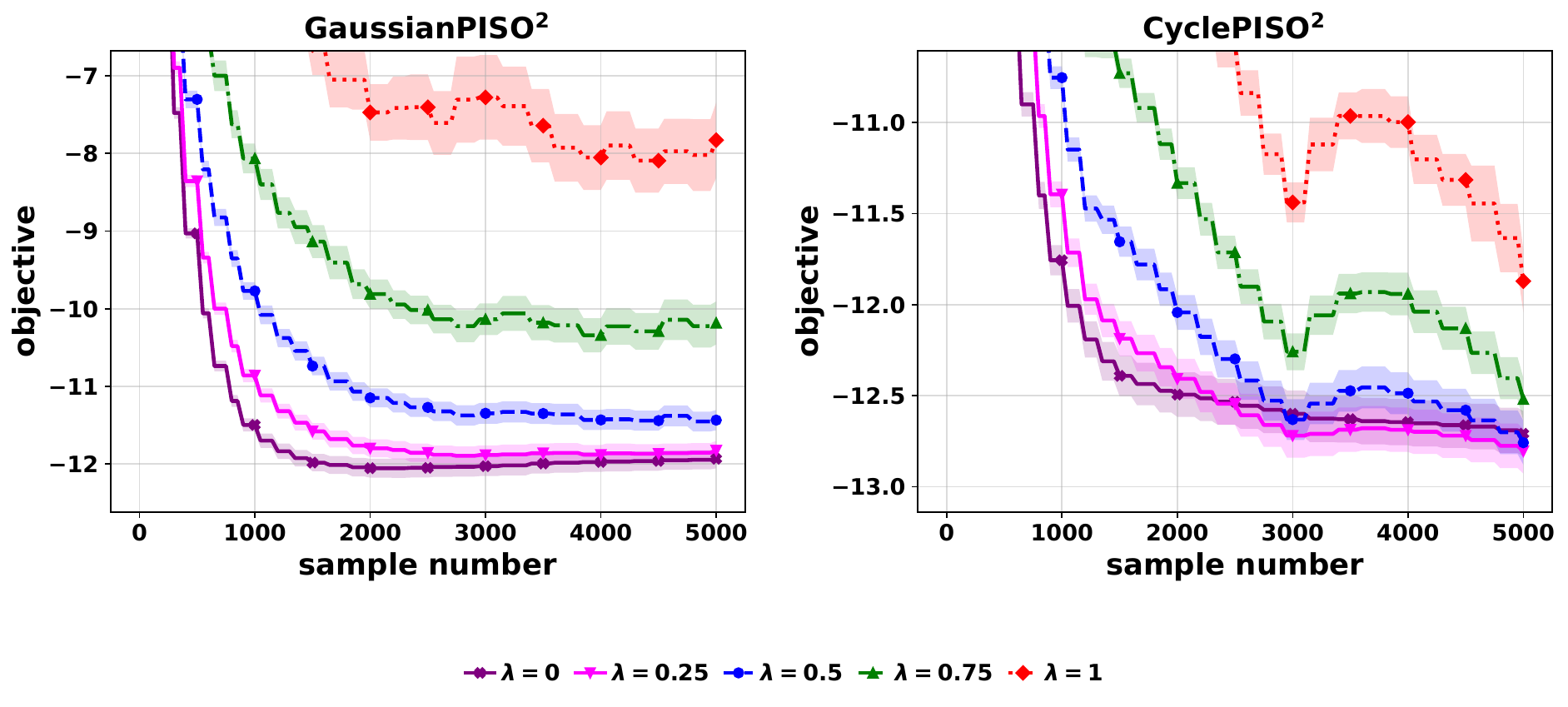}
\caption{Mixing weight $\lambda$. Performance is monotonically decreasing in $\lambda$ over the grid
$\{0,0.25,0.5,0.75,1\}$, with $\lambda=1$, which discards the state and uses the fresh observation
alone, far worse than $\lambda=0$. This is consistent with \eqref{eq:kalmanapp}: at
$\rho=0.92$ and $c_d\in\{10,12\}$ the high-signal-to-noise optimum
$\lambda^\star=(1/c_d-\beta)/\rho$ is $0.004$ and $0.022$, both below the smallest nonzero grid
point, so the grid cannot separate $\lambda^\star$ from $0$ and only the degradation for large
$\lambda$ is resolved. A finer sweep near zero would be needed to test the prescription itself.}
\label{fig:lambda}
\end{figure}

\begin{figure}[h]
\centering
\includegraphics[width=0.8\textwidth]{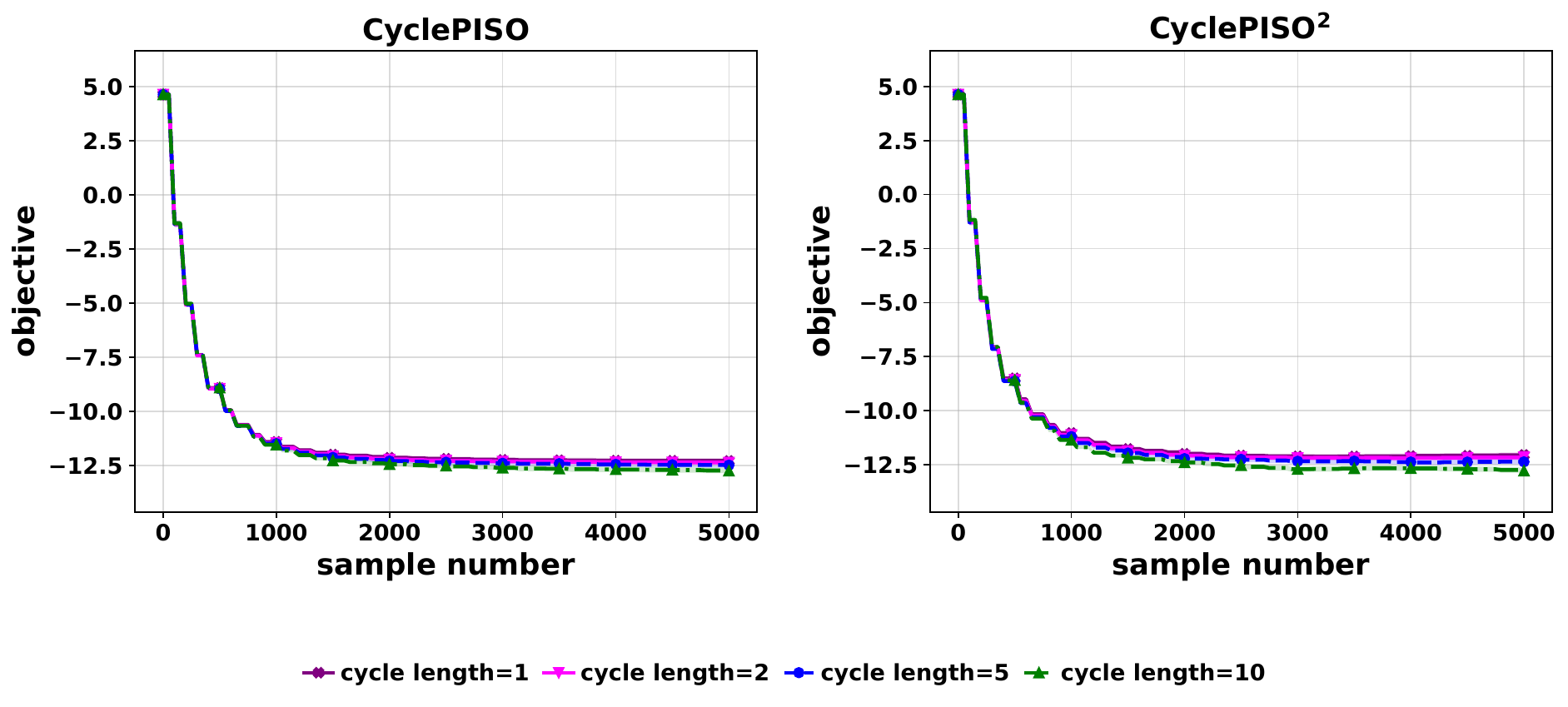}
\caption{Cycle length $L$ for the orthogonalized probe, at $d=10$. Longer cycles are mildly better
and $L=d=10$ is best, but the differences are small and within a fraction of the spread of the
$\rho$ ablation. Since the per-step marginal is the same for every $L$, and only the dependence
between consecutive probes changes, the analysis of \S\ref{sec:theory} does not distinguish these
curves; the ordering is consistent with orthogonality reducing the effective sketch variance, but we
do not claim it as a prediction.}
\label{fig:cyclelen}
\end{figure}

\subsection{Experimental hyperparameters}\label{app:hyperparameters}

For reproducibility, we report  the complete hyperparameter configurations used to produce the experimental results.
 Unless
stated otherwise, a decay factor is applied multiplicatively at every update. For the baselines, we use the hyperparameters reported as optimal in their respective references, except for performative RL, where we tune the zeroth-order baselines independently. 

\paragraph{Radius and known-part conventions.}
We denote the  perturbation radius by $r_k$. For ZO-TG, ZO-OG, ZO-OGVR, and the Gaussian
and Cycle variants of \piso and \pisotwo, we use
$\mu_k=r_k/\sqrt d$. For GZO-NS, GZO-HS, ZOS, and PZOS, we use
$\mu_k=r_k$. Thus, across methods, $r_k$ represents the scale of the
actual perturbation after accounting for the normalization of the probe
direction. For exponential schedules we write
$r_k=r_0q_r^k$ and $\eta_k=\eta_0q_\eta^k$, where $q_r$ and $q_\eta$
are the radius and step decay factors. When a radius floor is used, we denote
it by $r_{\min}$. We denote the known-sample fraction
by $f_K$. Unless stated otherwise, the reported configurations use a fresh known part,
equivalently $\rho_K=0$.

\subsubsection{Multiproduct pricing}

Let $m_k:=30+2k$, where $k$ is the current iteration number.

\noindent\textbf{GZO-NS.} We set $r_0:=0.19$, $r_{\min}:=10^{-4}$, $q_r:=0.95$,
$\beta_0:=0.01$, the step decay to $0.95$, $\alpha_0:=0$, and the damping factor to $0.98$. The
batch size is $m_k$.

\noindent\textbf{GZO-HS.} The $r_k$ and step schedules are the same as for GZO-NS. We use a
historical-guide window of $1$, $\alpha_0:=0$, damping $0.98$, and batch size $m_k$.

\noindent\textbf{ZO-TG.} We set $r_0:=0.19$, $r_{\min}:=10^{-4}$, and $q_r:=0.95$;
the step starts from $\beta_0:=0.0031623$ and decays by $0.95$. We use the Gaussian two-point
estimator with batch size $m_k$.

\noindent\textbf{ZO-OGVR.} We use the same $r_k$ schedule as above, with $\beta_0:=0.001$ and
step decay $0.95$. The initial baseline uses $20$ samples, the history length is $s_{\max}:=10$, the
distance coefficient is $M:=0.1$, and the batch size is $m_k$.

\noindent\textbf{Gaussian \piso.} We set $r_0:=0.40$, $q_r:=0.96$, $\eta_0:=0.005$, and
$q_\eta:=0.98$. We use $\gamma:=2.00$, $\rho:=0.92$, and $f_K:=0.50$.
The base batch is $m_k$, giving total \piso cost $3m_k$ per update.

\noindent\textbf{Cycle \piso.} The parameters are the same as for Gaussian \piso, with cycle length
$L:=10$.

\noindent\textbf{Gaussian \pisotwo.} We set $r_0:=0.40$, $q_r:=0.96$, $\eta_0:=0.005$, and
$q_\eta:=1$. We use $\gamma:=2.00$, $\rho:=0.92$, $\lambda:=0.25$, and
$f_K:=0.50$. the total cost is $3m_k$ per update.

\noindent\textbf{Cycle \pisotwo.} The parameters are the same as for Gaussian \pisotwo, with cycle
length $L:=10$.

We select the \piso hyperparameters using a two-stage grid search. In the first stage, we search
over $r_0\in\{0.20,0.40,0.80\}$,
$\eta_0\in\{0.0025,0.005,0.010\}$, and
$\rho\in\{0.90,0.92,0.97\}$ for all four variants. In the second stage, we search over
$f_K\in\{0.34,0.50,0.67\}$ for all four variants, the cycle length
$L\in\{2,5,10\}$ for Cycle \piso and Cycle \pisotwo, and the mixing weight
$\lambda\in\{0.10,0.25,0.50\}$ for Gaussian \pisotwo and Cycle \pisotwo.

\subsubsection{Strategic classification}

The optimization batch is $30+2k$ throughout unless stated otherwise.

\noindent\textbf{GZO-NS.} We set $r_0:=10$, $r_{\min}:=0.10$, $q_r:=0.95$,
$\beta_0:=1.0$, step decay $0.95$, $\alpha_0:=0$, and damping $0.98$. The optimization batch is
$30+2k$ and the guide batch is $100+2k$.

\noindent\textbf{GZO-HS.} The $r_k$ and step schedules are the same as for GZO-NS. We set
$\alpha_0:=0$, damping to $0.98$, and the historical window to $1$, with batch size $30+2k$.

\noindent\textbf{ZO-TG.} We set $r_0:=10$, $r_{\min}:=0.10$, and $q_r:=0.95$, and use
$\beta_0:=1.0$ with step decay $0.95$. The estimator is Gaussian two-point and the batch size is
$30+2k$.

\noindent\textbf{ZO-OGVR.} We use the same $r_k$ schedule, with $\beta_0:=d^{-1/2}$ and step decay
$0.95$. The initial baseline uses $20$ samples, $s_{\max}:=10$, $M:=0.10$, and batch size $30+2k$.

\noindent\textbf{Gaussian \piso.} We set $r_0:=3.00$, $q_r:=0.95$, $\eta_0:=0.70$, and
$q_\eta:=0.98$. We use $\gamma:=3.00$, $\rho:=0.98$, and $f_K:=0.67$, with
base batch $30+2k$.

\noindent\textbf{Cycle \piso.} The parameters are the same as for Gaussian \piso, with cycle length
$L:=6$.

\noindent\textbf{Gaussian \pisotwo.} The $r_k$ and $\eta_k$ schedules are the same as for Gaussian
\piso. We use $\gamma:=3.00$, $\rho:=0.98$, $\lambda:=0.05$, and $f_K:=0.75$, with
base batch $30+2k$.

\noindent\textbf{Cycle \pisotwo.} The parameters are the same as for Gaussian \pisotwo, with cycle
length $L:=6$.

We select the \piso hyperparameters using a two-stage grid search. In the first stage, we search
over $r_0\in\{1.50,3.00,6.00\}$,
$\eta_0\in\{0.35,0.70,1.40\}$, and
$\rho\in\{0.95,0.98,0.99\}$ for all four variants. In the second stage, we search over
$f_K\in\{0.50,0.67,0.75\}$ for all four variants, the cycle length
$L\in\{3,6,12\}$ for Cycle \piso and Cycle \pisotwo, and the mixing weight
$\lambda\in\{0.01,0.05,0.10\}$ for Gaussian \pisotwo and Cycle \pisotwo.

\subsubsection{Performative reinforcement learning}

For this experiment, an inverse-square-root schedule with initial value $q_0$, scale $S_q$, and
minimum $q_{\min}$ means
\[
q(N):=\max\!\left\{q_{\min},\frac{q_0}{\sqrt{1+N/S_q}}\right\},
\]
where $N$ is the cumulative trajectory count.

\noindent\textbf{Vanilla PG.} The step schedule has initial value $0.05$, scale $10^5$, and minimum
$0.02$. We use batches of $100$ trajectories and clip the gradient norm at $0.25$.

\noindent\textbf{ZO-TG.} We set $r_0:=4.0$, $S_r:=2\cdot10^5$, and
$r_{\min}:=3.0$; the step schedule has initial value $0.40$, scale $10^5$, and minimum $0.20$. We use
$50$ trajectories per side and gradient clip $0.25$.

\noindent\textbf{ZO-OG.} We set $r_k:=0.50$ and use a step schedule with initial value $0.08$,
scale $10^5$, and minimum $0.04$. We use batch size $100$, baseline decay $0.98$, and no gradient
clipping.

\noindent\textbf{ZO-OGVR.} We set $r_0:=2.0$, $S_r:=8\cdot10^5$, and $r_{\min}:=0.5$, and use
constant step $0.16$. We use batch size $100$, initial baseline batch $100$, history window $10$,
distance weight $0$, and no gradient clipping.

\noindent\textbf{GZO-NS.} We set $r_0:=4.0$, $S_r:=5\cdot10^4$, and $r_{\min}:=1.0$, and use
constant step $0.40$. We set $\alpha:=0.9$ and use total batch $100$, split into $20$ known-gradient
trajectories and $40$ trajectories for each zeroth-order side. The gradient clip is $1.0$.

\noindent\textbf{GZO-HS.} We set $r_k:=2.0$. The step schedule has initial value $0.40$, scale
$2.5\cdot10^4$, and minimum $0.04$. We set $\alpha:=0.9$, history window $1$, use $50$
trajectories per zeroth-order side, and disable gradient clipping.

\noindent\textbf{Gaussian \piso.} We set $r_k:=1.50$, with
$\mu_k:=1.50/\sqrt d$. For the step schedule, $\eta_0:=1.80$,
$S_\eta:=10^5$, and $\eta_{\min}:=1.00$. We set $\gamma:=0.75$, $\rho:=0.997$,
$\rho_K:=0.90$, and $f_K:=0.50$, with total batch $100$ and one direction per update.

\noindent\textbf{Cycle \piso.} The parameters are the same as for Gaussian \piso, with cycle length
$L:=24$.

\noindent\textbf{Gaussian \pisotwo.} The $r_k$ and $\eta_k$ schedules are the same as for Gaussian
\piso. We use $\gamma:=1.00$, $\rho:=0.997$, $\rho_K:=0.90$, $\lambda:=0.01$, and
$f_K:=0.50$, with total batch $100$ and one direction per update.

\noindent\textbf{Cycle \pisotwo.} The parameters are the same as for Gaussian \pisotwo, with cycle
length $L:=16$.

\noindent\textbf{PePG.} The step schedule has initial value $0.5$, scale $10^5$, and minimum $0.1$.
We use batch size $100$, gradient clip $0.25$, and transition-probability floor $10^{-12}$. PePG uses
exact reward- and transition-response derivatives and exact tabular advantages, while estimating the
required expectations from trajectories.

We select the zeroth-order baseline hyperparameters using a two-stage grid search. In the first
stage, for ZO-TG we search over $r_0\in\{1,2,4\}$,
$\eta_0\in\{0.10,0.20,0.40\}$, and gradient clip in $\{0.25,0.50,1.00\}$; for ZO-OG over
$r_0\in\{0.25,0.50,1.00\}$, $\eta_0\in\{0.04,0.08,0.16\}$, baseline decay in
$\{0.90,0.98,0.995\}$, and gradient clip in $\{\mathrm{none},0.25,0.50\}$; for ZO-OGVR over
$r_0\in\{1,2,4\}$, $\eta_0\in\{0.08,0.16,0.32\}$, $M\in\{0,0.1,1\}$, and gradient clip in
$\{\mathrm{none},0.25,0.50\}$; for GZO-NS over $r_0\in\{2,4,8\}$,
$\eta_0\in\{0.20,0.40,0.80\}$, known batch in $\{20,34,50\}$, and gradient clip in
$\{0.25,0.50,1.00\}$; and for GZO-HS over $r_0\in\{1,2,4\}$,
$\eta_0\in\{0.20,0.40,0.80\}$, historical window in $\{1,5,10\}$, and gradient clip in
$\{\mathrm{none},0.25,0.50\}$. In the second stage, we search over radius schedule in
$\{\mathrm{constant},\mathrm{inverse\mbox{-}sqrt}\}$, $S_r\in\{5\cdot10^4,2\cdot10^5,8\cdot10^5\}$,
$r_{\min}\in\{0.5,1,3\}$, step schedule in
$\{\mathrm{constant},\mathrm{inverse\mbox{-}sqrt}\}$, $S_\eta\in\{2.5\cdot10^4,10^5,4\cdot10^5\}$,
and $\eta_{\min}\in\{0.04,0.10,0.20\}$; for GZO-NS and GZO-HS we additionally search over
$\alpha\in\{0.75,0.90,1.00\}$.

We select the \piso hyperparameters using a two-stage grid search. In the first stage, we search
over $r_0\in\{1.00,1.50,2.00\}$,
$\eta_0\in\{1.20,1.80,2.40\}$,
$\gamma\in\{0.75,1.00,1.25\}$,
$\rho\in\{0.99,0.995,0.997\}$,
$\rho_K\in\{0.00,0.50,0.90\}$, and
$f_K\in\{0.34,0.50,0.67\}$. Gaussian \piso and Cycle \piso share these parameters, as do Gaussian
\pisotwo and Cycle \pisotwo. In the second stage, we search over
$S_\eta\in\{5\cdot10^4,10^5,2\cdot10^5\}$,
$\eta_{\min}\in\{0.50,1.00,1.50\}$,
the cycle length $L\in\{16,24,32\}$ for Cycle \piso and Cycle \pisotwo, and the mixing weight
$\lambda\in\{0.01,0.03,0.10\}$ for Gaussian \pisotwo and Cycle \pisotwo.

\subsubsection{Bilevel routing game}

For every method with dimension scaling, the base step is divided by the problem
dimension $d$.

\noindent\textbf{ZOS.} We set $r_0:=0.50$ and $q_r:=1$ and use constant step $0.70/d$ with a unit-sphere
two-point estimator.

\noindent\textbf{PZOS.} We use $r_0:=0.50$ and $q_r:=1$ and constant step $0.70/d$. The update combines the exact
leader partial gradient with a two-point estimate of the follower-response contribution.

\noindent\textbf{GZO-NS.} We set $r_0:=0.50$ and $q_r:=1$ and use constant step $0.70/d$, with
$\alpha_0:=0$, damping $0.98$, and history window $5$. Update clipping is disabled.

\noindent\textbf{GZO-HS.} The parameters are the same as for GZO-NS, using the historical-guide
variant.

\noindent\textbf{Gaussian \piso.} We set $r_0:=0.50$, $q_r:=1$, $\eta_0:=0.70/d$, and
$q_\eta:=1$, with $\gamma:=1.00$ and $\rho:=0.98$.

\noindent\textbf{Cycle \piso.} The parameters are the same as for Gaussian \piso, with cycle length
$L:=10$.

\noindent\textbf{Gaussian \pisotwo.} The $r_k$ and $\eta_k$ schedules are the same as for Gaussian
\piso. We use $\gamma:=1.00$, $\rho:=0.98$, and $\lambda:=0.10$.

\noindent\textbf{Cycle \pisotwo.} The parameters are the same as for Gaussian \pisotwo, with cycle
length $L:=10$.

We select the \piso hyperparameters using a two-stage grid search. In the first stage, we search
over $r_0\in\{0.25,0.50,1.00\}$,
$a\in\{0.35,0.70,1.40\}$ with $\eta_0=a/d$, and the residual momentum
$\rho\in\{0.97,0.98,0.99\}$ for all four variants. In the second stage, we search over the cycle
length $L\in\{5,10,20\}$, with effective cycle length $\min\{L,d\}$, for Cycle \piso and
Cycle \pisotwo, and the mixing weight $\lambda\in\{0.03,0.10,0.30\}$ for Gaussian \pisotwo
and Cycle \pisotwo.

\subsubsection{Bilevel security game}

All step sizes use the inverse-square-root iteration schedule $\eta_k=\eta_0/\sqrt{k+1}$.

\paragraph{Residual norm cap.}
For security, we apply a norm cap to the residual observation
and its momentum state as a numerical safeguard against occasional value
explosions induced by unusually large   observations. Let
$g_k^{\mathrm K}$ denote the unweighted known gradient, and define
\[
 B_k^{\mathrm{cap}}
 :=c_{\mathrm{cap}}\max\{\|g_k^{\mathrm K}\|_2,a_K\},
 \qquad c_{\mathrm{cap}}>0,\quad a_K>0.
\]
The norm-cap ratio $c_{\mathrm{cap}}$
 is the multiplier of this reference norm.
The known-norm floor $a_K$ 
is a lower bound on the reference scale, so the cap remains
$c_{\mathrm{cap}}a_K$ when the known-gradient norm is small. It does not
impose a lower bound on either gradient or residual norm. For $B>0$, let
\[
 \mathcal C_B(z):=
 \begin{cases}
 z,&\|z\|_2\le B,\\
 \displaystyle\frac{B}{\|z\|_2}z,&\|z\|_2>B.
 \end{cases}
\]
Writing the raw residual observation and raw momentum update as
\[
 R_k^{\mathrm{raw}}
 :=m_k+\bigl(D_k-\langle\gamma g_k^{\mathrm K}+m_k,u_k\rangle\bigr)s_k,
 \qquad
 \widetilde m_{k+1}:=\rho m_k+(1-\rho)R_k^{\mathrm{raw}},
\]
the implementation sets
\[
 m_{k+1}:=\mathcal C_{B_k^{\mathrm{cap}}}(\widetilde m_{k+1}),
 \qquad
 \overline R_k:=\mathcal C_{B_k^{\mathrm{cap}}}(R_k^{\mathrm{raw}}),
 \qquad
 G_k:=\gamma g_k^{\mathrm K}
       +(1-\lambda)m_{k+1}+\lambda\overline R_k.
\]
Here $\lambda=0$ for \piso and $\lambda>0$ for \pisotwo. Both vectors
are clipped separately, after forming the raw momentum update and before
forming the readout. Consequently,
$\|(1-\lambda)m_{k+1}+\lambda\overline R_k\|_2\le B_k^{\mathrm{cap}}$
for $\lambda\in[0,1]$. The cap uses $\|g_k^{\mathrm K}\|_2$, without
the factor $\gamma$. The security experiment uses
$c_{\mathrm{cap}}=a_K=1$, giving
$B_k^{\mathrm{cap}}=\max\{\|g_k^{\mathrm K}\|_2,1\}$.

\noindent\textbf{ZOS.} We set $r_0:=0.10$ and $q_r:=1$. The step base is $0.05$ for $d\le100$ and $0.03$ for
$d>100$. The estimator is unit-sphere two-point.

\noindent\textbf{PZOS.} We use the same $r_k$ and step schedule as ZOS, combining the exact leader
partial gradient with a finite-difference estimate of the follower-response contribution.

\noindent\textbf{GZO-NS.} We set $r_0:=0.10$ and $q_r:=1$ and use the same $0.05/0.03$ step bases as ZOS,
with $\alpha_0:=0$, damping $0.98$, and history window $5$. Update clipping is disabled.

\noindent\textbf{GZO-HS.} The parameters are the same as for GZO-NS, using the historical-guide
variant.

\noindent\textbf{Gaussian \piso.} We set $r_0:=0.10$ and $q_r:=1$. We set
\(\eta_0:= 0.06\) for \(d\le100\) and \(\eta_0:= 0.04\) for \(d>100\).
We use $\gamma:=1.00$ and $\rho:=0.98$. 

\noindent\textbf{Cycle \piso.} The parameters are the same as for Gaussian \piso, with cycle length
$L:=10$.

\noindent\textbf{Gaussian \pisotwo.} The $r_k$ and $\eta_k$ schedules are the same as for Gaussian
\piso. We use $\gamma:=1.00$, $\rho:=0.98$, and $\lambda:=0.10$.

\noindent\textbf{Cycle \pisotwo.} The parameters are the same as for Gaussian \pisotwo, with cycle length
$L:=10$.

We select the \piso hyperparameters using a two-stage grid search. In the first stage, we search
over $r_0\in\{0.05,0.10,0.20\}$, the paired
$(a_{d\le100},a_{d>100})\in
\{(0.03,0.02),(0.06,0.04),(0.12,0.08)\}$, and the residual momentum
$\rho\in\{0.97,0.98,0.99\}$ for all four variants. In the second stage, we search over the cycle
length $L\in\{5,10,20\}$, with effective cycle length $\min\{L,d\}$, for Cycle \piso and
Cycle \pisotwo, and the mixing weight $\lambda\in\{0.03,0.10,0.30\}$ for Gaussian \pisotwo
and Cycle \pisotwo.

\subsubsection{ Pricing ablations common settings}
The paper profile uses $100$ simulations, $5000$ training
samples, $1000$ metric samples, and $101$ reporting-grid points. The default direction is Gaussian;
Cycle panels replace it by the orthogonal cycle. Both \piso and \pisotwo use central differences with
independent coupling, $r_0:=0.40$, $q_r:=0.96$, $\gamma:=2.00$, $\rho:=0.92$,
$f_K:=0.50$, an initial batch of $30$ increasing by $2$ per update, one direction per step, and cycle
length $10$. For \piso, $\eta_0:=0.005$ and $q_\eta:=0.98$; for \pisotwo,
$\eta_0:=0.005$ and $q_\eta:=1$, and \pisotwo additionally uses $\lambda:=0.25$.

\end{document}